\documentclass[11pt]{article}
\usepackage{latexsym}
\usepackage{amsmath}
\usepackage{amssymb}
\usepackage{amsthm}
\usepackage{epsfig}
\usepackage{xcolor}
\usepackage{graphicx}
\usepackage{bm}
\usepackage{enumitem}

\usepackage[top=1in, bottom=1.25in, left=1in, right=1in]{geometry}

\renewcommand{\P}{\mathbb{P}}
\newcommand{\E}{\mathbb{E}}
\newcommand{\Var}{\text{Var}}

\newcommand{\cN}{\mathcal{N}}

\newcommand{\Z}{\mathbb{Z}}
\newcommand{\R}{\mathbb{R}}

\newcommand{\N}{\mathbb{N}}
\renewcommand{\S}{\mathbb{S}}

\newcommand{\eps}{\varepsilon} 
\newcommand{\vphi}{\varphi}
\def\id{{\mathbf I}}

\newcommand{\<}{\langle}
\renewcommand{\>}{\rangle}

\newcommand{\op}{{\rm op}}
\newcommand{\ones}{\bm{1}}

\def\sT{{\mathsf T}}
\def\bzero{{\boldsymbol 0}}

\DeclareMathOperator*{\argmin}{arg\,min}

\newtheorem{theorem}{Theorem}
\newtheorem*{theorem*}{Theorem}
\newtheorem{lemma}{Lemma}

\newtheorem{assumption}{Assumption}
\newtheorem{definition}{Definition}
\newtheorem{proposition}{Proposition}

\theoremstyle{definition}
\newtheorem{example}{Example}

\newtheoremstyle{myremark} 
    {\topsep}                    
    {\topsep}                    
    {\rm}                        
    {}                           
    {\bf}                        
    {.}                          
    {.5em}                       
    {}  

\theoremstyle{myremark}
\newtheorem{remark}{Remark}[section]

\usepackage[toc,page]{appendix}
\usepackage{graphicx}
\usepackage{bbm}

\DeclareSymbolFont{rsfs}{U}{rsfs}{m}{n}
\DeclareSymbolFontAlphabet{\mathscrsfs}{rsfs}

\def\bA{{\boldsymbol A}}

\def\bL{{\boldsymbol L}}
\def\bM{{\boldsymbol M}}

\def\bW{{\boldsymbol W}}
\def\bX{{\boldsymbol X}}

\def\ba{{\boldsymbol a}}

\def\be{{\boldsymbol e}}

\def\bg{{\boldsymbol g}}

\def\bu{{\boldsymbol u}}
\def\bv{{\boldsymbol v}}
\def\bw{{\boldsymbol w}}
\def\bx{{\boldsymbol x}}
\def\by{{\boldsymbol y}}
\def\bz{{\boldsymbol z}}

\def\beps{{\boldsymbol \eps}}

\def\btheta{{\boldsymbol \theta}}

\def\bxi{{\boldsymbol \xi}}

\def\bDelta{{\boldsymbol \Delta}}

\def\bSigma{{\boldsymbol \Sigma}}
\def\bTheta{{\boldsymbol \Theta}}

\def\hf{{\hat f}}

\def\Sft{{\rm Sft}}

\def\KR{{\rm KR}}
\def\bbHe{{\rm He}}

\def\de{{\rm d}}

\def\lin{{\rm lin}}

\def\Coeff{{\rm Coeff}}
\def\de{{\rm d}}
\def\Unif{{\rm Unif}}
\def\lin{{\rm lin}}

\def\RF{{\rm RF}}
\def\NT{{\rm NT}}
\def\Cyc{{\rm Cyc}}

\def\Sft{{\rm Sft}}

\def\KR{{\rm KR}}

\def\bbHe{{\rm He}}

\def\cD{{\mathcal D}}
\def\cV{{\mathcal V}}
\def\cG{{\mathcal G}}
\def\cO{{\mathcal O}}
\def\cP{{\mathcal P}}

\def\cF{{\mathcal F}}
\def\cE{{\mathcal E}}
\def\cS{{\mathcal S}}
\def\cI{{\mathcal I}}

\def\cV{{\mathcal V}}
\def\cG{{\mathcal G}}
\def\cO{{\mathcal O}}
\def\cP{{\mathcal P}}

\def\cH{{\mathcal H}}
\def\cA{{\mathcal A}}

\def\T{{\mathbb T}}

\def\W{{\mathbb W}}

\def\Unif{{\sf Unif}}
\def\normal{{\sf N}}
\def\proj{{\mathsf P}}

\def\RF{{\sf RF}}
\def\NT{{\sf NT}}
\def\NN{{\sf NN}}
\def\reals{{\mathbb R}}
\def\integers{{\mathbb Z}}
\def\naturals{{\mathbb N}}

\def\normal{{\sf N}}
\def\proj{{\mathsf P}}

\def\T{{\mathbb T}}

\def\Unif{{\sf Unif}}
\def\normal{{\sf N}}
\def\Uop{{\mathbb U}}
\def\Hop{{\mathbb H}}

\def\proj{{\mathsf P}}

\def\RF{{\sf RF}}
\def\NT{{\sf NT}}
\def\NN{{\sf NN}}

\def\reals{{\mathbb R}}
\def\integers{{\mathbb Z}}
\def\naturals{{\mathbb N}}
\def\proj{{\mathsf P}}
\def\Hop{{\mathbb H}}
\def\Uop{{\mathbb U}}

\def\bz{{\boldsymbol z}}
\def\proj{{\mathsf P}}
\def\He{{\rm He}}
\def\cE{{\mathcal E}}

\def\bDelta{{\boldsymbol \Delta}}

\def\lin{{\rm lin}}

\def\be{{\boldsymbol e}}
\def\bu{{\boldsymbol u}}
\def\bg{{\boldsymbol g}}

\def\Var{{\rm Var}}

\def\bA{{\boldsymbol A}}

\def\bv{{\boldsymbol v}}
\def\bxi{{\boldsymbol \xi}}
\def\btheta{{\boldsymbol \theta}}
\def\bTheta{{\boldsymbol \Theta}}

\def\Coeff{{\rm Coeff}}

\def\RF{{\sf RF}}
\def\NT{{\sf NT}}
\def\NN{{\sf NN}}
\def\reals{{\mathbb R}}
\def\integers{{\mathbb Z}}
\def\naturals{{\mathbb N}}
\def\ttau{\tilde{\tau}}

\def\bw{{\boldsymbol w}}
\def\de{{\rm d}}
\def\bx{{\boldsymbol x}}
\def\by{{\boldsymbol y}}
\def\bW{{\boldsymbol W}}
\def\ba{{\boldsymbol a}}
\def\cF{{\mathcal F}}
\def\Unif{{\rm Unif}}

\def\bz{{\boldsymbol z}}
\def\proj{{\mathsf P}}
\def\He{{\rm He}}
\def\cE{{\mathcal E}}

\def\normal{{\sf N}}

\def\bDelta{{\boldsymbol \Delta}}

\def\lin{{\rm lin}}

\def\be{{\boldsymbol e}}
\def\bu{{\boldsymbol u}}
\def\bg{{\boldsymbol g}}
\def\btheta{{\boldsymbol \theta}}

\def\Coeff{{\rm Coeff}}

\def\RF{{\rm RF}}
\def\NT{{\rm NT}}
\def\bA{{\boldsymbol A}}

\def\cS{{\mathcal S}}
\def\Cyc{{\rm Cyc}}
\def\cI{{\mathcal I}}

\def\Hop{{\mathbb H}}
\def\Uop{{\mathbb U}}
\def\cX{{\mathcal X}}

\def\bproj{{\overline \proj}}
\def\quadratic{{\rm quad}}
\def\cube{{\rm cube}}

\def\bDelta{{\boldsymbol \Delta}}
\def\be{{\boldsymbol e}}
\def\bu{{\boldsymbol u}}
\def\bg{{\boldsymbol g}}

\def\Coeff{{\rm Coeff}}
\def\RF{{\rm RF}}
\def\NT{{\rm NT}}
\def\bA{{\boldsymbol A}}
\def\btheta{{\boldsymbol \theta}}
\def\bTheta{{\boldsymbol \Theta}}

\def\Tr{{\rm Tr}}

\def\cV{{\mathcal V}}

\def\bL{{\boldsymbol L}}

\def\osigma{\overline{\sigma}}

\def\hf{\hat f}

\def\cuH{\mathscrsfs{H}}

\def\osigma{\overline{\sigma}}

\def\oY{\overline{Y}}
\def\cH{\mathcal{H}}

\def\NTK{{\rm NTK}}

\def\Cube{{\mathscrsfs Q}}

\def\oproj{{\overline \proj}}
\def\tproj{{\overline \proj}}

\def\CNN{{\sf CNN}}
\def\data{{\rm data}}
\def\inv{{\rm inv}}

\def\noise{\sigma_{\varepsilon}}

\def\tQ{\tilde{Q}}

\def\evn{{\mathsf m}}
\def\evN{{\mathsf M}}

\def\lvn{{\mathsf s}}
\def\lvN{{\mathsf S}}

\def\seff{\mbox{\tiny\rm eff}}

\def\snew{\mbox{\tiny\rm new}}

\def\TwoCyc{{\rm Cyc2D}}

\def\osigma{\overline \sigma}
\def\cA{\mathcal{A}}

\def\SO{\rm SO}

\def\tg{\Tilde g}

\def\tbg{\Tilde \bg}

\def\barphi{\overline{\varphi}}

\def\barQ{\overline{Q}}

\def\barHe{\overline{\He}}
\def\ttau{\Tilde \tau}
\def\cY{\mathcal{Y}}

\usepackage{hyperref}
\hypersetup{
    colorlinks,
    linkcolor={blue!80!black},
    citecolor={green!50!black},
}
\colorlet{linkequation}{blue}

\begin{document}
\title{Learning with invariances in random features and kernel models}

\author{Song Mei\thanks{Department of Statistics, University of California, Berkeley},\;\; Theodor Misiakiewicz\thanks{Department of
    Statistics, Stanford University}, \;\;
  Andrea Montanari\footnotemark[2] \thanks{Department of Electrical Engineering,
    Stanford University} }

\maketitle

\begin{abstract}
A number of machine learning tasks entail a high degree of invariance: the data distribution does not change if we act on the data with a certain group of transformations. For instance, labels of images are invariant under translations of the images. Certain neural network architectures ---for instance, convolutional networks---are believed to owe their success to the fact that they exploit such invariance properties. With the objective of quantifying the gain achieved by invariant architectures, we introduce two classes of models: invariant random features and invariant kernel methods. The latter includes, as a special case, the neural tangent kernel for convolutional networks with global average pooling. We consider uniform covariates distributions on the sphere and hypercube and a general invariant target function. We characterize the test error of invariant methods in a high-dimensional regime in which the sample size and number of hidden units scale as polynomials in the dimension, for a class of groups that we call `degeneracy $\alpha$', with $\alpha \leq 1$. We show that exploiting invariance in the architecture saves a $d^\alpha$ factor ($d$ stands for the dimension) in sample size and number of hidden units to achieve the same test error as for unstructured architectures. 

Finally, we show that output symmetrization of an unstructured kernel estimator does not give a significant statistical improvement; on the other hand, data augmentation with an unstructured kernel estimator is equivalent to an invariant kernel estimator and enjoys the same improvement in statistical efficiency. 
\end{abstract}

\tableofcontents

\section{Introduction}

Consider the following image classification problem. We are given data $\{(\bx_i,y_i)\}_{i\le n}$ where $\bx_i\in \reals^d$ is an image, and $y_i\in\reals$ is its label.  We would like to learn a function $\hf:\reals^d\to \reals$ to predict labels of
new unseen images. Throughout this paper we will measure prediction error in terms of the square loss $R(\hf):=\E\{(y_{\snew}-\hf(\bx_{\snew}))^2\}$.

We can  think of $\bx\in \reals^d$ as a pixel representation of an image. For instance if this is a grayscale (one channel) two-dimensional
image, $\bx$ can represent the pixel values on a $d_1\times d_2$ grid with $d=d_1d_2$. For mathematical convenience, we here work with the cartoon example of one-dimensional `images' (or `signals') with $d$ pixels arranged on a line. Most of our results cover two-dimensional images as well. 

We assume a model whereby
the labels are $y_i=f_*(\bx_i)+\eps_i$, with noise $\eps_i$ independent of $\bx_i$ with $\E(\eps_i) = 0$ and $\E(\eps_i^2) = \sigma^2_{\eps}$.
In many applications, the target function $f_*$ is invariant under translations of the image: if $\bx'$ is obtained by translating image $\bx$, then $f_*(\bx')= f_*(\bx)$. We will consider here periodic
shifts (in the case of one-dimensional images): for $\bx\in \reals^d$, $g_{\ell}\cdot \bx:= (x_{\ell+1},\dots,x_d,x_1,\dots,x_{\ell})$ denotes its $\ell$-shift. Invariance implies $f_*(\bx) = f_*(g_{\ell}\cdot \bx)$ for all $\ell$ and $\bx$.

Convolutional neural networks are the state-of-the-art architecture for image classification and related computer vision tasks, and they are believed to exploit the translation invariance in a crucial way \cite{krizhevsky2012imagenet}. Consider the simple example of two-layer convolutional networks with global average pooling. The  network computes a nonlinear convolution of
$N$ filters $\bw_1,\dots,\bw_N$ with the image $\bx$. The results are then combined linearly with coefficients $a_1,\dots,a_N$:
\begin{align}
  f_{\CNN}(\bx) = \frac{1}{d}\sum_{i=1}^N a_i \sum_{\ell=1}^d\sigma(\< \bw_i, g_{\ell} \cdot \bx\>)\, .\label{eq:CNN}
\end{align}
This simple convolutional network can be compared with a standard fully-connected two-layer network with the same number of parameters: $f_{\NN}(\bx) = \sum_{i=1}^N
a_i\sigma(\< \bw_i, \bx\>)$. It is clear that ---when the target function $f_*$ is translation invariant---
the convolutional model $f_{\CNN}(\bx)$ is at least as powerful as $f_{\NN}(\bx)$ in terms of approximation, since it is invariant by construction (see Appendix \ref{sec:approx_invariant_network}
for a simple formal argument). 

The main objective of this paper is to quantify the advantage of  architectures ---such as convolutional ones--- that
enforce invariance. We are interested in characterizing the gain both  in approximation error and in generalization error.
We consider a general type of  invariance, defined by a group $\cG_d$ that is represented as a subgroup of
$\cO(d)$, the orthogonal group in $d$ dimensions. This means that each element $g\in\cG_d$ is identified with an orthogonal matrix
(which we will also denote by $g$), and group composition corresponds to matrix multiplication.  The group element $g\in\cG_d$ acts on
$\reals^d$ via $\bx\mapsto g\cdot\bx$. We will consider two simple distributions for the the signals $\bx$:
$\bx\sim\Unif(\S^{d-1}(\sqrt{d}))$ (the uniform distribution over the sphere in $d$ dimensions with radius $\sqrt{d}$)
and $\bx\sim\Unif(\Cube^d)$ (with $\Cube^d = \{+1,-1\}^d$ the discrete hypercube in $d$ dimensions).
We will write $(\cA_d,\tau_d)\in\{(S^{d-1}(\sqrt{d}),\Unif), (\Cube^d,\Unif)\}$ for either of these two probability spaces. In the case of $\cA_d = \Cube^d$, we will further require the action of $\cG_d$ to preserve $\Cube^d$.

In order to gain some insights on the behavior of actual neural networks, we consider
two classes of linear `overparametrized' models: invariant random features models and invariant kernel machines. We next describe these two approaches.

\vspace{0.2cm}

\noindent\emph{Invariant random feature models.}  Given an activation function $\sigma:\reals\to\reals$ and a group $\cG_d$
endowed with invariant (Haar) measure $\pi_d$, we define the invariant random features (RF) function class
\begin{equation}\label{eq:RF_inv}
\cF_{\RF,\inv}^N(\bW, \cG_d) = \Big\{ f(\bx) = \sum_{i=1}^N a_i \int_{\cG_d} \sigma (\< \bw_i, g \cdot \bx\>) \, \pi_d(\de g) : a_i \in \R, i \in [N] \Big\}\, .
\end{equation}
Here $\bW:=(\bw_1,\dots,\bw_N)$ is the set of first layer weights which are fixed and not optimized over. We draw them
randomly with $(\sqrt{d} \cdot \bw_i)_{i\le N}\sim_{iid}\Unif(\S^{d-1}(\sqrt{d}))$ or $\Unif(\Cube^d)$ depending on whether the feature vectors are $\bx_i\sim \Unif(\S^{d-1}(\sqrt{d}))$ or $\Unif(\Cube^d)$. 
If we let $\cG_d$ be the cyclic group $\Cyc_d:=\{g_0,g_1,\dots,g_{d-1}\}$ (here $g_{\ell}$ is the shift by $\ell$ positions), we obtain a random features
version of the convolutional network of Eq.~\eqref{eq:CNN}. Other examples will be presented in Section \ref{sec:Examples}.

Given data $\{(\bx_i,y_i)\}_{i\le n}$, we consider to fit the second-layer coefficients $(a_i)_{i\le N}$ in Eq. (\ref{eq:RF_inv}) using the random features ridge regression (RFRR). Notice that the estimated function
$\hf$ is invariant by construction, $\hf(\bx) = \hf(g\cdot \bx)$. We will denote the space of square integrable $\cG_d$-invariant functions
on $\cA_d\in\{\S^{d-1}(\sqrt{d}),\Cube^d\}$  by $L^2(\cA_d,\cG_d)$.

\vspace{0.2cm}

\noindent\emph{Invariant kernel machines.} We then consider kernel ridge regression (KRR) in the
reproducing kernel Hilbert space (RKHS) defined by a $\cG_d$-invariant kernel.
By this we mean a kernel $H\in L^2(\cA_d\times \cA_d)$ such that, for all $g,g'\in\cG_d$, the following folds for every $\bx_1,\bx_2$:
\begin{align}
  H(\bx_1,\bx_2) = H(g \cdot \bx_1,g'\cdot\bx_2)\, .\label{eq:Invariant}
\end{align}
Note that, as a consequence of this property, any function that is not in 
$L^2(\cA_d,\cG_d)$ (i.e. any function that is not invariant) has infinite RKHS norm: indeed this provides an alternate
characterization of invariant kernel methods.
Among  $\cG_d$-invariant kernels, we focus on the subclass that is obtained
by averaging an inner product kernel over the group $\cG_d$ 
\begin{align}
H_\inv(\bx_1,\bx_2) = \int_{\cG_d} h(\< \bx_1, g \cdot \bx_2\>/d) \, \pi_d(\de g). \label{eq:InvariantInnerProd}
\end{align}

Invariant kernel machines can be regarded as large-width ($N\to\infty$) limits of invariant random features methods.
Vice versa, the latter can be regarded as randomized approximations of invariant kernel methods. Moreover, invariant kernel methods also capture the large-width limits of other models, for instance, neural tangent models
associated to convolutional networks (c.f. Section \ref{sec:CNN_to_CNTK}).

\vspace{0.2cm}

We focus on a type of groups $\cG_d$ that we call \emph{groups of degeneracy $\alpha$}.
\begin{definition}[Groups of degeneracy $\alpha$]\label{def:degeneracy}
Let $V_{d,k}$ be the subspace of degree-$k$ polynomials that are orthogonal to polynomials of degree at most $(k-1)$
in $L^2(\cA_d)$, and denote by $V_{d,k}(\cG_d)$ the subspace of $V_{d,k}$ formed by polynomials that are $\cG_d$-invariant.
We say that $\cG_d$ has \emph{degeneracy $\alpha$} if for any integer $k \ge \alpha$ we have
$\dim(V_{d,k}/V_{d,k}(\cG_d)) \asymp d^\alpha$ (i.e., there exists $0 < c_k \le C_k < \infty$ such that $c_k \le \dim(V_{d,k}/V_{d,k}(\cG_d)) / d^\alpha \le C_k$ for any $d \ge 2$).
\end{definition}
This definition includes as special cases the cyclic group for one and two-dimensional signals (see Section
\ref{sec:Examples}), which have both degeneracy $1$.

We compare invariant methods to standard (non-invariant) random features models with inner product activation, defined as
\begin{equation}
\cF_{\RF}^N(\bW) = \Big\{ f(\bx) = \sum_{i=1}^N a_i \sigma (\< \bw_i,  \bx\>) : a_i \in \R, i \in [N] \Big\}\, ,\label{eq:RF_std}
\end{equation}
and standard inner product kernels $H ( \bx_1 , \bx_2) = h_d ( \<\bx_1 , \bx_2 \>/d)$.
For groups with degeneracy $\alpha \le 1$, we obtain a fairly complete characterization of the gain achieved by using invariant models, when the target function is an arbitrary invariant function $f_* \in L^2(\cA_d;\cG_d)$. 

\begin{description}
\item[Invariance gain: underparametrized case.] Consider the invariant RF class \eqref{eq:RF_inv} in the underparametrized
  regime $N\ll n$. We prove that the test error is dominated by the approximation error. Namely,
  if $d^{\ell-\alpha} \ll N \ll d^{\ell+1 - \alpha}$, then the test error (c.f. Eq. (\ref{eq:test_error_RFRR})) gives $R(f_*;\lambda)\approx \|\oproj_{>\ell}f_* \|_{L^2}^2$, where $\oproj_{>\ell}$ is the projection orthogonal to the subspace of degree $\ell$ polynomials. 
 In order to achieve the same risk, standard (non-invariant) RF models would require
  $d^{\ell} \ll N\ll d^{\ell + 1}$: invariance saves a $d^\alpha$ factor in the network width to achieve the same risk. 
\item[Invariance gain: overparametrized case.] Consider next the overparametrized
  regime $n\ll N$. In this case the test error is dominated by the statistical error. Namely,
  if $d^{\ell-\alpha}\ll n\ll d^{\ell + 1 - \alpha}$, then the test error gives $R(f_*;\lambda)\approx \|\oproj_{>\ell}f_* \|_{L^2}^2$.
 In order to achieve the same risk, standard (non-invariant) RF models would require
  $d^{\ell}\ll n\ll d^{\ell +1}$: invariance saves a $d^\alpha$ factor in the sample size to achieve the same risk. 
  \end{description}

  These results are precisely presented in Theorem \ref{thm:RFRRinvar} and summarized in Table \ref{tab:comparison}. We establish the same gain for invariant kernel methods in Theorem \ref{thm:invar_KRR}. While we focused in this paper on groups with degeneracy $\alpha \leq 1$ (which include our primary motivating examples, cyclic group in one or two dimensions), we expect similar results to hold for groups with $\alpha > 1$. We defer this to future work.

\begin{description}
\item[Output symmetrization and data augmentation.] Output symmetrization and data augmentation are two alternative approaches to incorporate invariances in machine learning models. We show that the performance of output symmetrization of standard KRR does not improve over standard KRR, and hence is sub-optimal compared to invariant KRR. On the other hand, it was shown that (c.f. \cite{li2019enhanced}) data augmentation is mathematically equivalent to invariant KRR for discrete groups. As a consequence, our theoretical results characterize the statistical gain by performing data augmentation.


\end{description}
It is important to mention that our treatment omits an important characteristic of convolutional
architectures: the fact that the filters $\bw_i$ of Eq.~\eqref{eq:CNN} have a short window size $q\ll d$.
Namely, they have only $q$ non-zero entries, for instance the first $q$ entries. Using short-window filters
has some interesting consequences, which can be investigated using the same approach developed here. We will report on these in a forthcoming article, and instead focus here on the impact of invariance. 

Our analysis is enabled by a simple yet important observation, which might generalize
to other settings. The subspaces $V_{d,k}$ of degree-$k$ polynomials (see Definition \ref{def:degeneracy})
are eigenspaces for inner product kernels. At the same time, they are preserved under the symmetry group $\cG_d$.
Namely, define  $f^{(g)}(\bx) = f(g\cdot\bx)$, we have $f^{(g)}\in V_{d,k}$ for any $f\in V_{d,k}$, $g\in\cG_d$. This
observation is crucial in determining the eigendecomposition of the relevant kernels. 

Let us finally emphasize, that the factor-$d$ gain in sample size for degeneracy-one groups is not correctly predicted by
a naive `data augmentation heuristics'.
The latter would suggest a gain of the order of $|\cG_d|$ or of
the size of orbits of $\cG_d$. As shown by the example of band limited functions (see below) $|\cG_d|$ can be $\infty$ but the degeneracy can still be one (and hence the gain is $d$).

\vspace{0.5cm}

\begin{table}[h!]
  \begin{center}
    \label{tab:table1}
    \begin{tabular}{|c|c|c|} 
      \hline
      To fit a degree $\ell$ polynomial&  Inner product random features & Invariant random features \\
      \hline
      \vtop{\hbox{\strut Underparameterized regime}\hbox{\strut \hspace{1.4cm} ($N \ll n$) }}
      & $ N \gg d^{\ell}$ & $N \gg d^{\ell-\alpha}$ \\
      \hline
        \vtop{\hbox{\strut Overparameterized regime}\hbox{\strut \hspace{1.4cm} ($n \ll N$) }}
      
       & $ n \gg d^{\ell}$ & $n \gg d^{\ell-\alpha}$  \\
      \hline
    \end{tabular}
        \caption{Sample size $n$ and number of features $N$ required to fit a $\cG_d$-invariant polynomial of degree $\ell$ using ridge regression with the standard random features model (Eq.~\eqref{eq:RF_std}) and the invariant random features model (Eq.~\eqref{eq:RF_inv}), for group $\cG_d$ of degeneracy $\alpha \leq 1$. 
        }\label{tab:comparison}
  \end{center}
\end{table}


\subsection{Related literature}

\vspace{0.1cm}

\noindent\textbf{Invariant function estimation}

A number of  mathematical works emphasized the role of invariance in neural network architectures. 
Among others, \cite{mallat2012group,bruna2013invariant,mallat2016understanding} propose architectures
(`deep scattering networks') that explicitly achieve invariance to a rich group of transformations.
However, these papers do not characterize the statistical error of these approaches. 

The recent paper \cite{li2020convolutional} constructs a simple data distribution on which a gap is proven
between the sample complexity for convolutional architectures, and the one for standard (fully connected) architectures.
This result differs from ours in several aspects. Most importantly,  we study the risk for estimating
general invariant functions using invariant kernels and random features, while \cite{li2020convolutional}
obtain results for a specific distribution using CNNs. Also, the weight sharing structure in \cite{li2020convolutional}
is different from the one  in Eq.~\eqref{eq:CNN}. 

Another work \cite{chen2020group} studied the statistical benefits of data augmentation in the parametric setting via a group theory framework. Our result is different in the sense that we consider the non-parametric setting to estimate an invariant function using kernel methods.

To the best of the our knowledge, our paper is the first that characterizes the precise statistical benefit of using invariant random features  and  kernel models.  

\vspace{0.3cm}

\noindent{\textbf{Convolutional neural networks and convolutional kernels}}


A recent line of work \cite{jacot2018neural, li2018learning, du2018gradient, du2018gradientb, allen2018convergence, allen2018learning, arora2019fine, zou2018stochastic, oymak2019towards} studied the training dynamics of overparametrized neural networks under certain random initialization, and showed that it converges to a kernel estimator, which corresponds to the ``neural tangent kernel". The convolutional neural tangent kernel, which corresponds to the tangent kernel of convolutional neural networks, was studied in \cite{arora2019exact,li2019enhanced, bietti2019inductive}. The connection between convolutional kernel ridge regression and data augmentation was pointed out in \cite{li2019enhanced}.

The network in Eq. (\ref{eq:CNN}) corresponds to a two-layer convolutional neural network with global average pooling, which is a special case of the convolutional network that was defined as in \cite{arora2019exact}. 

\vspace{0.3cm}

\noindent{\textbf{Random features and kernel methods}}


A number of authors have studied the generalization error of kernel machines
\cite{caponnetto2007optimal, jacot2020kernel, liang2020just, liang2019risk}
\cite[Theorem 13.17]{wainwright2019high} and random features models
\cite{rahimi2009weighted, rudi2017generalization, ma2020towards, bach2015equivalence}. However,  these results
are not fine-grained enough to characterize the separation between invariant kernels
(or random feature models) and standard  inner product kernels, for several  reasons. First, some of these results concern restricted target functions with bounded RKHS norm. Second, we establish a gap that holds pointwise, i.e. for any given target function $f_*$, while most of earlier work only obtain minimax lower bounds. Finally, we need the upper and lower bounds match up to a $1 + o_d(1)$ factor, while earlier results only match up to unspecified constants. 

The recent paper \cite{jacot2020kernel} provides sharp predictions for kernel machines, but it assumes that a certain random kernel matrix behaves like a random matrix with Gaussian components: proving an equivalence of this type  is the central mathematical challenge we face here.

Our analysis builds on the general results of \cite{ghorbani2019linearized,mei2021generalization}. In particular,
\cite{mei2021generalization} provides general conditions under which the risk of random features and kernel methods
can be characterized precisely. Checking these conditions for invariant
methods requires to prove certain concentration properties for the entries of
the relevant kernels.  We achieve this goal for the cyclic group with general activations,
and for degeneracy-$\alpha$ groups (for $\alpha \le 1$) with polynomial activations. Generalizing these results to other groups, data distributions, and activations is a promising direction.

\section{Examples}
\label{sec:Examples}

In this section, we provide three examples of our general setting. We show in Appendix \ref{sec:counting_degeneracy} that all these groups have degeneracy 1 and therefore satisfy the assumptions of our general theorems.

\begin{example}[One-dimensional images]\label{ex:OneDImages}
  The cyclic group has elements  $\Cyc_d = \{ g_0, g_1, \ldots, g_{d-1}\}$ where $g_{i}$ is a shift by $i$ pixels.
  For any $\bx = (x_1, \ldots, x_d)^\sT \in \cA_d$, the action of group element $g_i$ on $\bx$ is defined by $g_i \cdot \bx = (x_{i+1}, x_{i+2}, \ldots x_d, x_1, x_2, \ldots, x_{i})^\sT \in  \cA_d.$ (In particular, $g_i$ is identified with an orthogonal transfromation in $\reals^d$.) 
The measure $\pi_d$ is the uniform probability measure on $\Cyc_d$, i.e., 
\[
\int_{\Cyc_d} f(g) \pi_d(\de g) = \frac{1}{d} \sum_{i=0}^{d-1} f(g_i). 
\]
We will refer to the invariant functions $L^2(\cA_d, \Cyc_d)$ as the `cyclic functions'.
\end{example}

\begin{example}[Two-dimensional images]\label{ex:TwoDImages}
  Let $d = d_1 \times d_2$. We identify $\cX_{d_1 \times d_2} = \{ \bX \in \R^{d_1 \times d_2}: \| \bX \|_F^2 = d \}$
with $\S^{d-1}(\sqrt{d})$ (simply by `vectorizing' the matrix).  
  The two-direction cyclic group has elements 
$\TwoCyc_{d_1, d_2} = \{ g_{ij}: 0\le i <d_1, 0\le j <d_2, \}$.
For any $\bX = (X_{ij})_{i \in [d_1], j \in [d_2]} \in \cX_{d_1 \times d_2}$, the action of group element $g_{ij} \in \TwoCyc_{d_1, d_2}$ on $\bX$ is defined by 
\[
g_{ij} \cdot \bX = \begin{bmatrix}
X_{i+1, j+1} &  \ldots & X_{i+1, d_2} & X_{i+1, 1} & \ldots & X_{i+1, j} \\
\ldots &  \ldots & \ldots & \ldots & \ldots & \ldots \\
X_{d_1, j+1}  & \ldots & X_{d_1, d_2} & X_{d_1, 1} & \ldots & X_{d_1, j} \\
X_{1, j+1} &  \ldots & X_{1, d_2} & X_{1, 1} & \ldots & X_{1, j} \\
\ldots &  \ldots & \ldots & \ldots & \ldots & \ldots \\
X_{i, j+1}  & \ldots & X_{i, d_2} & X_{i, 1} & \ldots & X_{i, j} \\
\end{bmatrix}. 
\]
Again, this is an orthogonal transformation in $\cX_{d_1 \times d_2}\cong \S^{d-1}(\sqrt{d})$, and $\TwoCyc_{d_1, d_2}$ is isomorphic to a subgroup of $\cO(d)$. The measure $\pi_d$ is the uniform probability measure on $\TwoCyc_{d_1, d_2}$. We will refer to the invariant functions  $L^2(\S^{d-1}(\sqrt{d}), \TwoCyc_d)$ as the `two-direction cyclic functions'. 

\end{example}

\begin{example}[The translation invariant function class on band-limited signals]\label{ex:band-limited}
Suppose we have one-dimensional signals with very high resolution, but the signals are band-limited: their Fourier transforms have only $d$ non-zero coefficients. We assume that the labels of the band-limited signals are invariant under translations. The following model captures this setting. 

Let $\{ \vphi_j \}_{j \in [d]} \subseteq \cF([0, 1])$ be the real Fourier basis functions in $L^2([0, 1], \Unif)$. That is, we define $\vphi_1(t) = 1$, and for $p = 1, 2, \ldots, \lfloor d/2 \rfloor$ (we assume $d$ is odd), $\vphi_{2p}(t) = \sqrt{2} \cos(2 \pi p t)$, $\vphi_{2p+1}(t) = \sqrt{2} \sin(2 \pi pt )$. We define the band-limited covariate subspace $\W_d \subseteq L^2([0, 1], \Unif)$ to be ($\W$ stands for waves)
\[
\W_d =  \Big\{ x \in L^2([0, 1]): x(t) = \sum_{j = 1}^d \hat x_j \vphi_j(t), ~~ \hat \bx = (\hat x_1, \ldots, \hat x_d) \in \S^{d-1}(\sqrt d)  \Big\}. 
\]
Then the space $\W_d$ can be identified with the space $\S^{d-1}(\sqrt d)$. 

Let $\Sft_d = \{ g_u, u \in [0, 1]\} \simeq \SO(2)$ be the translation group that can act on $\W_d$. For any $x \in \W_d$, the action of group element $g_u \in \Sft_d$ on $x$ is defined by 
\[
[g_u \cdot x](t) = x(t - u).
\]
Equivalently, the action of group element $g_u \in \Sft_d$ on $\hat \bx \in \S^{d-1}(\sqrt d)$ is defined by
\[
g_u \cdot \hat \bx = (\hat x_1, \cos(2\pi u) \hat x_2 + \sin(2 \pi u) \hat x_3, -\sin(2 \pi u) \hat x_2 + \cos(2 \pi u) \hat x_3, \ldots ).
\]
That means, $\Sft_d$ can be interpreted as a subgroup of $\cO(d)$. The measure $\pi_d$ is the uniform distribution on $\Sft_d$, i.e., 
\[
\int_{\Sft_d} f(g) \pi_d(\de g) = \int_{[0, 1]} f(g_s) \de s. 
\]
The function class $L^2(\W_d, \Sft_d)$, or equivalently $L^2(\S^{d-1}(\sqrt d), \SO(2))$, can be regarded as the translation invariant function class on band-limited signals. 

\end{example}

\section{Invariant random feature models}

Let $\cG_d$ be a group of degeneracy $\alpha$ with $\alpha \le 1$ as defined in Definition \ref{def:degeneracy} and $f_d$ be a function that is invariant under the action of $\cG_d$, i.e., $f_d \in L^2(\cA_d, \cG_d)$. We consider fitting the data with the invariant random features model defined in Eq.~\eqref{eq:RF_inv} using ridge regression, which we call invariant RFRR. Namely, we learn a function $\hat f^{\inv}_{N, \lambda} ( \bx ; \hat \ba (
\lambda) ) = \sum_{1 \leq j  \leq N} \hat a_j \int_{\cG_d} \sigma (\< \bw_j, g \cdot \bx\>) \pi_d(\de g)$ with 
\begin{equation}\label{eq:RFRR_problem}
\hat \ba (\lambda) = \argmin_{\ba} \left\{ \sum_{i = 1}^n  \big( y_i - \hat f_{N,\lambda}^{\inv} ( \bx_i ; \ba ) \big)^2  + \frac{N \lambda}{ d^\alpha} \| \ba \|_2^2 \right\} \, ,
\end{equation}
where the regularization parameter $\lambda$ can depend on the dimension $d$. (The factor $d^\alpha$ in the ridge penalty is introduced to compensate for the effect of averaging the random features over $\cG_d$.) We further denote the test error of invariant RFRR by
\begin{equation}\label{eq:test_error_RFRR}
R_{\RF,\inv} (f_d, \bX, \bW, \lambda) := \E_\bx \Big[ \Big(f_d (\bx) - \hat f_{N,\lambda}^{\inv} ( \bx ; \hat \ba (
\lambda) ) \Big)^2 \Big] \, .
\end{equation}

We will make the following assumption on $\sigma$. 

\begin{assumption}[Conditions on $\sigma$, $n, N$, and $(\cA_d, \cG_d)$ at level $(\lvn , \lvN) \in \naturals^2$]\label{ass:activation_RC}
For $\sigma : \R \to \R$, we assume the following conditions hold.
\begin{itemize}
\item[(a)] For $(\cA_d , \cG_d ) = (\S^{d-1} (\sqrt{d}) , \Cyc_d)$, we assume $\sigma$ to be $(\min ( \lvn , \lvN)  + 1) \vee 3$ differentiable and there exists constants $c_0 >0$ and $c_1 < 1$ such that $|\sigma^{(k)}(u)| \le c_0 e^{c_1 u^2/2}$ for any $2 \le k \le (\min ( \lvn , \lvN) + 1) \vee 3$. Moreover, there exists an integer $p > 1/ \delta$ such that $n \leq N^{1 - \delta}$ or $N \leq n^{1 - \delta}$ and $| \sigma (x)|, | \sigma ' (x) | \leq c_0 \exp (c_1 x^2/ (8p))$. 

For general $(\cA_d , \cG_d)$, we assume that $\sigma$ is a (finite degree) polynomial function. 
\item[(b)] The Hermite coefficients $\mu_k(\sigma) \equiv \E_{G \sim \normal (0, 1)}[\sigma(G) \He_k(G)]$ verify $\mu_k(\sigma) \neq 0$ for any $0 \le k \le \min(\lvn , \lvN)$ (see Appendix \ref{sec:technical_background} for definitions).
\item[(c)] We assume that $\sigma$ is not a polynomial with degree less or equal to $\max(\lvn , \lvN)$.
\end{itemize}
\end{assumption}

For $k \in \naturals$, we denote by $\oproj_{\leq k} : L^2(\cA_d) \to L^2(\cA_d)$ the orthogonal projection operator onto the subspace of polynomials of degree at most $k$, and $\oproj_{> k} = \id - \oproj_{\leq k}$ (see Appendix \ref{sec:technical_background} for details). We denote $f(d) = o_{d,\P}(g(d))$ if $f(d) / g(d)$ converges to $0$ in probability as $d \to \infty$.

\begin{theorem}[Test error of invariant RFRR]\label{thm:RFRRinvar}
Let $\cG_d$ be a group of degeneracy $\alpha \le 1$ and let $\{ f_d \in L^2(\cA_d, \cG_d) \}_{d \ge 1}$ be a sequence of $\cG_d$-invariant functions. Assume $d^{\lvn - \alpha + \delta} \leq n \leq d^{\lvn + 1 - \alpha - \delta}$ and $ d^{\lvN -\alpha + \delta} \leq N \leq d^{\lvN + 1 - \alpha - \delta}$ for fixed integers $\lvn$, $\lvN$ and some $\delta>0$. Let $\sigma$ be an activation function that satisfies Assumption \ref{ass:activation_RC} at level $(\lvn, \lvN)$. Then the following hold for the test error of invariant RFRR (see Eq.~\eqref{eq:test_error_RFRR}):
\begin{itemize}
\item[(a)] (Overparametrized regime) Assume $N \geq n d^\delta$ for some $\delta>0$. Then for any regularization parameter $\lambda = O_d(1)$ (including $\lambda = 0$) and $\eta >0$, we have
\begin{align}
 R_{\RF,\inv} (f_d, \bX, \bW, \lambda) =&~   \| \oproj_{> \lvn} f_d \|_{L^2}^2 + o_{d,\P}(1) \cdot (  \|  f_d \|_{L^{2+\eta}}^2  +\noise^2 ).
\label{eq:RC_test_error_over}
\end{align}
\item[(b)] (Underparametrized regime) Assume $n \geq N d^\delta$ for some $\delta>0$. Then for any regularization parameter $\lambda = O_d(n/N)$ (including $\lambda = 0$) and any $\eta >0$, we have,
\begin{align}
 R_{\RF,\inv} (f_d, \bX, \bW, \lambda) =&~   \| \oproj_{> \lvN} f_d \|_{L^2}^2 + o_{d,\P}(1) \cdot (  \|  f_d \|_{L^{2+\eta}}^2  +\noise^2 ).
\label{eq:RC_test_error_under}
\end{align}
\end{itemize}
\end{theorem}

In particular, this theorem applies to the one-dimensional and two-dimensional cyclic groups, and band-limited functions listed in Section \ref{sec:Examples}. We refer readers to Appendix \ref{sec:key_intuition} for an informal intuition and Appendix \ref{sec:proof_RFRR} for the proof of this result.

We can compare these bounds with ridge regression on the standard random features model of Eq.~\eqref{eq:RF_std}. Theorem 2 in \cite{mei2021generalization} (with Assumption \ref{ass:activation_RC}) shows that the same test error holds as in Theorem \ref{thm:RFRRinvar} but with $d^{\lvn  + \delta} \leq n \leq d^{\lvn + 1 - \delta}$ and $d^{\lvN  + \delta} \leq N \leq d^{\lvN + 1 - \delta}$. We thus gain a factor $d^\alpha$ in the sample and feature complexity by using invariant features compared to non invariant ones.

\begin{remark}
  Assumption \ref{ass:activation_RC} requires the activation function to be polynomial, except for the cyclic group, for which only differentiability conditions are assumed. We believe that the differentiability condition (and indeed weaker conditions)
  should be sufficient for general groups. We defer these improvements to future work.
\end{remark}

\begin{remark}
  Consider two-dimensional images with $d = D \times D$ (Example \ref{ex:TwoDImages})
  and functions $f_d$ that are invariant with respect to the group of cyclic translations along the horizontal direction only.
  It can be shown that this group has degeneracy  $\alpha = 1/2$, and in fact $\dim(V_{d,k}/V_{d,k}(\cG_d))\asymp D = d^{1/2}$. Our theory also applies to this group. 

\end{remark}

\section{Invariant kernel machines}

Note that any invariant kernel of the form \eqref{eq:InvariantInnerProd} can be written as a kernel of the form:
\begin{equation}\label{eq:invar_kernel}
H_{d, \inv} ( \bx_1 , \bx_2 ) = \int_{\cG_d} \E_{\bw \sim \Unif (\S^{d-1})} \big[ \sigma ( \< \bx_1 , \bw \> ) \sigma ( \< \bx_2 , g \cdot \bw \> )  \big] \pi_d (\de g) \, .
\end{equation}
To see this, note that any inner product kernel $h$ can be decomposed as
\[
h(\< \bx_1,  \bx_2\>/d) =
\E_{\bw \sim \Unif (\S^{d-1})} \big[ \sigma ( \< \bx_1 , \bw \> ) \sigma ( \< \bx_2 , \bw \> )  \big]
\]
for some activation function $\sigma$, which amounts to taking the square root
of the positive semidefinite operator associated to
$h$. Substituting in Eq.~\eqref{eq:InvariantInnerProd}, we get the desired representation.

Consider Kernel ridge regression with regularization parameter
$\lambda$ associated to $H_{d,\inv}$, that we call invariant KRR. Namely, we learn a function $\hat f_\lambda^\inv ( \bx ; \hat \bu (
\lambda) ) = \sum_{i  \in [n]} \hat u_i H_{d,\inv} ( \bx_i , \bx)$ where 
\begin{equation}\label{eq:KRR_problem}
\hat \bu (\lambda) = \argmin_{\bu} \left\{ \sum_{i = 1}^n  \big( y_i - \hat f_\lambda^\inv ( \bx_i ; \bu ) \big)^2  + \frac{\lambda}{d^\alpha}  \| \hat f_\lambda^\inv ( \,\cdot\, ; \bu ) \|_{\cH}^2 \right\} \, .
\end{equation}
with $\| \cdot \|_{\cH}$ the RKHS norm associated to $H_{d,\inv}$.
We further denote the test error of invariant KRR by
\begin{equation}\label{eq:test_error_KRR}
R_{\KR, \inv} (f_d, \bX, \lambda) := \E_\bx \Big[ \Big(f_d (\bx) - \hat f_\lambda^\inv (\bx ; \hat \bu ) \Big)^2 \Big] \, .
\end{equation}
\begin{theorem}[Test error of invariant KRR]\label{thm:invar_KRR}
Let $\cG_d$ be a group of degeneracy $\alpha \le 1$ and $\{ f_d \in L^2(\cA_d, \cG_d) \}_{d \ge 1}$ be a sequence of $\cG_d$-invariant functions. Assume $d^{\lvn - \alpha + \delta} \le n \le d^{\lvn + 1 - \alpha - \delta}$ for some fixed integer $\lvn \ge 1$ and some $\delta > 0$. Let $\sigma$ be an activation function that satisfies Assumption \ref{ass:activation_RC} at level $(\lvn, \lvn)$ (and $N = \infty$) and let $H_{d,\inv}$ be the associated invariant kernel as defined in Eq.~\eqref{eq:invar_kernel}. Then, the following holds for the test error of invariant KRR (c.f. Eq.~\eqref{eq:test_error_KRR}): for any $\lambda = O_d(1)$ (including $\lambda = 0$ identically) any $\eta > 0$, we have
\begin{align}
 R_{\KR, \inv}(f_{d}, \bX, \lambda) =&~ \| \oproj_{> \lvn} f_d \|_{L^2}^2 + o_{d,\P}(1) \cdot (\|  f_d \|_{L^{2 + \eta}}^2 + \noise^2).
\label{eq:invarKRR_bound}
\end{align}
\end{theorem}

We can compare the performance of this kernel against a standard (inner product) kernel $H_d (\bx , \by ) = h_d ( \< \bx , \by \>/d)$. Then Theorem 4 in \cite{ghorbani2019linearized} shows that the above theorem holds but with $d^{\lvn  + \delta} \leq n \leq d^{\lvn + 1 - \delta}$. We gain a factor $d^\alpha$ in sample complexity by using an invariant kernel.

\begin{remark}
Recall that the neural tangent kernel (NTK) associated to a function $f ( \bx ; \bTheta )$ with random initialization $\bTheta_0$ is defined as
\[
H_{\NT} ( \bx , \by ) := \E_{\bTheta_0} \Big[ \<\nabla_\bTheta f ( \bx ; \bTheta_0 ), \nabla_\bTheta f ( \by ; \bTheta_0 ) \> \Big] \, .
\]
The neural tangent kernel associated to a multi-layers fully connected network is an inner-product kernel
(as long as the weights are initialized to be isotropic Gaussian.)
In contrast, the NTK associated to the CNN of Eq.~\eqref{eq:CNN} is an example of invariant kernel, and is covered
by Theorem \ref{thm:invar_KRR} (see Appendix \ref{sec:CNN_to_CNTK} for more details).
\end{remark}

\section{Comparison with alternative approaches}

To provide further context, it is useful to compare invariant random features and kernel models with other approaches.
Here we consider two alternatives:
$(i)$~\emph{output symmetrization}, which uses a non-invariant method for training and then
symmetrizes the estimated function over the group $\cG_d$ to obtain an invariant function;
$(ii)$~\emph{data augmentation}, which trains the model on a dataset augmented by samples obtained by applying group transformations to
the original data. As shown in  \cite{li2019enhanced}, data augmentation
is mathematically equivalent to invariant kernel methods, so that it is superior to standard kernel methods (with inner-product kernels). On the other hand, we show that output symmetrization of standard kernel estimators does not significantly improve over the standard kernel estimator, and is fundamentally sub-optimal comparing to invariant kernel methods.

\subsection{Output symmetrization}

Given an estimater $\hf$, the symmetrization operator $\cS \hf$ computes the average of $\hf$ over the group:
\begin{align}
  (\cS \hat f) (\bx) \equiv \int_{\cG_d} \hat f (g \cdot \bx)  \pi_d(\de g).
\end{align}
When the target function $f_d$ is $\cG_d$-invariant, one might naively think that the symmetrization operation will significantly improve the performance of standard kernel estimators (standard RFRR and KRR). Indeed, when $f_d \in L^2(\cA_d, \cG_d)$, Jensen's inequality gives $\| f_d - \cS \hf \|_{L^2}^2 = \| \cS(f_d - \hf) \|_{L^2}^2 \le \| f_d - \hf \|_{L^2}^2$. However, the proposition below (which is proved in Section \ref{sec:proof_output_symmetrization}) shows that $\cS \hf$ is not significantly better when $\hf$ is a standard kernel estimator. 

\begin{proposition}\label{prop:output_symmetrization}
Let $f_d \in L^2(\cA_d, \cG_d)$ be a sequence of target functions. For any sequence of estimators $\hf_d$ satisfying $\| \hf_d - \oproj_{\le \ell} f_d \|_{L^2}^2 \leq \eps$, we have
\begin{equation} 
\begin{aligned}
\| \oproj_{> \ell} f_d \|_{L^2}^2 - 2 \eps \| \oproj_{> \ell} f_d \|_{L^2}  \le&~ \| f_d - \cS \hf_d \|_{L^2}^2\\
 \le&~ \| f_d - \hf_d \|_{L^2}^2 \le \| \oproj_{> \ell} f_d \|_{L^2}^2 +  2 \eps \| \oproj_{> \ell} f_d \|_{L^2} + \eps^2. 
\end{aligned}
\end{equation}
\end{proposition}

Now consider ---to be definite--- a setting in which $N\ge nd^{\delta}$ and $d^{\ell+\delta}\le n\le d^{\ell+1-\delta}$, and $\cG_d$ is a group with degeneracy $1$. For any $f_d \in L^2(\cA_d, \cG_d)$ with $\|f_d \|_{L^{2+\eta}}^2 = O_d(1)$, the results of \cite{mei2021generalization} imply that standard RFRR (c.f. Eq. (\ref{eq:RF_std})) with sufficiently small regularization returns a function $\hf_\RF$ with $\| \oproj_{\le \ell} f_d - \hf_\RF\|_{L^2}^2 = o_{d, \P}(1)$. Consequently, Proposition \ref{prop:output_symmetrization} implies that we have 
\[
\| f_d - \cS \hat f_{\RF} \|_{L^2}^2 = \| f_d - \hat f_{\RF} \|_{L^2}^2 +o_{d, \P}(1) = \| \oproj_{>\ell} f_d \|_{L^2}^2 +o_{d, \P}(1),
\]
while Theorem \ref{thm:RFRRinvar} implies that invariant RFRR $\hat f_{\RF}^\inv$ with sufficiently small regularization achieves a substantially smaller risk: \[
\| f_d - \hat f_{\RF}^\inv \|_{L^2}^2 =  \| \oproj_{>\ell+1} f_d \|_{L^2}^2 +o_{d, \P}(1).
\]

 
 \subsection{Data augmentation}
 
 We consider full data augmentation whereby we replace each sample $(y_i , \bx_i)$ in the dataset by $|\cG_d|$ samples $\{(y_i , g \cdot \bx_i): g\in\cG_d\}$ (for simplicity we consider here the case of a finite group $\cG_d$), and perform standard KRR on the augmented dataset. One might naively think that this is not as effective as enforcing invariance in the kernel structure. After all, we are only requiring invariance to hold at the sampled points. However, \cite{li2019enhanced} showed that these two approaches are in fact equivalent.
 
 We compare KRR using the kernel $H(\bx,\by) = h(\< \bx, \by\> /d)$ on the augmented dataset, with invariant KRR on the original dataset using the symmetrized kernel $H_{\inv}(\bx, \by) = \int_{\cG_d} h(\< \bx, g \cdot \by\> /d) \pi_d(\de g)$. 
Denote by $\hf_{\lambda}^{\data}$ and $\hf_{\lambda}^{\inv}$ the KRR estimates with the standard kernel $H$ and full data augmentation, and with the invariant kernel $H_\inv$ respectively.

\begin{proposition}[\cite{li2019enhanced}]\label{prop:data_augmentation}
  Let $\cG$ be a finite group, and $H$, $H_\inv$ as defined above. Then we have
$\hf_{\lambda}^{\data}=\hf_{ \lambda}^{\inv}$.
\end{proposition}
A couple of remarks are in order. First, this equivalence is general (holds for any dataset $\{(y_i,\bx_i)\}_{i\le n}$),
and is in fact a consequence of the algebraic structure of ridge regressions.
Second, while this result establishes that the two approaches are mathematically equivalent, there are computational advantages for invariant KRR. Indeed, full data augmentation increases the size of the kernel matrix from $n$ to $n|\cG_d|$ which is  computationally more expensive. Finally, this equivalence shows that data augmentation with standard KRR is superior to output symmetrization of standard KRR.

\section{Numerical illustration}

\begin{figure}
\centering
\includegraphics[width = 1\linewidth]{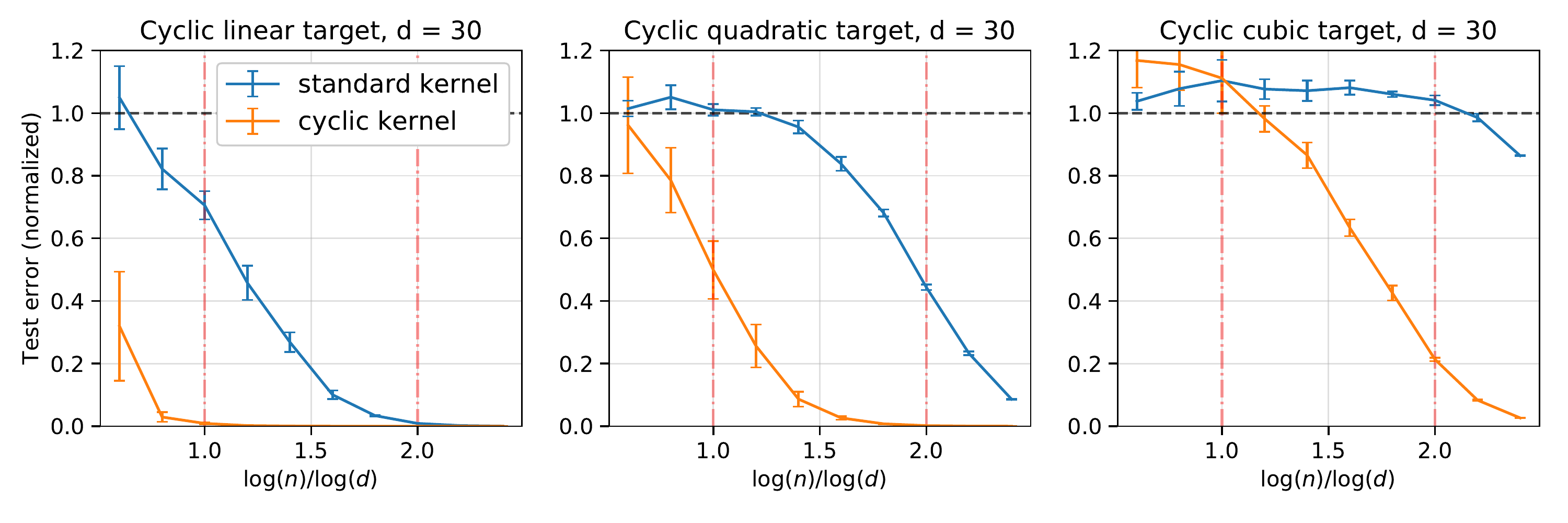}
\vspace{-1cm}
\caption{Learning cyclic polynomials (cf. Eq.~\eqref{eq:TargetPolynomials}) over the $d$-dimensional sphere, $d=30$, using KKR with a standard (inner-product) kernel and a cyclic invariant kernel, and regularization parameter $\lambda = 0^+$. We report the average and the standard deviation of the test error over 10 realizations, against the sample size $n$.}  
\label{fig:polynomials} 
\end{figure}

To check our predictions, we first consider the setting of $\bx \sim \Unif(\S^{d-1} (\sqrt{d}))$ with $d = 30$, and three cyclic invariant polynomials $f_{d, \lin}$, $f_{d, \quadratic}$, $f_{d, \cube} \in L^2(\S^{d-1}(\sqrt d), \Cyc_d)$ defined as
\begin{equation}\label{eq:TargetPolynomials}
f_{d, \lin} = \frac{1}{\sqrt d}\sum_{i = 1}^d x_i, ~~~~ f_{d, \quadratic} = \frac{1}{\sqrt d}\sum_{i = 1}^d x_i x_{i + 1}, ~~~~ f_{d, \cube} = \frac{1}{\sqrt d} \sum_{i = 1}^d x_i x_{i + 1} x_{i + 2}, 
\end{equation}
where the sub-index $i$ in $x_i$ should be understood in the modulo $d$ sense ($d + 1 = 1 ~ ({\rm mod }~d)$). We compare the performance between two kernels: a standard (inner product) kernel $H_d (\bx , \by) := h_d (\< \bx , \by \>/d)$ that we take to be the neural tangent kernel associated to a
depth-$5$ neural network with fully connected layers and ReLu activations $\sigma( x) = \max(x, 0)$.
We compare this with its cyclically invariant counterpart 
\[
H_{d,\Cyc} (\bx , \by) = \frac{1}{d} \sum_{0 \le i < d} h_d (\< \bx , g_i \cdot \by \>/d)\, ,
\]
 where $g_i \in \Cyc_d$ is the shift by $i$ positions as defined in Example \ref{ex:OneDImages}.
Note that the precise number of layers $L$ is not important. As long as  $L$ is fixed in the large $N,n$ limit,
our predictions remain unchanged, and the simulations appear to confirm this.

In Figure \ref{fig:polynomials}, we report the test errors of fitting each cyclic polynomials with KRR with the two kernels, and regularization parameter $\lambda =0^+$ (min-norm interpolation). We consider $\noise = 0$ and we report the risk averaged over 10 instances against the number of samples $n$. 
We observe that the risk in fitting $f_{d, \lin}$, $f_{d, \quadratic}$ and $f_{d, \cube}$, using
KRR with the cylcic kernel $H_{d,\Cyc}$ drops when 
$n =\Theta_d(1)$, $n= \Theta_d (d)$ and $n = \Theta_d(d^2)$ respectively. In contrast, the
risk of KRR with the standard kernel drops when $n =\Theta_d(d)$, $n= \Theta_d (d^2)$ and $n = \Theta_d(d^3)$
respectively. This matches well the predictions of  Theorem \ref{thm:invar_KRR}.

We next investigate the relevance of our results for real data.
We consider the MNIST dataset ($d= 28 \times 28 = 784$, $n_{{\rm train}} = 60000$, $n_{{\rm test}} = 10000$ and $10$ classes). We encoded class labels by $y_i\in\{ -4.5 , -3.5 , \ldots , 3.5 , 4.5\}$. We make these data invariant under cyclic translations in two dimensions (Example \ref{ex:TwoDImages}): for each samples in the training and test sets, we replace the image by a uniformly generated 2 dimensional (cyclic) translation of the image (see Fig.~\ref{fig:shifted_MNIST} in Appendix \ref{sec:MNIST_details}).
In this cyclic invariant MNIST data set, the labels are therefore invariant under the action of $\TwoCyc_{28,28}$. 

Images are highly anisotropic in pixel space $\reals^{784}$. In particular, directions corresponding to low-frequency
components of the Fourier transform of $\bx$ have significantly larger variance than directions corresponding to
high-frequency components.
Nevertheless, \cite{ghorbani2020neural}, showed that the analysis of random features and kernel models
of \cite{ghorbani2019linearized,mei2021generalization} extends to certain anisotropic models provided
the ambient dimension $d$ is replaced by a suitably defined effective dimension $d_{\seff}$.

In order to explore the role of data anisotropy, we pre-process images as follows. We compute the discrete Fourier
transform components of the images in the training set and select the $T \in \{ 20, 70, 120 , 200, 400, 784 \}$
components with the highest average absolute value. For each $T$, we then construct  training and test sets
in which we project each image onto the top $T$ frequencies (see Fig.~\ref{fig:DFT_MNIST} in Appendix \ref{sec:MNIST_details}).
When $T$ is small, we expect all the non-zero frequencies to have comparable variance and therefore
$d_{\seff}\approx T$. For larger $T$, we include frequencies of progressively small variance, and therefore $d_{\seff}$ should
saturate.

For each frequency content $T$, we compare the performance of two kernels: a standard inner-product kernel
$H_d (\bx , \by) := h_d (\< \bx , \by \>/d)$ and its cyclic counterpart given by 
\[
H_{d,\Cyc} (\bx , \by) = \frac{1}{28^2} \sum_{0 \le i,j < 28} h_d (\< \bx , g_{ij} \cdot \by \>/d)\, ,
\]
where $g_{ij} \in \TwoCyc_{28,28}$. We choose $H_d$ to be the neural tangent kernel associated to a two-layers neural network, and hence
$H_{d,\Cyc} $ is the one associated to a CNN analogous to \eqref{eq:CNN}
(but in two dimensions). We compute the KRR estimates with regularization parameter $\lambda = 0^+$. In Fig.~\ref{fig:ClassErrorMNIST}, we report the classification error averaged over 5 instances against the number of samples $\log (n) / \log (d)$.

We observe that the cyclic invariant kernel vastly outperform the inner product kernel:
the same test error is achieved at a significantly smaller sample size, in qualitative agreement with our general theory.
In order to quantify this gap, for each $T$ we fit two curves to the test error of the two kernels,
which differ uniquely in an horizontal shift (see Appendix \ref{sec:MNIST_details}). We estimate the sample complexity gain by
the difference between these shifts, and denote this estimate by $d_{\seff}$.

It is visually clear that $d_{\seff}$ increases with $T$, as expected. We
plot $d_{\seff}$ as a function of $T$ in Fig.~\ref{fig:EffDimMNIST} in Appendix \ref{sec:MNIST_details}. We observe that the behavior of $d_{\seff}$ roughly matches our expectations:
it grows linearly at small $T$ and eventually saturates.

\begin{figure}
\centering
\includegraphics[width = .87\linewidth]{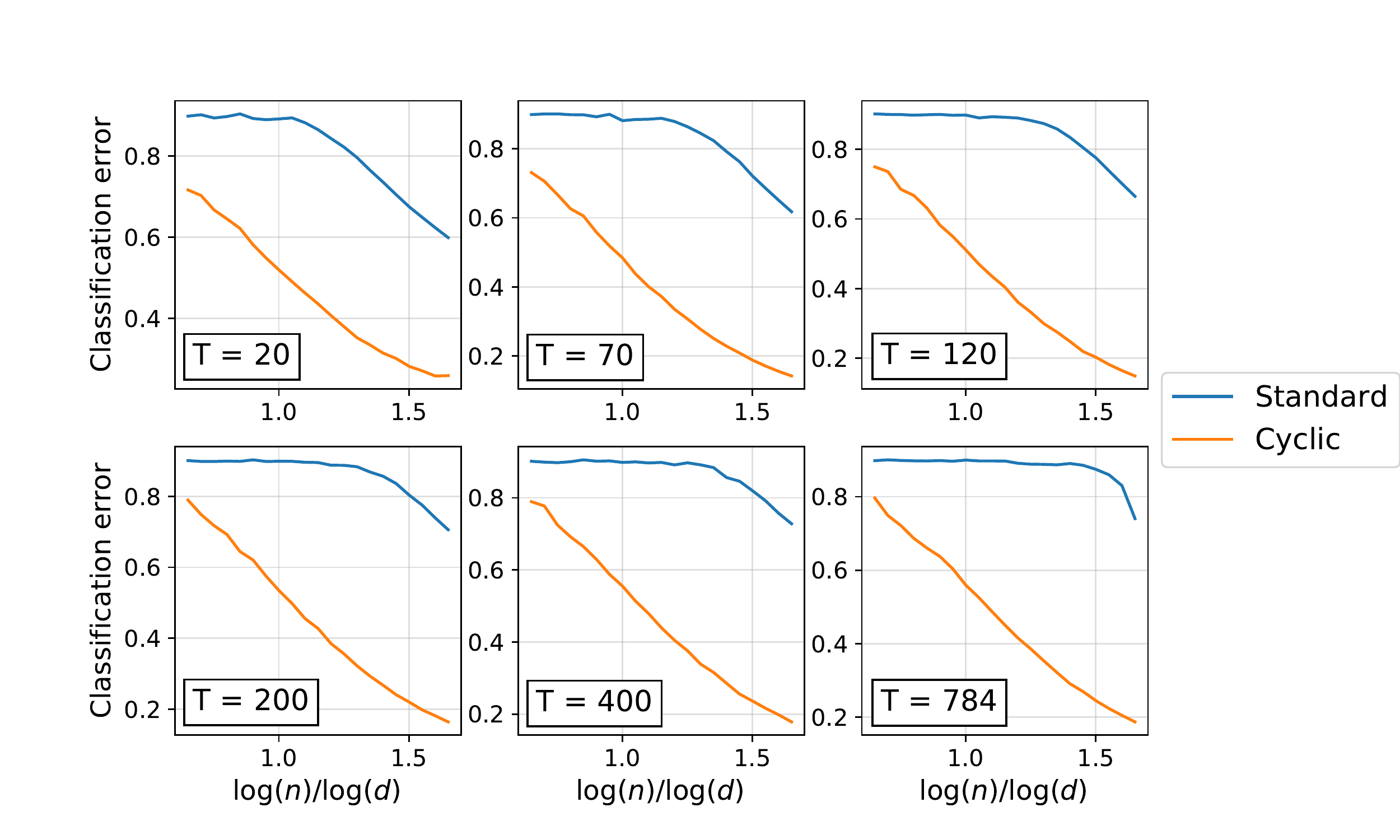}
\vspace{-0.7cm}
\caption{Classification error   for the cyclic invariant
  MNIST dataset. For each  frequencies content $T$, we plot the classification error averaged over $5$ instances against the number of samples $\log(n)/\log(d)$,
  for KRR using a standard (inner-product) kernel and a cyclic invariant kernel and regularization parameter $\lambda = 0^+$.}  
\label{fig:ClassErrorMNIST} 
\end{figure}

\section*{Acknowledgments}

This work was supported by NSF through award DMS-2031883 and from the Simons Foundation through Award 814639 for the
Collaboration on the Theoretical Foundations of Deep Learning
We also acknowledge  NSF grants CCF-2006489, IIS-1741162 and the ONR
grant N00014-18-1-2729.

\bibliographystyle{amsalpha}
\bibliography{all-bibliography.bbl}

\newcommand{\etalchar}[1]{$^{#1}$}
\providecommand{\bysame}{\leavevmode\hbox to3em{\hrulefill}\thinspace}
\providecommand{\MR}{\relax\ifhmode\unskip\space\fi MR }
\providecommand{\MRhref}[2]{%
  \href{http://www.ams.org/mathscinet-getitem?mr=#1}{#2}
}
\providecommand{\href}[2]{#2}
\begin{thebibliography}{MWW{\etalchar{+}}20}

\bibitem[ADH{\etalchar{+}}19a]{arora2019fine}
Sanjeev Arora, Simon Du, Wei Hu, Zhiyuan Li, and Ruosong Wang,
  \emph{Fine-grained analysis of optimization and generalization for
  overparameterized two-layer neural networks}, Proceedings of the 36th
  International Conference on Machine Learning, Proceedings of Machine Learning
  Research, vol.~97, PMLR, 09--15 Jun 2019, pp.~322--332.

\bibitem[ADH{\etalchar{+}}19b]{arora2019exact}
Sanjeev Arora, Simon~S Du, Wei Hu, Zhiyuan Li, Russ Salakhutdinov, and Ruosong
  Wang, \emph{On exact computation with an infinitely wide neural net},
  Advances in Neural Information Processing Systems, 2019, pp.~8139--8148.

\bibitem[AZLL18]{allen2018learning}
Zeyuan Allen-Zhu, Yuanzhi Li, and Yingyu Liang, \emph{Learning and
  generalization in overparameterized neural networks, going beyond two
  layers}, arXiv:1811.04918 (2018).

\bibitem[AZLS19]{allen2018convergence}
Zeyuan Allen-Zhu, Yuanzhi Li, and Zhao Song, \emph{On the convergence rate of
  training recurrent neural networks}, Advances in Neural Information
  Processing Systems, 2019, pp.~6676--6688.

\bibitem[Bac15]{bach2015equivalence}
Francis Bach, \emph{On the equivalence between quadrature rules and random
  features}, arXiv preprint arXiv:1502.06800 (2015), 135.

\bibitem[BB20]{bietti2020deep}
Alberto Bietti and Francis Bach, \emph{Deep equals shallow for relu networks in
  kernel regimes}, arXiv preprint arXiv:2009.14397 (2020).

\bibitem[Bec75]{beckner1975inequalities}
William Beckner, \emph{{Inequalities in Fourier analysis}}, Annals of
  Mathematics (1975), 159--182.

\bibitem[Bec92]{beckner1992sobolev}
\bysame, \emph{{Sobolev inequalities, the Poisson semigroup, and analysis on
  the sphere $S^n$}}, Proceedings of the National Academy of Sciences
  \textbf{89} (1992), no.~11, 4816--4819.

\bibitem[BM13]{bruna2013invariant}
Joan Bruna and St{\'e}phane Mallat, \emph{Invariant scattering convolution
  networks}, IEEE transactions on pattern analysis and machine intelligence
  \textbf{35} (2013), no.~8, 1872--1886.

\bibitem[BM19]{bietti2019inductive}
Alberto Bietti and Julien Mairal, \emph{On the inductive bias of neural tangent
  kernels}, arXiv preprint arXiv:1905.12173 (2019).

\bibitem[Bon70]{bonami1970etude}
Aline Bonami, \emph{{Etude des coefficients de Fourier des fonctions de
  $L^p(G)$}}, Annales de l'institut Fourier, vol.~20, 1970, pp.~335--402.

\bibitem[CDL20]{chen2020group}
Shuxiao Chen, Edgar Dobriban, and Jane~H Lee, \emph{A group-theoretic framework
  for data augmentation}, Journal of Machine Learning Research \textbf{21}
  (2020), no.~245, 1--71.

\bibitem[CDV07]{caponnetto2007optimal}
Andrea Caponnetto and Ernesto De~Vito, \emph{Optimal rates for the regularized
  least-squares algorithm}, Foundations of Computational Mathematics \textbf{7}
  (2007), no.~3, 331--368.

\bibitem[DLL{\etalchar{+}}18]{du2018gradientb}
Simon~S Du, Jason~D Lee, Haochuan Li, Liwei Wang, and Xiyu Zhai, \emph{Gradient
  descent finds global minima of deep neural networks}, arXiv:1811.03804
  (2018).

\bibitem[DZPS19]{du2018gradient}
Simon~S Du, Xiyu Zhai, Barnabas Poczos, and Aarti Singh, \emph{Gradient descent
  provably optimizes over-parameterized neural networks}, International
  Conference on Learning Representations, 2019.

\bibitem[GMMM19]{ghorbani2019linearized}
Behrooz Ghorbani, Song Mei, Theodor Misiakiewicz, and Andrea Montanari,
  \emph{Linearized two-layers neural networks in high dimension}, arXiv
  preprint arXiv:1904.12191 (2019).

\bibitem[GMMM20]{ghorbani2020neural}
\bysame, \emph{When do neural networks outperform kernel methods?}, arXiv
  preprint arXiv:2006.13409 (2020).

\bibitem[Gro75]{gross1975logarithmic}
Leonard Gross, \emph{Logarithmic sobolev inequalities}, American Journal of
  Mathematics \textbf{97} (1975), no.~4, 1061--1083.

\bibitem[JGH18]{jacot2018neural}
Arthur Jacot, Franck Gabriel, and Cl{\'e}ment Hongler, \emph{Neural tangent
  kernel: Convergence and generalization in neural networks}, Advances in
  Neural Information Processing Systems, 2018, pp.~8580--8589.

\bibitem[J{\c{S}}S{\etalchar{+}}20]{jacot2020kernel}
Arthur Jacot, Berfin {\c{S}}im{\c{s}}ek, Francesco Spadaro, Cl{\'e}ment
  Hongler, and Franck Gabriel, \emph{Kernel alignment risk estimator: Risk
  prediction from training data}, arXiv preprint arXiv:2006.09796 (2020).

\bibitem[KSH12]{krizhevsky2012imagenet}
Alex Krizhevsky, Ilya Sutskever, and Geoffrey~E Hinton, \emph{Imagenet
  classification with deep convolutional neural networks}, Advances in Neural
  Information Processing Systems, 2012, pp.~1097--1105.

\bibitem[LL18]{li2018learning}
Yuanzhi Li and Yingyu Liang, \emph{Learning overparameterized neural networks
  via stochastic gradient descent on structured data}, Advances in Neural
  Information Processing Systems, 2018, pp.~8168--8177.

\bibitem[LR{\etalchar{+}}20]{liang2020just}
Tengyuan Liang, Alexander Rakhlin, et~al., \emph{Just interpolate: Kernel
  ``ridgeless'' regression can generalize}, Annals of Statistics \textbf{48}
  (2020), no.~3, 1329--1347.

\bibitem[LRZ19]{liang2019risk}
Tengyuan Liang, Alexander Rakhlin, and Xiyu Zhai, \emph{On the risk of
  minimum-norm interpolants and restricted lower isometry of kernels}, arXiv
  preprint arXiv:1908.10292 (2019).

\bibitem[LWY{\etalchar{+}}19]{li2019enhanced}
Zhiyuan Li, Ruosong Wang, Dingli Yu, Simon~S Du, Wei Hu, Ruslan Salakhutdinov,
  and Sanjeev Arora, \emph{Enhanced convolutional neural tangent kernels},
  arXiv preprint arXiv:1911.00809 (2019).

\bibitem[LZA20]{li2020convolutional}
Zhiyuan Li, Yi~Zhang, and Sanjeev Arora, \emph{Why are convolutional nets more
  sample-efficient than fully-connected nets?}, arXiv preprint arXiv:2010.08515
  (2020).

\bibitem[Mal12]{mallat2012group}
St{\'e}phane Mallat, \emph{Group invariant scattering}, Communications on Pure
  and Applied Mathematics \textbf{65} (2012), no.~10, 1331--1398.

\bibitem[Mal16]{mallat2016understanding}
\bysame, \emph{Understanding deep convolutional networks}, Philosophical
  Transactions of the Royal Society A: Mathematical, Physical and Engineering
  Sciences \textbf{374} (2016), no.~2065, 20150203.

\bibitem[MMM21]{mei2021generalization}
Song Mei, Theodor Misiakiewicz, and Andrea Montanari, \emph{Generalization
  error of random features and kernel methods: hypercontractivity and kernel
  matrix concentration}, arXiv preprint arXiv:2101.10588 (2021).

\bibitem[MWW{\etalchar{+}}20]{ma2020towards}
Chao Ma, Stephan Wojtowytsch, Lei Wu, et~al., \emph{Towards a mathematical
  understanding of neural network-based machine learning: what we know and what
  we don't}, arXiv preprint arXiv:2009.10713 (2020).

\bibitem[O'D14]{o2014analysis}
Ryan O'Donnell, \emph{Analysis of boolean functions}, Cambridge University
  Press, 2014.

\bibitem[OS19]{oymak2019towards}
Samet Oymak and Mahdi Soltanolkotabi, \emph{Towards moderate
  overparameterization: global convergence guarantees for training shallow
  neural networks}, arXiv preprint arXiv:1902.04674 (2019).

\bibitem[RR09]{rahimi2009weighted}
Ali Rahimi and Benjamin Recht, \emph{Weighted sums of random kitchen sinks:
  Replacing minimization with randomization in learning}, Advances in neural
  information processing systems, 2009, pp.~1313--1320.

\bibitem[RR17]{rudi2017generalization}
Alessandro Rudi and Lorenzo Rosasco, \emph{Generalization properties of
  learning with random features}, Advances in Neural Information Processing
  Systems, 2017, pp.~3215--3225.

\bibitem[RV{\etalchar{+}}13]{rudelson2013hanson}
Mark Rudelson, Roman Vershynin, et~al., \emph{Hanson-wright inequality and
  sub-gaussian concentration}, Electronic Communications in Probability
  \textbf{18} (2013).

\bibitem[Wai19]{wainwright2019high}
Martin~J Wainwright, \emph{High-dimensional statistics: A non-asymptotic
  viewpoint}, vol.~48, Cambridge University Press, 2019.

\bibitem[ZCZG18]{zou2018stochastic}
Difan Zou, Yuan Cao, Dongruo Zhou, and Quanquan Gu, \emph{Stochastic gradient
  descent optimizes over-parameterized deep relu networks}, arXiv preprint
  arXiv:1811.08888 (2018).

\end{thebibliography}

\newpage

\appendix

\section{Some details in the main text}\label{sec:miscell}

\subsection{Approximation power of invariant networks}\label{sec:approx_invariant_network}

In the proposition below, we show that the approximation power of two-layers $\cG_d$-invariant neural networks are always no worse than two-layers fully-connected neural networks when the target function is $\cG_d$-invariant. 

\begin{proposition}\label{prop:approx_invariant_network}
Let $\sigma \in C(\R)$ be an activation function. Let $\cA_d \in \{ \S^{d-1}(\sqrt{d}), \Cube^d \}$. Let $\cG_d$ be a subgroup of $\cO(d)$ that preserves $\cA_d$. Let $\pi_d$ be the Haar measure of $\cG_d$. Let $f_* \in L^2(\cA_d; \cG_d)$ be a $\cG_d$-invariant function. Define the function classes of two-layers invariant neural networks and two-layers fully-connected neural networks by
\begin{align}
  \cF_{\NN, \cG_d, N} =&~ \Big\{ f(\bx) = \sum_{i=1}^N a_i \int_{\cG_d} \sigma(\< \btheta_i, g_{\ell} \cdot \bx\> / \sqrt{d}) \pi_d(\de g): \btheta_i \in \cA_d, a_i \in \R \Big\}, \\
 \cF_{\NN, N} =&~ \Big\{ f(\bx) = \sum_{i=1}^N
a_i\sigma(\< \btheta_i, \bx\> / \sqrt{d}) : \btheta_i \in \cA_d, a_i \in \R \Big\}.
\end{align}
Then we have 
\[
\inf_{f \in \cF_{\NN, \cG_d, N}} \| f_* - f \|_{L^2}^2 \le \inf_{f \in \cF_{\NN, N}} \| f_* - f \|_{L^2}^2.
\]
\end{proposition}

\begin{proof}[Proof of Proposition \ref{prop:approx_invariant_network}]

We define the symmetrization operator $\cS: L^2(\cA_d) \to L^2(\cA_d; \cG_d)$ by 
\[
(\cS f) (\bx) = \int_{\cG_d} f(g \cdot \bx) \pi_d(\de g). 
\]
Since $f_* \in L^2(\cA_d; \cG_d)$, by Jensen's inequality, for any $f \in L^2(\cA_d)$, we have 
\[
\| f_\star - \cS f \|_{L^2}^2 = \| \cS (f_\star -  f) \|_{L^2}^2 \le \| f_\star - f \|_{L^2}^2. 
\]
Moreover, for any $f \in \cF_{\NN, N}$, we have $\cS f \in \cF_{\NN, \cG_d, N}$. This gives 
\[
\inf_{f \in \cF_{\NN, \cG_d, N}} \| f_* - f \|_{L^2}^2 \le \inf_{f \in \cF_{\NN, N}} \| f_\star - \cS f \|_{L^2}^2 \le \inf_{f \in \cF_{\NN, N}} \| f_\star - f \|_{L^2}^2. 
\]
This concludes the proof. 
\end{proof}

\subsection{Intuition for the proofs of Theorems \ref{thm:RFRRinvar} and \ref{thm:invar_KRR}}\label{sec:key_intuition}
 
 Theorem \ref{thm:RFRRinvar} and \ref{thm:invar_KRR} are consequences of general theorems proved in \cite{mei2021generalization}. The $d^\alpha$ improvement between invariant and non-invariant models can be understood as follows: consider an inner-product activation $\sigma ( \< \bx , \btheta \> / \sqrt{d})$ with $\bx , \btheta \sim \Unif (\cA_d)$ (where we denoted $\btheta = \sqrt{d} \cdot \bw$), then we have the following eigendecomposition
\[
\sigma ( \< \bx , \btheta \> / \sqrt{d} ) = \sum_{k = 0}^\infty \xi_{d,k}^2 \sum_{l = 1 }^{B(\cA_d ; k)} Y^{(d)}_{kl} ( \bx ) Y^{(d)}_{kl} ( \btheta ) \, ,
\]
where $\{ Y^{(d)}_{kl} \}_{l \in [B(\cA_d ; k)]}$ form an orthonormal basis of $V_{d,k}$, the subspace of degree-$k$ polynomials on $\cA_d$ (see Section \ref{sec:technical_background} for background on functional spaces on the sphere and hypercube). The eigenvalues of $\sigma$ are given by $\{ \xi_{d,k} \}_{k \ge 0}$ with each having degeneracy $B(\cA_d ; k)$.

As mentioned in the introduction, the symmetry group $\cG_d$ preserves $V_{d,k}$ (see Section \ref{sec:invariant_polynomials}) and the invariant activation function has the following eigendecomposition
\[
\osigma ( \bx ; \btheta ) := \int_{\cG_d} \sigma ( \< \bx , g \cdot \btheta \> / \sqrt{d} ) \pi_d ( \de g ) = \sum_{k = 0}^\infty \xi_{d,k}^2 \sum_{l = 1 }^{D (\cA_d ; k)} \oY^{(d)}_{kl} ( \bx ) \oY^{(d)}_{kl} ( \btheta ) \, ,
\]
where the $\{ \oY^{(d)}_{kl} \}_{l \in [B(\cA_d ; k)]}$ form an orthonormal basis of $V_{d,k} ( \cG_d)$, the subspace of degree-$k$ invariant polynomials on $\cA_d$. The eigenvalues of $\osigma$ are given by $\{ \xi_{d,k} \}_{k \ge 0}$ with each having degeneracy $D(\cA_d ; k)$.

Hence $\osigma$ has the same eigenvalues $\xi_{d,k}$ as $\sigma$, but with degeneracy smaller by a factor 
\[
\frac{B(\cA_d ; k) }{ D( \cA_d ; k) } = \Theta_d ( d^\alpha ) \, .
\]
In other words, in order to fit degree $\ell$ polynomials using invariant methods, one needs to fit a factor $d^\alpha$ less eigendirections, which translates to a factor $d^\alpha$ improvement in the sample and features complexity. 

This intuition is verified rigorously in the proof of these theorems in Appendix \ref{sec:proofs_main}.

\subsection{Convolutional neural tangent kernel}\label{sec:CNN_to_CNTK}

\begin{proposition}\label{prop:CNN_to_CNTK}
Let $\sigma \in C^1(\R)$ be an activation function. Let $\cG_d$ be a discrete subgroup of $\cO(d)$ with Haar measure $\pi_d$. Let $f_N$ be an invariant neural network
\[
f_N(\bx; \bTheta) = \sum_{i=1}^N a_i \int_{\cG_d} \sigma(\< \bw_i, g \cdot \bx\>) \pi_d(\de g). 
\]
Let $a_i^0 \sim_{i.i.d.} \cN(0, 1)$ and $\bw_i^0 \sim_{i.i.d.} \Unif(\S^{d-1})$ independently, and $\bTheta^0 = (a_1^0, \ldots, a_N^0, \bw_1^0, \ldots, \bw_N^0)$. Then there exists $h_d: [-1, 1] \to \R$, such that for any $\bx, \by \in \S^{d-1}(\sqrt{d})$, we have almost surely
\[
\lim_{N\to\infty}\< \nabla_\bTheta f_N(\bx; \bTheta^0), \nabla_\bTheta f_N(\by; \bTheta^0)\> / N = \int_{\cG_d} h_d(\< \bx, g \cdot \by\>/d) \pi_d(\de g). 
\]
\end{proposition}

\begin{proof}[Proof of Proposition \ref{prop:CNN_to_CNTK}] For $\bx, \by \in \S^{d-1}(\sqrt{d})$, define
\[
\begin{aligned}
h_d^{(1)}(\< \bx, \by\>/d) =& \E_{\bw \sim \Unif(\S^{d-1})}[\sigma(\< \bw, \bx\>) \sigma(\< \bw, \by\>)],\\
h_d^{(2)}(\< \bx, \by\>/d) =& \E_{\bw \sim \Unif(\S^{d-1})}[\sigma'(\< \bw, \bx\>) \sigma'(\< \bw, \by\>)\< \bx, \by\>].
\end{aligned}
\]
By the technical backgrounds in Section \ref{sec:technical_background}, we can see that $h_d^{(1)}$ and $h_d^{(2)}$ can be well-defined. 

Calculating the derivative of the neural network with respect to $\ba = (a_1, \ldots, a_N)$, we have 
\[
\frac{1}{N}\<\nabla_\ba f ( \bx ; \bTheta^0 ), \nabla_\ba f ( \by ; \bTheta^0 )\> = \int_{\cG_d \times \cG_d} \frac{1}{N}\sum_{i = 1}^N \Big[ \sigma ( \< \bw_i ,g \cdot \bx  \> ) \sigma ( \< \bw_i , g' \cdot \by \> ) \Big] \pi_d(\de g) \pi_d(\de g'). 
\]
Since $\cG_d$ is a discrete group, by law of large numbers, we have 
\[
\begin{aligned}
&~ \lim_{N \to \infty} \frac{1}{N}\<\nabla_\ba f ( \bx ; \bTheta^0 ), \nabla_\ba f ( \by ; \bTheta^0 )\>\\
=&~ \int_{\cG_d \times \cG_d} \E_\bw[ \sigma ( \< \bw ,g \cdot \bx  \> ) \sigma ( \< \bw , g' \cdot \by \> )] \pi_d(\de g) \pi_d(\de g')\\
=&~ \int_{\cG_d \times \cG_d} h_d^{(1)}(\< g \cdot \bx, g' \cdot \by \>/d ) \pi_d(\de g) \pi_d(\de g')\\
=&~ \int_{\cG_d} h_d^{(1)}(\< \bx, g \cdot \by \>/d ) \pi_d(\de g). 
\end{aligned}
\]
Moreover, calculating the derivative of the neural network with respect to $\bW = (\bw_1, \ldots, \bw_N)$, we have 
\[
\begin{aligned}
&~\frac{1}{N}\<\nabla_\bW f ( \bx ; \bTheta^0 ), \nabla_\bW f ( \by ; \bTheta^0 )\> \\
=&~ \int_{\cG_d \times \cG_d} \frac{1}{N}\sum_{i = 1}^N \Big[ (a_i^0)^2\sigma' ( \< \bw_i ,g \cdot \bx  \> ) \sigma' ( \< \bw_i , g' \cdot \by \> ) \< g \cdot \bx, g' \cdot \by\> \Big]\pi_d(\de g) \pi_d(\de g')
\end{aligned}
\]
Since $\cG_d$ is a discrete group, by law of large numbers, we have 
\[
\begin{aligned}
&~ \lim_{N \to \infty} \frac{1}{N}\<\nabla_\bW f ( \bx ; \bTheta^0 ), \nabla_\bW f ( \by ; \bTheta^0 )\>\\
=&~ \int_{\cG_d \times \cG_d} \E_\bw[ \sigma' ( \< \bw ,g \cdot \bx  \> ) \sigma' ( \< \bw , g' \cdot \by \> ) \< g \cdot \bx, g' \cdot \by\>] \pi_d(\de g) \pi_d(\de g')\\
=&~ \int_{\cG_d \times \cG_d} h_d^{(2)}(\< g \cdot \bx, g' \cdot \by \>/d ) \pi_d(\de g) \pi_d(\de g')\\
=&~ \int_{\cG_d} h_d^{(2)}(\< \bx, g \cdot \by \> /d) \pi_d(\de g). 
\end{aligned}
\]
Taking $h_d = h_d^{(1)} + h_d^{(2)}$ concludes the proof. 
\end{proof}

\subsection{Proof of Proposition \ref{prop:output_symmetrization}}\label{sec:proof_output_symmetrization}

  
Let $\hf_d$ be an estimator satisfying 
\begin{equation}\label{eqn:output_sym_1}
\eps^2 : = \| \oproj_{\le \ell} f_d - \hf_d\|_{L^2}^2 = \| \oproj_{\le \ell} f_d - \oproj_{\le \ell} \hf_d\|_{L^2}^2 + \| \oproj_{> \ell} \hf_d\|_{L^2}^2.
\end{equation}
By Jensen's inequality and by the equation above, we have
\begin{equation}\label{eqn:output_sym_2}
\| \cS \oproj_{>\ell}\hf_d\|_{L^2}^2 \le \| \oproj_{> \ell} \hf_d\|_{L^2}^2 \leq \eps^2.
\end{equation}
As a consequence, we have
\begin{equation}
\begin{aligned}
&~ \|\oproj_{>\ell} f_d\|_{L^2}^2 - 2 \eps \| \oproj_{> \ell}  f_d\|_{L^2}  \stackrel{(1)}{\le} \|\oproj_{>\ell} f_d - \cS \oproj_{>\ell} \hf_d \|_{L^2}^2 \stackrel{(2)}{=} \|\oproj_{>\ell} f_d - \oproj_{>\ell} \cS  \hf_d \|_{L^2}^2 \\
\stackrel{(3)}{\le}&~ \|f_d - \cS  \hf_d  \|_{L^2}^2 \stackrel{(4)}{=} \| \cS (f_d -  \hf_d ) \|_{L^2}^2 \stackrel{(5)}{\le}  \|f_d - \hf_d \|_{L^2}^2\\ \stackrel{(6)}{=}&~ \| \oproj_{\le \ell} f_d - \oproj_{\le \ell} \hf_d\|_{L^2}^2 + \| \oproj_{> \ell} f_d - \oproj_{>\ell} \hf_d\|_{L^2}^2 \stackrel{(7)}{=} \|\oproj_{>\ell} f_d \|_{L^2}^2 + \eps^2 + 2 \eps \|\oproj_{>\ell} f_d \|_{L^2}.
 \end{aligned}
 \end{equation}
 Here, (1) is by Eq. (\ref{eqn:output_sym_2}); (2) is by the fact that $\cS$ is exchangable with $\oproj_{> \ell}$ (c.f. Section \ref{sec:invariant}); (3) is by the fact that $\oproj_{> \ell}$ is a projection operator; (4) is by the fact that $f_d$ is $\cG_d$-invariant; (5) is by Jensen's inequality; (6) is by orthogonal decomposition; (7) is by Eq. (\ref{eqn:output_sym_1}). This concludes the proof.

\subsection{Details of numerical simulations}

\subsubsection{Synthetic data}

We consider the standard (inner-product) kernel $H_d ( \bx , \by ) = h_{\NTK} (\< \bx , \by \>/d)$ to be the neural tangent kernel associated to a depth-$5$ neural network with fully connected layers and ReLu activation $\sigma ( x) = \max(x , 0)$. This can be obtained iteratively as follow (see \cite{jacot2018neural} and \cite{bietti2020deep}): define for $u \in [-1,1]$,
\[
h_0 ( u) = \frac{1}{\pi} ( \pi - \arccos (u) ), \qquad h_1 ( u) = u \cdot  h_0 ( u) + \frac{1}{\pi} \sqrt{1 - u^2} \, ,
\]
and $h_{\NTK} (u) = h_{\NTK}^{5} (u)$ with $h^{1}_{\NTK} (u) = h^{1} (u) = u$ and for $k = 2 , \ldots, 5$,
\[
\begin{aligned}
 h^{k} (u) =&~ h_1 (  h^{k-1} (u) ) \, ,\\
 h_{\NTK}^{k} (u) =&~ h_{\NTK}^{k-1} (u) h_0 (  h^{k-1} (u) ) + h^{k} (u) \, .
\end{aligned}
\]
We compute the cyclic invariant kernel by summing over all cyclic translations $g \in \Cyc_d$:
\[
H_{d,\inv} ( \bx , \by ) = \frac{1}{d} \sum_{g \in \Cyc_d} h_{\NTK} (\< \bx , g \cdot \by \>/d) \, .
\]

\subsubsection{Cyclic invariant MNIST data set}\label{sec:MNIST_details}

\begin{figure}
\centering
\includegraphics[width = 1\linewidth]{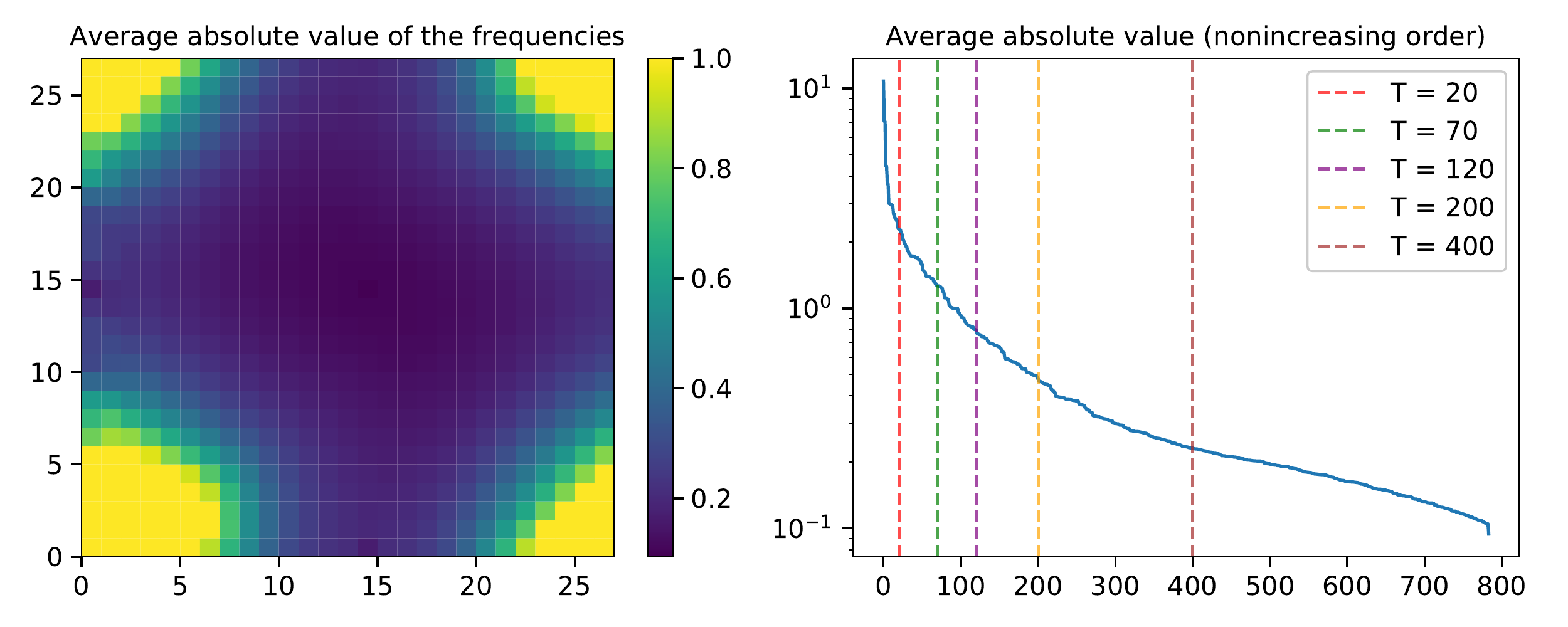}
\caption{Left frame: the absolute value of the frequency components of MNIST images averaged over the training set (threshold at $1$ in the figure). Coordinates on the bottom left-hand side correspond to lower frequency components while coordinates closer to the top right-hand side represent the high frequency directions. Right frame: average absolute value of the frequencies in nonincreasing order. The vertical lines correspond to the different $T$ chosen ($T=784$ corresponds to keeping all the frequencies).}  
\label{fig:DFT_MNIST} 
\end{figure}

We consider the MNIST data set of $28 \times 28$ grayscale images ($d = 784$) of handwritten digits, which contains $60000$ training images and $10000$ testing images. We pre-process the images in three steps:
\begin{itemize}
\item[(a)] We compute the discrete Fourier transform of the images in the training set and compute the average absolute value of the frequency components (see left frame of Fig.~\ref{fig:DFT_MNIST}). For each $T \in \{ 20, 70, 120 , 200, 400, 784 \}$, we select $\Omega_T \subset [28] \times [28]$ to be the set of the top $T$ frequencies (i.e., the $T$ frequencies with highest absolute value averaged on the training set).

    \item[(b)] For each $T$, we construct a train and test sets in which we project each image onto $\Omega_T$ (i.e., we set all the frequency components not in $\Omega_T$ to $0$). We displayed in Fig.~\ref{fig:projected_images} two digits and their projection on the top $T$ frequencies $\Omega_T$ for different $T$.
    
    \item[(c)] For each image in the training and test sets, we replace the image by a uniformly generated 2 dimensional (cyclic) translation of the image. We display some examples in Fig.~\ref{fig:shifted_MNIST}.
\end{itemize}

We further normalize the images so that $\| \bx \|_2 = 1$ and center the labels $y_i \in \cY $ where $\cY = \{ -4.5 , -3.5 , \ldots , 3.5 , 4.5 \}$. In order to compute the classification error, we round the prediction value to the nearest label in $\cY$.

\begin{figure}
\centering
\includegraphics[width = 1\linewidth]{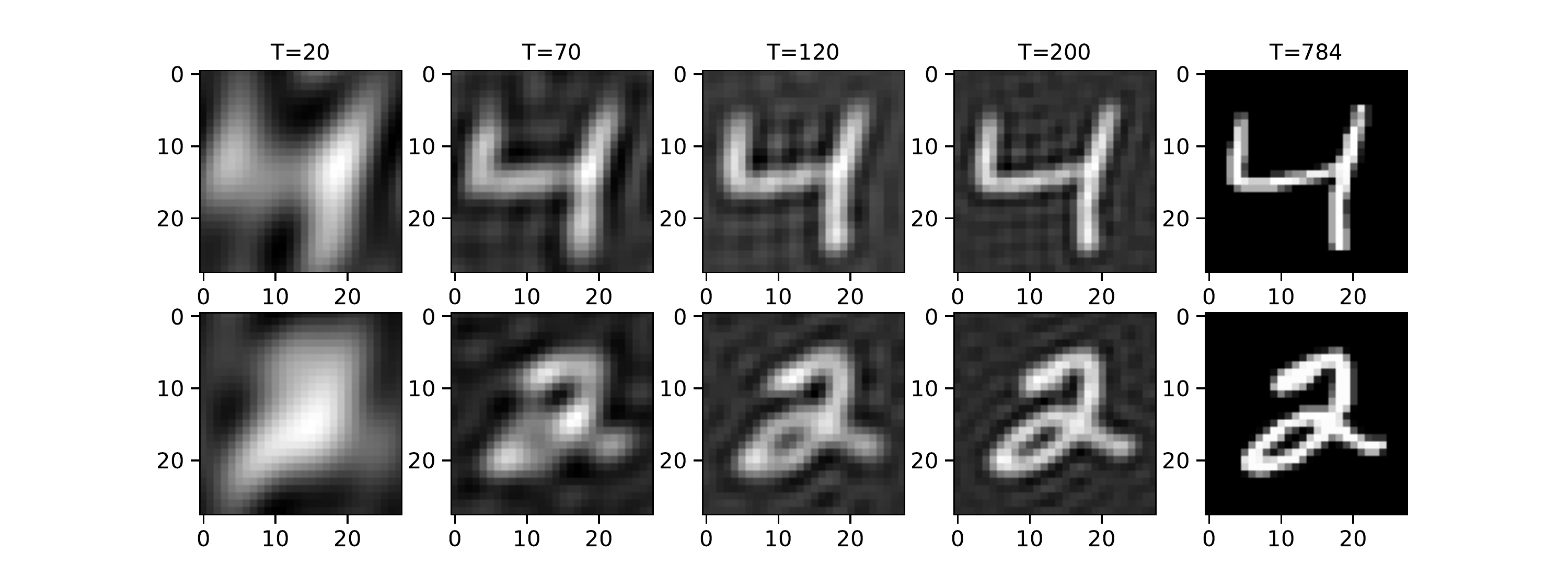}
\caption{Examples of two images projected on the top $T$ frequencies.}  
\label{fig:projected_images} 
\end{figure}


We use the inner-product kernel $H_d ( \bx , \by ) = h_{\NTK} (\< \bx , \by \>/d)$ where $h_{\NTK}$ is the neural tangent kernel associated to a 2-layers neural network with fully connected layers and ReLu activation $\sigma ( x) = \max(x , 0)$, which given by
\[
h_{\NTK} (u) =  u \cdot \Big( \pi - \frac{\arccos (u)}{\pi}  \Big) + \frac{1}{\pi} \sqrt{1 - u^2} \, .
\]
The cyclic invariant kernel is computed by summing over all two-dimensional cyclic translations $g_{ij} \in \TwoCyc_{28,28}$:
\[
H_{d,\inv} ( \bx , \by ) = \frac{1}{28^2} \sum_{i,j = 0}^{27} h_{\NTK} (\< \bx , g_{ij} \cdot \by \>/d) \, .
\]

For each $T$, we estimate the effective dimension $d_{\seff}$ by fitting two parallel lines through the classification error points of the standard and cyclic kernels at the same time (keeping only the points where the curves decrease). The estimated (log) effective dimension is then given by the difference of the offsets. We report these estimates for different $T$ in Fig.~\ref{fig:EffDimMNIST}.

\begin{figure}
\centering
\includegraphics[width = 1.\linewidth]{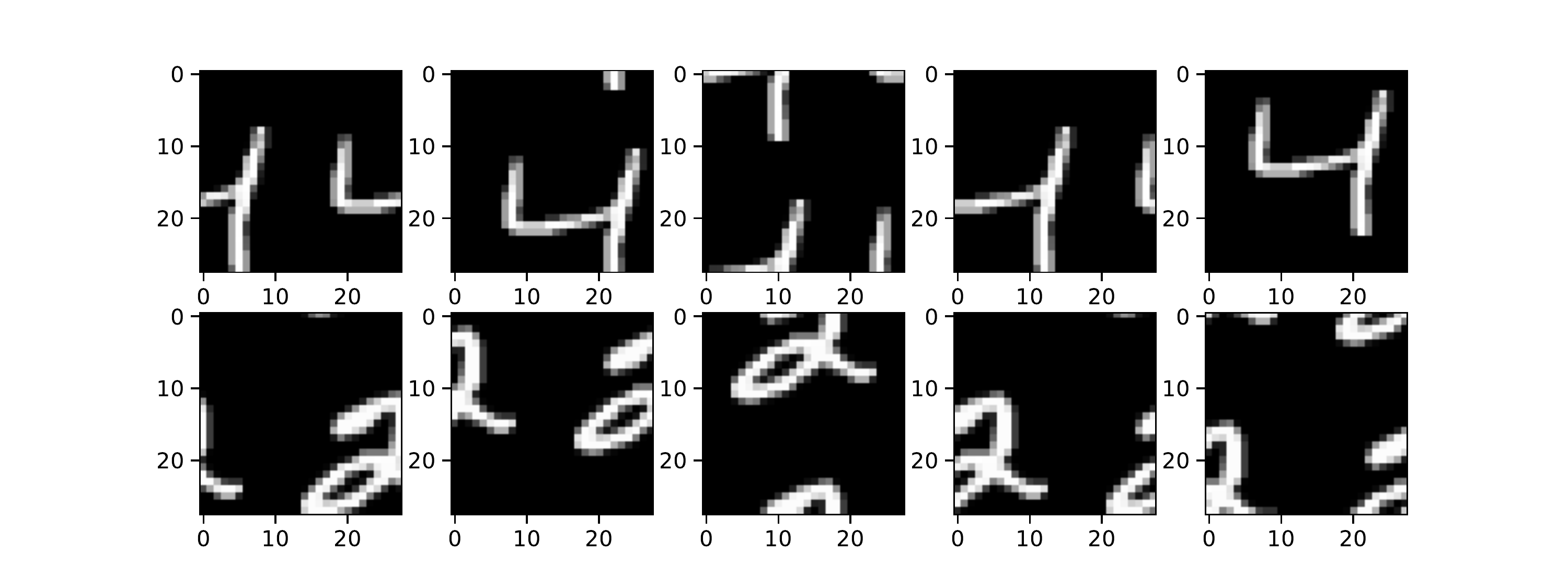}
\caption{Examples of random 2-dimensional cyclic translations of the images (for $T=784$).}  
\label{fig:shifted_MNIST} 
\end{figure}

\begin{figure}
\centering
\includegraphics[width = 0.4\linewidth]{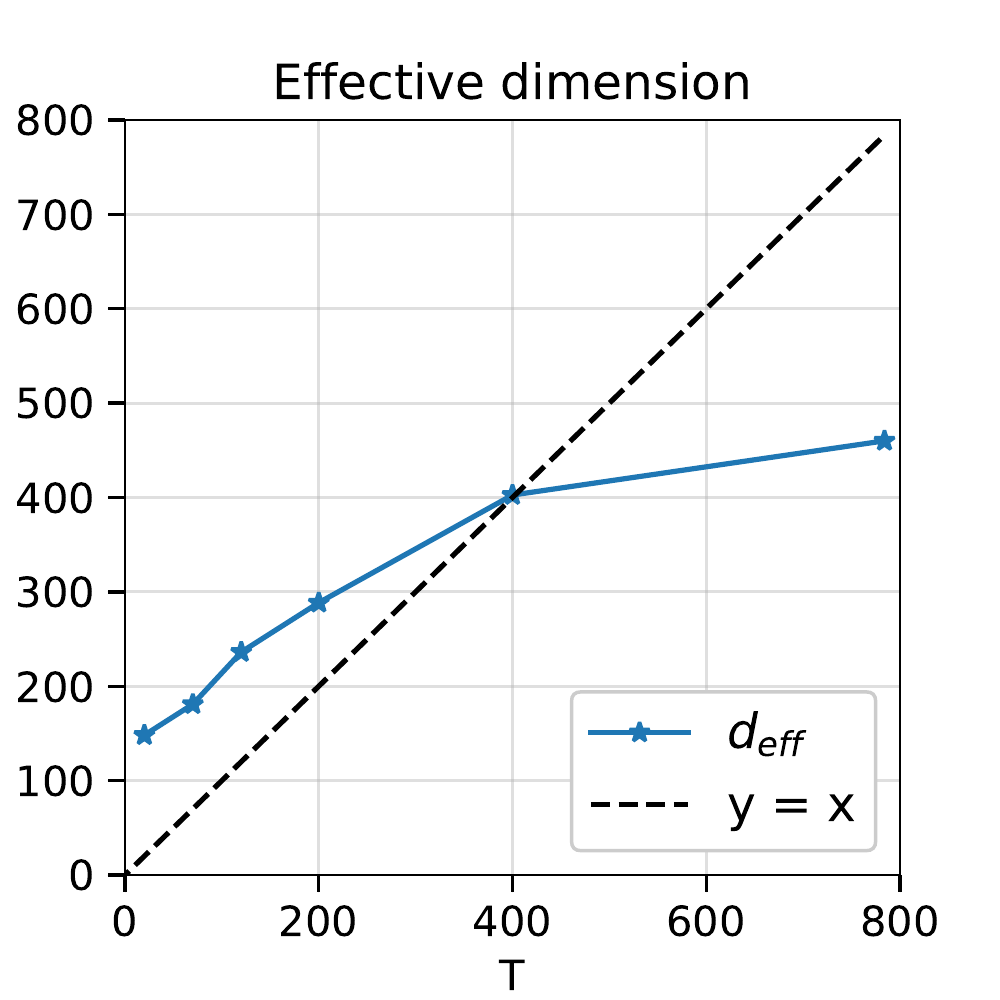}
\caption{Estimated sample size gap between standard and invariant kernel methods,  for the translationally invariant
  MNIST dataset, as a function of the frequency content $T$. }  
\label{fig:EffDimMNIST} 
\end{figure}

\section{Proof of the main theorems}\label{sec:proofs_main}

In this section, we present the proofs of Theorem \ref{thm:RFRRinvar} and \ref{thm:invar_KRR} stated in the main text. The rest of the appendices are organized as follow:
\begin{itemize}
    \item Appendix \ref{sec:invariant} presents key properties of the decomposition of invariant functions, while Appendix \ref{sec:technical_background} reviews some technical background on the functional spaces on the sphere and the hypercube.
    
    \item Appendix \ref{sec:counting_degeneracy} proves that the examples of symmetry group listed in Section \ref{sec:Examples} (one and two-dimensional cyclic groups and band-limited functions) have degeneracy 1.
    
    \item Appendix \ref{sec:concen_alpha} presents a key concentration result on the diagonal elements of polynomial invariant kernels. In particular, the results of Appendix \ref{sec:concen_alpha} are the only ones required in the proofs of Theorems \ref{thm:RFRRinvar} and \ref{thm:invar_KRR} in the case of polynomial activations for general symmetry group $\cG_d$ of degeneracy $\alpha \leq 1$. 
    
    \item Appendices \ref{sec:concentration_cyc} and \ref{sec:hyper_high} provides necessary results to extend the proofs to non-polynomial activations in the case of $(\cA_d , \cG_d ) = ( \S^{d-1} (\sqrt{d}), \Cyc_d)$.
\end{itemize}

 \subsection{Notations}

For a positive integer, we denote by $[n]$ the set $\{1 ,2 , \ldots , n \}$. For vectors $\bu,\bv \in \R^d$, we denote $\< \bu, \bv \> = u_1 v_1 + \ldots + u_d v_d$ their scalar product, and $\| \bu \|_2 = \< \bu , \bu\>^{1/2}$ the $\ell_2$ norm. Given a matrix $\bA \in \R^{n \times m}$, we denote $\| \bA \|_{\op} = \max_{\| \bu \|_2 = 1} \| \bA \bu \|_2$ its operator norm and by $\| \bA \|_{F} = \big( \sum_{i,j} A_{ij}^2 \big)^{1/2}$ its Frobenius norm. If $\bA \in \R^{n \times n}$ is a square matrix, the trace of $\bA$ is denoted by $\Tr (\bA) = \sum_{i \in [n]} A_{ii}$.

We use $O_d(\, \cdot \, )$  (resp. $o_d (\, \cdot \,)$) for the standard big-O (resp. little-o) relations, where the subscript $d$ emphasizes the asymptotic variable. Furthermore, we write $f = \Omega_d (g)$ if $g(d) = O_d (f(d) )$, and $f = \omega_d (g )$ if $g (d) = o_d (f (d))$. Finally, $f =\Theta_d (g)$ if we have both $f = O_d (g)$ and $f = \Omega_d (g)$.

We use $O_{d,\P} (\, \cdot \,)$ (resp. $o_{d,\P} (\, \cdot \,)$) the big-O (resp. little-o) in probability relations. Namely, for $h_1(d)$ and $h_2 (d)$ two sequences of random variables, $h_1 (d) = O_{d,\P} ( h_2(d) )$ if for any $\eps > 0$, there exists $C_\eps > 0 $ and $d_\eps \in \Z_{>0}$, such that
\[
\begin{aligned}
\P ( |h_1 (d) / h_2 (d) | > C_{\eps}  ) \le \eps, \qquad \forall d \ge d_{\eps},
\end{aligned}
\]
and respectively: $h_1 (d) = o_{d,\P} ( h_2(d) )$, if $h_1 (d) / h_2 (d)$ converges to $0$ in probability.  Similarly, we will denote $h_1 (d) = \Omega_{d,\P} (h_2 (d))$ if $h_2 (d) = O_{d,\P} (h_1 (d))$, and $h_1 (d) = \omega_{d,\P} (h_2 (d))$ if $h_2 (d) = o_{d,\P} (h_1 (d))$. Finally, $h_1(d) =\Theta_{d,\P} (h_2(d))$ if we have both $h_1(d) =O_{d,\P} (h_2(d))$ and $h_1(d) =\Omega_{d,\P} (h_2(d))$.

\subsection{Proof of Theorem \ref{thm:RFRRinvar}}\label{sec:proof_RFRR}

Let $\cG_d$ be a group of degeneracy $\alpha \leq 1$. Consider $\bx , \btheta \sim  \Unif (\cA_d)$, $d^{\lvn - \alpha + \delta_0} \leq n \leq d^{\lvn   - \alpha  + 1 - \delta_0}$, $d^{\lvN - \alpha +\delta_0 } \leq N \leq d^{\lvN - \alpha + 1 - \delta_0}$ and an activation function $\sigma$ that satisfies Assumption \ref{ass:activation_RC} at level $( \lvn , \lvN)$. Denote
\[
\osigma ( \bx ; \btheta ) = \int_{\cG_d} \sigma ( \< \btheta, g \cdot \bx \> / \sqrt{d} ) \pi_d ( \de g ).
\]
Theorem \ref{thm:RFRRinvar} is a consequence of Theorem 1 in \cite{mei2021generalization} where we take $\cX_d = \Omega_d = \cA_d$, $\nu_d = \tau_d = \Unif(\cA_d)$ and $\cD_d = \cV_d = L^2(\cA_d, \cG_d) \subset L^2(\cA_d)$. The proof amounts to checking that $\osigma$ indeed verifies the feature map concentration and spectral gap assumptions (see Section 2.2 in \cite{mei2021generalization}). We borrow some of the notations introduced in \cite{mei2021generalization} and refer the reader to their Section 2.1.

\begin{proof}[Proof of Theorem \ref{thm:RFRRinvar}]
For the sake of simplicity, we consider the overparametrized case $N(d) \geq n(d) d^\delta$ for some $\delta >0$, and therefore $\lvN \geq \lvn$. The underparametrized case $d^\delta N(d) \leq n(d)$ is treated analogously. 

 \noindent
{\bf Step 1. Diagonalization of the activation function $\osigma$ and choosing $\evn = \evn (d) $, $\evN = \evN(d)$. } 

We can decompose the inner product activation $\sigma$ in the basis of Gegenbauer polynomials (see Section \ref{sec:technical_background} for definitions):
\[
\sigma ( \< \bx , \btheta \> / \sqrt{d} ) = \sum_{k = 0}^\infty \xi_{d,k} B (\cA_d ; k ) Q_k^{(d)} ( \<\bx , \btheta \>) \, ,
\]
where (with $\be \in \cA_d$ arbitrary)
\[
\xi_{d,k} ( \sigma ) = \E_{\btheta \sim \Unif (\cA_d)} [ \sigma ( \< \be , \btheta \> /\sqrt{d} ) Q_k^{(d)} ( \< \be , \btheta \>)] \, .
\]
From Assumption \ref{ass:activation_RC}.$(a)$ that $| \sigma ( x ) | \leq c_0 \exp ( c_1 x^2 /2 )$ for some constants $c_0>0$ and $c_1 <1$ (which is trivially verified for a polynomial activation function), there exists a constant $C > 0 $ such that (see for example Lemma 5 in \cite{ghorbani2019linearized})
\begin{equation}\label{eq:upper_norm_sig}
\| \sigma ( \< \be , \, \cdot  \, \>/\sqrt{d} ) \|_{L^2(\cA_d)} = \sum_{k =1}^\infty \xi_{d,k}^2 B ( \cA_d ; k) \leq C.
\end{equation}
We have for fixed $k$, $B(\cA_d ; k ) = \Theta (d^k)$. Furthermore, for non-polynomial activation functions in the case of $(\cA_d , \cG_d) = ( \S^{d-1} (\sqrt{d} ) , \Cyc_d)$, we use $\sup_{k > \lvn} B(\S^{d-1} ; k )^{-1} = O_d ( d^{-\lvn - 1})$ (Lemma 1 in \cite{ghorbani2019linearized}). We deduce that 
\begin{align}
\sup_{k > \lvn} \xi_{d,k}^2 = &~ O_d ( d^{-\lvn - 1} )\, ,\\
\sup_{k > \lvN} \xi_{d,k}^2 = &~ O_d ( d^{-\lvN - 1} )\, .
\end{align}
From the correspondence between Gegenbauer and Hermite polynomials when $d \to \infty$ (see Eq.~\eqref{eqn:mu_lambda_relationship} in Section \ref{sec:Hermite}), Assumption \ref{ass:activation_RC}.$(b)$ implies that $\xi_{d,k}^2 = \Theta_d (d^{-k})$ for $k = 0, \ldots , \lvn$.

Let us diagonalize $\osigma$ in the basis of $\cG_d$-invariant polynomials $\{ \oY_{kl} \}_{k \ge 0 , \ell \in [D( \cA_d ; k)]}$ (see Section \ref{sec:invariant} for definitions). From Lemma \ref{lem:isometry_of_orthonomal_polynomials} stated in Section \ref{sec:invariant_representation}, we have
\begin{equation}\label{eq:decompo_osigma}
\begin{aligned}
\osigma ( \bx ; \btheta ) =&~ \sum_{k = 0}^\infty \xi_{d,k} B (\cA_d ; k ) \int_{\cG_d} Q_k^{(d)}(\< \bx, g \cdot \btheta \> ) \pi_d(\de g) \\
=&~ \sum_{k = 0}^\infty \xi_{d,k} \sum_{l = 1}^{D(\cA_d ; k)} \oY_{k l}^{(d)}(\bx) \oY_{kl}^{(d)} (\btheta) \, .
\end{aligned}
\end{equation}

Denote $(\lambda_{d,j})_{j \ge 1}$ the eigenvalues of $\osigma$ in non increasing order of their absolute value (namely, the $\xi_{d,k}$'s which have degeneracies $D(\cA_d ; k)$). Set $\evn$ and $\evN$ to be the number of eigenvalues $\lambda_{d,j}^2$ that are bigger than $d^{-\lvn - 1 + \delta}$ and $d^{-\lvN - 1 +\delta}$ respectively, for a constant $\delta>0$ that will be set sufficiently small (see Step 4). From the above discussion, $(\lambda_{d,j})_{j \leq \evn}$ corresponds exactly to all the eigenvalues associated to invariant polynomials of degree less of equal to $\lvn$, while $(\lambda_{d,j})_{j \leq \evN}$ does not contain any eigenvalues associated to invariant polynomials of degree bigger or equal to $\lvN+1$. Hence,
\begin{equation}\label{eq:bound_m_M}
\evn = \sum_{k = 0}^\lvn D( \cA_d ; k) = \Theta_d ( d^{\lvn - \alpha}), \qquad \evN \leq \sum_{k = 0}^\lvN D( \cA_d ; k) = O_d ( d^{\lvN - \alpha})\, ,
\end{equation}
where we used that $\cG_d$ has degeneracy $\alpha$ so that $D(\cA_d; k) = \Theta_d (d^{-\alpha}) \cdot B(\cA_d; k)$.

 \noindent
{\bf Step 2. Diagonal elements of the truncated kernel. } 

We introduce the kernel associated to activation $\osigma$:
\[
H_d ( \bx_1 , \bx_2 ) = \E_{\btheta} [ \osigma ( \bx_1 ; \btheta) \osigma ( \bx_2 ; \btheta )] = \sum_{k = 0}^\infty \xi_{d,k}^2 D (\cA_d ; k) \Upsilon_k^{(d)} (\bx_1 , \bx_2) \, ,
\]
where we denote
\[
\Upsilon_k^{(d)} (\bx_1 , \bx_2) = \frac{1}{D(\cA_d ; k)} \sum_{l = 1}^{D(\cA_d ; k)} \overline Y_{k l}^{(d)}(\bx_1) \overline Y_{kl}^{(d)} (\by_1 ) \, . 
\]
Similarly, we introduce a kernel in the feature space
\[
U_d ( \btheta_1 , \btheta_2 ) = \E_{\bx} [ \osigma ( \bx ; \btheta_1) \osigma ( \bx ; \btheta_2 )] = \sum_{k = 0}^\infty \xi_{d,k}^2 D (\cA_d ; k) \Upsilon_k^{(d)} (\btheta_1 , \btheta_2) \, .
\]
We denote $\Hop_d , \Uop_d :  L^2(\cA_d, \cG_d)  \to  L^2(\cA_d, \cG_d) $ the kernel operators with kernel representation $H_d$ and $U_d$, and denote $\Hop_{d,> \evn}$ and $\Uop_{d, > \evN}$ the kernel operators where the biggest $\evn$ and $\evN$ eigenvalues respectively are set to $0$. Recalling the discussion on the choice of $\evn$ and $\evN$, denote $E = \{ k : \xi_{d,k}^2 \leq d^{-\lvN -1 + \delta}\}$: $E$ contain all integers bigger or equal to $\lvN + 1$ and none smaller or equal to $\lvn$. 

The diagonal elements of the truncated kernels are then given by 
\begin{equation}\label{eq:truncated_kernels}
    \begin{aligned}
     H_{d, > \evn} ( \bx , \bx ) =&~  \sum_{k = \lvn + 1}^\infty \xi_{d,k}^2 D (\cA_d ; k) \Upsilon_k^{(d)} (\bx , \bx) \, , \\
     U_{d, > \evN} ( \btheta , \btheta ) =&~  \sum_{k \in E} \xi_{d,k}^2 D (\cA_d ; k) \Upsilon_k^{(d)} (\btheta , \btheta) \, ,
    \end{aligned}
\end{equation}
and 
\[
    \begin{aligned}
     \Tr ( \Hop_{d, > \evn}  ) = &~ \E_{\bx} [H_{d, > \evn} ( \bx , \bx )] = \sum_{k = \lvn + 1}^\infty \xi_{d,k}^2 D (\cA_d ; k) \, , \\
     \Tr ( \Uop_{d, > \evN}  ) = &~ \E_{\btheta} [U_{d, > \evN} ( \btheta , \btheta )] = \sum_{k \in E} \xi_{d,k}^2 D (\cA_d ; k)  \, .
    \end{aligned}
\]
From Assumption \ref{ass:activation_RC}.$(c)$, $\sigma$ is not a polynomial of degree less or equal to $\lvN$. Hence, there exists $\ell > \lvN$ such that $\mu_{\ell} (\sigma) \neq 0$ and therefore $\xi_{d,\ell}^2 D(\cA_d ; \ell) = \Theta (d^{-\alpha})$. Furthermore, from Eq.~\eqref{eq:upper_norm_sig} and the assumption that $\cG_d$ is of degeneracy $\alpha$ (for polynomial activation functions, see Proposition \ref{prop:h_concentration_general_sigma} in Section \ref{sec:hyper_high} for general $\sigma$ in the case of $(\cA_d , \cG_d) = (\S^{d-1} (\sqrt{d} ), \Cyc_d)$), we have
\begin{equation}\label{eq:scale_trace}
\Tr ( \Hop_{d, > \evn}  ) = \Theta (d^{-\alpha}) , \qquad \Tr ( \Uop_{d, > \evN}  ) =  \Theta (d^{-\alpha} ).
\end{equation}

 \noindent
{\bf Step 3. Checking the feature map concentration property at level $\{ N(d), \evN (d) , n (d) , \evn (d) \}_{d \ge 1} $. } 

Let us first consider the case of a polynomial activation function $\sigma$. Denote $D$ its degree and $u = u(d)$ the total (finite) number of nonzero eigenvalues of $\osigma$ (which are associated to invariant polynomials of degree less or equal to $D$). Let us verify the feature map concentration property (Assumption 1 in \cite{mei2021generalization}) with sequence $u(d) \ge \max( \evn, \evN )$. Note that $u \ge \max( \evn, \evN )$, part $(b)$ and $(c)$ of the property are trivially verified in that case.

\begin{itemize}
    \item[(a)] \textit{(Hypercontractivity of finite eigenspaces on $\cD_d$.)} The subspace of polynomials of degree less or equal to $D$ on the hypercube and the sphere verifies the hypercontractivity property (see Lemmas \ref{lem:hypercontractivity_hypercube} and \ref{lem:hypercontractivity_sphere} in Section \ref{app:hypercontractivity}).  
    
    \item[(d)] \textit{(Concentration of diagonal elements.)} From Eq.~\eqref{eq:truncated_kernels} and Proposition \ref{prop:general_concentration_diagonal} stated in Section \ref{sec:concen_alpha}, we have 
    \[
    \begin{aligned}
    &~ \sup_{i \in [n]} \Big\vert H_{d, > \evn} ( \bx_i , \bx_i )  - \E_{\bx} [ H_{d, > \evn} ( \bx , \bx )] \Big\vert \\
    \leq &~ \sum_{k = \lvn +1}^D \xi_{d,k}^2 D ( \cA_d ; k) \sup_{i \in [n]} \Big\vert \Upsilon_k^{(d)} (\bx_i , \bx_i) - \E_\bx [ \Upsilon_k^{(d)} (\bx , \bx) ]\Big\vert = o_{d,\P}(1) \cdot \E_{\bx} [ H_{d, > \evn} ( \bx , \bx )].
    \end{aligned}
    \]
    A similar computation shows the concentration of the diagonal elements of $U_{d, > \evN}$.
\end{itemize}

Let us now consider a non polynomial activation function $\sigma$ in the case of $\cA_d = \S^{d-1} (\sqrt{d} )$ and $\cG_d = \Cyc_d$ (of degeneracy 1). Let us choose $\ell > 2 \lvN + 10$ such that $\mu_\ell (\sigma ) \neq 0$ (it must exists otherwise $\sigma$ would be a polynomial) and therefore $\xi_{d, \ell}^2 = \Theta_d ( d^{-\ell})$. Consider $u =u(d)$ to be the number of eigenvalues such that $\lambda_{d,j}^2$ is strictly bigger than $\xi_{d, \ell}^2$. Then, $(\lambda_{d,j})_{j \leq u}$ do not contain any eigenvalues $\xi_{d,k}$ for $k \geq \ell$ and contain all $\xi_{d,k}$ for $k \leq \lvn$. In particular, $u \geq \max(\evn, \evN)$. Denote $E = \{ k : \xi_{d,k}^2 \leq \xi_{d, \ell}^2 \}$: $E$ contain all integers bigger or equal to $\ell$. 

Let us verify the feature map concentration property with the sequence $u(d)$ (part $(a)$ is the same with $D$ replaced by $\ell - 1$).

\begin{itemize}
    \item[(b)] \textit{(Properly decaying eigenvalues.)} We have
    \[
    \begin{aligned}
     \Tr (\Hop_{d, > u} ) \geq &~ \xi_{d, \ell}^2 D (\S^{d-1} ; \ell ) = \Omega_d ( d^{-\alpha} ) \, , \\
     \Tr (\Hop_{d, > u}^2 ) = & \sum_{k \in E} \xi_{d, k}^4 D (\S^{d-1} ; k ) \leq \xi_{d, \ell}^2 \Tr (\Hop_{d, > u} ) \, .
    \end{aligned}
    \]
    Hence,
    \[
    \frac{\Tr (\Hop_{d, > u} )^2}{\Tr (\Hop_{d, > u}^2 )} \geq \xi_{d, \ell}^{-2} \cdot \Tr (\Hop_{d, > u} ) = \Omega_d (1) \cdot d^{2\lvN + 9}  \geq \max(n, N)^{2 + \delta}. 
    \]
    
    \item[(c)] \textit{(Hypercontractivity of the high degree part.)} Denote $\osigma_{>u} = \oproj_E \osigma$ the activation $\osigma$ obtained by setting the first $u$ eigenvalues to $0$ (i.e., setting coefficients $k \not\in E$ to zero in Eq.~\eqref{eq:decompo_osigma}). From Eq.~\eqref{eq:scale_trace}, we need to show that for $p$ as defined in Assumption \ref{ass:activation_RC}.$(a)$, we have 
    \[
    \E_{\bx,\btheta} [ \osigma_{>u} ( \bx ; \btheta )^{2p} ]^{1/(2p)} = O_d (d^{-1/2+ \delta}).
    \]
    Denote $E_{\leq 4p} = E \cap \{ 0 , \ldots , 4p \}$ (recall that $E$ contains all $k \geq \ell$) and decompose $\osigma_{>u} = \oproj_{E_{\leq 4p}} \osigma + \oproj_{>4p} \osigma$. Then by triangle inequality we have,
    \[
    \begin{aligned}
    \E_{\bx,\btheta} [ \osigma_{>u} ( \bx ; \btheta )^{2p} ]^{1/(2p)}  \leq \E_{\bx,\btheta} [ \oproj_{E_{\leq 4p }} \osigma ( \bx ; \btheta )^{2p} ]^{1/(2p)} + \E_{\bx,\btheta} [\oproj_{> 4p} \osigma ( \bx ; \btheta )^{2p} ]^{1/(2p)}.
    \end{aligned}
    \]
    Using hypercontractivity of polynomials of degree less or equal to $4p$, the first term is bounded by $O_d( d^{-1/2})$, while the second term is bounded in Proposition \ref{prop:assumption2a_cyclic} in Section \ref{sec:hyper_high}.
    
    \item[(d)] \textit{(Concentration of diagonal elements.)} This is proved in Proposition \ref{prop:h_concentration_general_sigma} in Section \ref{sec:concentration_cyc}.
\end{itemize}

 \noindent
{\bf Step 4. Checking the spectral gap property at level $\{ N(d), \evN (d) , n (d) , \evn (d) \}_{d \ge 1} $. } 

Let us now check the spectral gap property (Assumption 2 in \cite{mei2021generalization}).

\begin{itemize}
    \item[(a)] \textit{(Number of samples.)} First by Eq.~\eqref{eq:bound_m_M} and the assumption $d^{\lvn - \alpha  + \delta_0} \leq n \leq d^{\lvn +1 - \alpha - \delta_0}$, we have $\evn \leq n^{1- \delta}$ for $\delta > 0$ chosen sufficiently small. By the choice of $\evn$ and recalling Eq.~\eqref{eq:scale_trace}, we have
    \[
    \begin{aligned}
     \lambda^{-2}_{\evn+1} \Tr ( \Hop_{d, > \evn} ) = &~ \sup_{k \geq \lvn +1} \{ \xi_{d,k}^{-2} \} \cdot \Tr ( \Hop_{d, > \evn} )  = \Omega_d(d^{\lvn + 1 - \alpha} ) \geq n^{1 + \delta} \, ,\\
     \lambda^{-2}_{\evn} \Tr ( \Hop_{d, > \evn} )  = &~ \xi_{d,\lvn}^{-2}  \Tr ( \Hop_{d, > \evn} ) = O_d(1) \cdot d^{\lvn - \alpha} \leq n^{1 - \delta}\, ,  
    \end{aligned}
    \]
    with $\delta >0$ chosen sufficiently small.
    
    \item[(b)] \textit{(Number of features.)} By construction $\evN \geq \evn$. Furthermore, recalling Eq.~\eqref{eq:bound_m_M} and the assumption $d^{\lvN - \alpha + \delta_0} \leq N \leq d^{\lvN + 1 - \alpha - \delta_0}$, we have $\evN \leq N^{1 - \delta}$ for $\delta > 0$ chosen sufficiently small. By choice of $\evN$, $\lambda_{d, \evN+1}^2  \leq d^{- \lvN - 1 + \delta }$. Hence,
    \[
    \lambda^{-2}_{\evN+1} \Tr ( \Uop_{d, > \evN} ) = \Omega_d(1) \cdot d^{ \lvN + 1 - \alpha - \delta } \geq N^{1 + \delta},
    \]
    for $\delta>0$ chosen sufficiently small.
\end{itemize}

Finally notice that we used a different parametrization of $\lambda$ in Eq.~\eqref{eq:RFRR_problem} and the condition in \cite{mei2021generalization} becomes $\lambda / d^\alpha = O_d(1) \cdot \Tr ( \Hop_{d, >\evn}) $, i.e., $\lambda = O_d(1)$. This concludes the proof.
\end{proof}

\subsection{Proof of Theorem \ref{thm:invar_KRR}}

We consider the same setting as in the previous section and consider
\[
H_d ( \bx_1 , \bx_2 ) = \E_{\btheta} [ \osigma ( \bx_1 ; \btheta) \osigma ( \bx_2 ; \btheta )] \, .
\]
Theorem \ref{thm:invar_KRR} is a consequence of Theorem 4 in \cite{mei2021generalization} and the proof amounts to checking that $H_d$ verifies the kernel concentration properties and eigenvalue condition (see Section 3.2 in \cite{mei2021generalization}). Note that some of the conditions were already covered in the proof of Theorem \ref{thm:RFRRinvar} and we will only mention the ones that still need to be verified. Furthermore, by the spectral gap property proven in Section \ref{sec:proof_RFRR}, the bound in Theorem 4 in \cite{mei2021generalization} (which is in term of a shrinkage operator) can indeed be rewritten as 
\[
 R_{\KR, \inv}(f_{d}, \bX, \lambda) = \| \oproj_{> \lvn} f_d \|_{L^2}^2 + o_{d,\P}(1) \cdot (\|  f_d \|_{L^{2 + \eta}}^2 + \noise^2).
\]

\begin{proof}[Proof of Theorem \ref{thm:invar_KRR}] We choose $\evn$ as in the proof of Theorem \ref{thm:RFRRinvar}.

 \noindent
{\bf Step 1. Checking the kernel concentration property at level $\{ n (d) , \evn (d) \}_{d \ge 1} $. } 

First notice that
\[
\E_{\bx} [H_{d, > \evn} ( \bx_i , \bx ) ] =  \sum_{k = \lvn + 1}^\infty \xi_{d,k}^4 D (\cA_d ; k) \Upsilon_k^{(d)} (\bx_i , \bx_i),
\]
and the concentration of the diagonal elements in the case of a polynomial activation function follows from the same argument as in Section \ref{sec:proof_RFRR}.

Hence, we only need to check this property in the case non polynomial activation function $\sigma$ ($\cA_d = \S^{d-1} (\sqrt{d} )$ and $\cG_d = \Cyc_d$ of degeneracy $1$). Let us choose $u$ as in the proof of the feature map concentration property in Theorem \ref{thm:RFRRinvar}.
\begin{itemize}
    \item \textit{(Properly decaying eigenvalues.)} We have
    \[
    \begin{aligned}
     \Tr ( \Hop_{d, >u}^2 ) \geq &~ \xi_{d,\ell}^4 D( \S^{d-1} ; \ell ) = \Omega_d (1) \cdot d^{-\ell -1} \, , \\
     \Tr ( \Hop_{d, >u}^4 ) \leq &~  \sup_{j \leq u} \{ \lambda_{d,j}^6 \} \Tr (  \Hop_{d, >u} ) = O_d (1) \cdot d^{-3 \ell}  \, .
    \end{aligned}
    \]
    Hence,
    \[
    \frac{\Tr ( \Hop_{d, >u}^2 )^2}{ \Tr ( \Hop_{d, >u}^4 ) } = \Omega_d (1) \cdot d^{ \ell -2 } = \Omega_d (d^{2\lvn} ) \geq n^{2 + \delta}.
    \]
    
    \item \textit{(Concentration of the diagonal elements of the kernel.)} This is proven in Proposition \ref{prop:hh_concentration_general_sigma} in Section \ref{sec:concentration_cyc}.
\end{itemize}

\noindent
{\bf Step 2. Checking the eigenvalue condition at level $\{ n (d) , \evn (d) \}_{d \ge 1} $. } 

By the choice of $\evn$, we have
\[
\lambda_{d,\evn +1}^{-4} \Tr ( \Hop_{d, >\evn}^2 ) = \frac{\sum_{k \geq \lvn +1} \xi_{d,k}^4 D(\cA_d ; k) }{\sup_{k \ge \lvn +1} \xi_{d,k}^4} \geq D(\cA_d ; \lvn +1) = \Omega_d ( d^{\lvn + 1 - \alpha} ) \geq n^{1 + \delta}.
\]

Again notice that we used a different parametrization of $\lambda$ in Eq.~\eqref{eq:KRR_problem} and the condition in \cite{mei2021generalization} becomes $\lambda / d^\alpha = O_d(1) \cdot \Tr ( \Hop_{d, >\evn}) $, i.e., $\lambda = O_d(1)$. This concludes the proof.
\end{proof}

\section{Decomposition of invariant functions}\label{sec:invariant}

In this section, we take $\cA_d \in \{ \S^{d-1}(\sqrt {d}), \Cube^d \}$, and $\cG_d$ to be any group that is isomorphic to a subgroup of $\cO(d)$ and that preserves $\cA_d$. This section is mostly built on the technical background presented in Appendix \ref{sec:technical_background}. 

\subsection{The invariant function class and the symmetrization operator}\label{sec:invariant_class}

Let $L^2(\cA_d)$ be the class of $L^2$ functions on $\cA_d$ equipped with uniform probability measure $\Unif(\cA_d)$. We define the invariant function class to be 
\[
L^2(\cA_d, \cG_d) = \Big\{ f \in L^2(\cA_d): f(\bx) = f(g \cdot \bx), ~~ \forall \bx \in \cA_d, ~~ \forall g \in \cG_d \Big\}.
\]
We define the symmetrization operator $\cS: L^2(\cA_d) \to L^2(\cA_d, \cG_d)$ to be 
\[
(\cS f) (\bx) = \int_{\cG_d} f(g \cdot \bx) \pi_d(\de g). 
\]

\subsection{Orthogonal polynomials on invariant function class}\label{sec:invariant_polynomials}


For either $\cA_d \in \{ \S^{d-1}(\sqrt{d}), \Cube^d \}$, we define $V_{d, \le k} \subseteq L^2(\cA_d)$ to be the subspace spanned by all the degree $\ell$ polynomials, $V_{d, >k} \equiv V_{d, \le k}^\perp \subseteq L^2(\cA_d)$ to be the orthogonal complement of $V_{d, \le k}$, and $V_{d, k} = V_{d, \le k} \cap V_{d, \le k-1}^\perp$. In words, $V_{d, k}$ contains all degree $k$ polynomials that orthogonal to all polynomials of degree at most $k-1$. We further define $V_{d, <k} = V_{d, \le k-1}$ and $V_{d, \ge k} = V_{d, > k-1}$. 

Let $\bproj_{\le \ell}$ to be the projection operator on $L^2(\cA_d, \Unif)$ that project a function onto $V_{d, \le \ell}$, the space spanned by all the degree $\ell$ polynomials. Then it is easy to see that $\bproj_{\le \ell}$ and $\cS$ operator commute. This means, for any $f \in L^2(\cA_d)$, we have 
\[
\bproj_{\le \ell}[\cS( f )] = \cS[\bproj_{\le \ell}(f)]. 
\]
Similarly, we can define $\bproj_{\ell}$, $\bproj_{< \ell}$, $\bproj_{> \ell}$, $\bproj_{\ge \ell}$, which commute with $\cS$. We denote $V_{d, \ell}(\cG_d) \equiv \cP_{\ell}(\cA_d, \cG_d)$ to be the space of polynomials in the images of $\bproj_\ell \cS$ (which is consistent with the definition of $V_{d, \ell}(\cG_d)$ in Definition \ref{def:degeneracy}). Then we have
\[
\cP_{\ell}(\cA_d, \cG_d) = \bproj_\ell( L^2(\cA_d, \cG_d)) = \cS[\bproj_\ell(L^2(\cA_d))]. 
\]

We denote $D(\cA_d; k) = D(\cA_d; \cG_d; k) \equiv \dim(\cP_k(\cA_d, \cG_d))$ to be the dimension of $\cP_k(\cA_d, \cG_d)$. We denote $\{ \overline Y_{kl}^{(d)} \}_{l \in [D(\cA_d; k)]}$ to be a set of orthonormal polynomial basis in $\cP_k(\cA_d, \cG_d)$. That means 
\[
\E_{\bx \sim \Unif(\cA_d)} [\overline Y_{k_1 l_1}^{(d)}(\bx) \overline Y_{k_2 l_2}^{(d)}(\bx)] = \ones\{k_1 = k_2, l_1 = l_2 \},
\]
and 
\[
\overline Y_{k l}^{(d)}(\bx) = \overline Y_{k l}^{(d)}(g \cdot \bx), ~~ \forall \bx \in \cA_d, ~~ \forall g \in \cG_d.
\]

\subsection{A representation lemma}\label{sec:invariant_representation}

We have the following representation lemma. This lemma is important in the proofs of counting the degeneracy of groups (See Section \ref{sec:counting_degeneracy}). 

\begin{lemma}[Convolution representation of projection operator]\label{lem:isometry_of_orthonomal_polynomials}
Let $Q_k^{(d)}$ be the $k$-th Gegenbauer polynomial, or the $k$-th hypercubic Gegenbauer polynomial. For any fixed integer $k$, we have 
\begin{equation}\label{eqn:YY_Q}
\frac{1}{D(\cA_d; k)}\sum_{l = 1}^{D(\cA_d; k)} \overline Y_{k l}^{(d)}(\bx) \overline Y_{kl}^{(d)} (\by) =  \frac{B(\cA_d; k)}{D(\cA_d; k)}\int_{\cG_d} Q_k^{(d)}(\< \bx, g \cdot \by\> ) \pi_d(\de g). 
\end{equation}
\end{lemma}

\begin{proof}[Proof of Lemma \ref{lem:isometry_of_orthonomal_polynomials}] Define 
\[
\Gamma_{1k}(\bx, \by) = \sum_{l = 1}^{D(\cA_d; k)} \overline Y_{k l}^{(d)}(\bx) \overline Y_{kl}^{(d)} (\by), 
\]
and
\[
\Gamma_{2k}(\bx, \by) = B(\cA_d; k) \int_{\cG_d} Q_k^{(d)}(\< g \cdot \bx, \by\> ) \pi_d(\de g). 
\]
Then $\Gamma_{1k}$ and $\Gamma_{2k}$ define two operators $\T_{1k}, \T_{2k}: L^2(\cA_d) \to L^2(\cA_d)$, i.e., for $j = 1,2$, 
\[
\T_{j k} f(\bx) = \E_{\by \sim \Unif(\cA_d)}[\Gamma_{j k}(\bx, \by) f(\by)]. 
\]

Recall that $Q_k^{(d)}$ is a representation of the projector onto the subspace of degree-$k$ spherical harmonics (see Eq.~\eqref{eq:ProjectorGegenbauer} in Section \ref{sec:Gegenbauer}). We deduce that
\[
\T_{2k} f ( \bx ) = \cS \E_{\by} [ B(\cA_d; k) Q_k^{(d)}(\< \bx, \by\> ) f (\by )] = \cS   \bproj_k f (\bx),
\]
and therefore $\T_{2k} = \cS \bproj_k$.

Furthermore, we have $\T_{1k} =  \bproj_k \cS$. Indeed, the images of both $\T_{1k}$ and $\bproj_k \cS$ are $\cP_k(\cA_d, \cG_d)$, the space $\cP_k(\cA_d, \cG_d)^\perp$ is the null space of both $\T_{1k}$ and $\bproj_k \cS$, and $\T_{1k} \overline Y_{kp}^{(d)}(\bx)  = \bproj_k \cS \overline Y_{kp}^{(d)}(\bx) = \overline Y_{kp}^{(d)}(\bx)$. 

By the commutativity of $\bproj_k $ and $\cS$ operator, we have $\T_{1k} = \bproj_k \cS = \cS \bproj_k  = \T_{2k}$, and hence $\Gamma_{1k} = \Gamma_{2k}$. 
\end{proof}

\subsection{Gegenbauer decomposition of invariant features and kernels}


By Section \ref{sec:technical_background}, for either $\cA_d \in \{\S^{d-1}(\sqrt{d}), \Cube^d \}$, for any activation function $\sigma \in L^2([-\sqrt d, \sqrt d], \tau^1_{d})$ (where $\tau^1_{d}$ is the distribution of $\< \bx_1, \bx_2 \> $ when $\bx_1, \bx_2 \sim_{iid} \Unif(\cA_d)$), we can define its coefficients $\xi_{d, k}(\sigma)$ defined by 
\begin{align}
\xi_{d, k}(\sigma) = \int_{[-\sqrt d , \sqrt d]} \sigma(x) Q_k^{(d)}(\sqrt d x) \tau^1_{d}(\de x),
\end{align}
so that we have the following equation holds in $L^2([-\sqrt d, \sqrt d],\tau^1_{d})$ sense
\[
\sigma(x) = \sum_{k = 0}^\infty \xi_{d, k}(\sigma) B(\cA_d; k) Q_k^{(d)}(\sqrt d x). 
\]
For any group $\cG_d$ that is a subgroup of $\cO(d)$, we define 
\[
\osigma(\bx; \btheta) \equiv \int_{\cG_d} \sigma(\< \bx, g \cdot \btheta \> / \sqrt{d}) \pi_d(\de g). 
\]
Then, by the representation lemma (Lemma \ref{lem:isometry_of_orthonomal_polynomials}), we have 
\[
\begin{aligned}
\osigma(\bx; \btheta) \equiv &~ \sum_{k = 0}^\infty \xi_{d, k}(\sigma) B(\cA_d; k) \int_{\cG_d} Q_k^{(d)}(\<\bx, g \cdot \btheta\>)\pi_d(\de g) \\
=&~ \sum_{k = 0}^\infty \xi_{d, k}(\sigma) \sum_{l = 1}^{D(\cA_d; k)} \overline Y_{k l}^{(d)}(\bx) \overline Y_{kl}^{(d)} (\btheta). 
\end{aligned}
\]
As a consequence, suppose we define 
\[
H_d(\bx, \by) = \E_{\btheta \sim \Unif(\cA_d)}[\osigma(\bx; \btheta) \osigma(\by; \btheta)].
\]
Then we have 
\[
H_d(\bx, \by) =\sum_{k = 0}^\infty \xi_{d, k}(\sigma)^2 \sum_{l = 1}^{D(\cA_d; k)} \overline Y_{k l}^{(d)}(\bx) \overline Y_{kl}^{(d)} (\by). 
\]



\section{Counting the degeneracy}\label{sec:counting_degeneracy}

\subsection{Counting the degeneracy of $\Cyc_d$ and $\TwoCyc_{d_1, d_2}$ (Example \ref{ex:OneDImages} and \ref{ex:TwoDImages})}

\begin{proposition}\label{prop:degeneracy_cyclic}
Let $\cG_d \in \{ \Cyc_d, \TwoCyc_{d_1, d_2}\}$ with $d = d_1 \times d_2$. Let $\cA_d \in \{ \S^{d-1}(\sqrt{d}), \Cube^d \}$. Then for any fixed $k \ge 1$, we have 
\[
\dim(\cP_k(\cA_d, \cG_d)) \equiv D(\cA_d; k) = \Theta_d(d^{k-1}). 
\]
\end{proposition}

\subsubsection{Proof of Proposition \ref{prop:degeneracy_cyclic}}

Here we state a key lemma that is used to prove Proposition \ref{prop:degeneracy_cyclic}. 

\begin{lemma}\label{lem:cyclic_integral}
Let $\cG_d \in \{ \Cyc_d, \TwoCyc_{d_1, d_2}\}$ with $d = d_1 \times d_2$. Denote 
\[
F_k(\bz) = \int_{\cG_d} (\< \bz, g \cdot \bz\>/d)^k \pi_d(\de g). 
\]
Then for any fixed $k \ge 1$, we have
\begin{align}
\E_{\bz \sim \cN(\bzero, \id_d)}[F_k(\bz)] =&~ \Theta_d(d^{-1}), \label{eqn:F_integral_Gaussian}\\
\E_{\btheta \sim \Unif(\Cube^d)}[F_k(\btheta)] =&~ \Theta_d(d^{-1}), \label{eqn:F_integral_Cube} \\
\E_{\btheta \sim \S^{d-1}(\sqrt{d})}[F_k(\btheta)] =&~ \Theta_d(d^{-1}).\label{eqn:F_integral_Spherical}
\end{align}
\end{lemma}

\begin{proof}[Proof of Lemma \ref{lem:cyclic_integral}]
We prove Eq. (\ref{eqn:F_integral_Gaussian}) and (\ref{eqn:F_integral_Spherical}). The proof for Eq. (\ref{eqn:F_integral_Cube}) is similar to the proof of Eq. (\ref{eqn:F_integral_Gaussian}). 


Let $\{ L_\ell \}_{0 \le \ell \le d - 1}$ be the matrix representation of group elements of $\Cyc_d$ or $\TwoCyc_{d_1, d_2}$: when $\cG_d = \Cyc_d$, $g_\ell \in \Cyc_d$ gives matrix representation $L_\ell$ for $0 \le \ell \le d-1$; when $\cG_d = \TwoCyc_{d_1, d_2}$, $g_{st} \in \TwoCyc_{d_1, d_2}$ gives matrix representation $L_{s \times d_2 + t}$ for $0 \le s \le d_1 - 1$, $0 \le t \le d_2 - 1$. As a consequence, for either $\cG_d \in \{ \Cyc_d, \TwoCyc_{d_1, d_2}\}$, $L_0 = \id_d$ is the identity matrix. This gives 
\[
F_k(\bz) = \| \bz \|_2^{2k} / d^{k + 1} + \sum_{l = 1}^{d-1} \< \bz, L_l \bz \>^k / d^{k+1}. 
\]

\noindent
{\bf Step 1. The case $k = 1$. } For either $\cG_d \in \{ \Cyc_d, \TwoCyc_{d_1, d_2}\}$, we have $\E[\< \bz, L_l \bz\>] = 0$ for $1 \le l \le d-1$. As a consequence, we have 
\begin{equation}\label{eqn:cyclic_integral_F1}
\E_{\bz \sim \cN(\bzero, \id_d)}[F_1(\bz)] = \E[\| \bz \|_2^2 / d^2] + \sum_{l = 1}^{d-1} \E[\< \bz, L_l \bz \> / d^2] =\frac{1}{d}. 
\end{equation}
\noindent
{\bf Step 2. The case $k = 2$. } Note we have 
\[
\begin{aligned}
\E_{\bz \sim \cN(\bzero, \id_d)}[F_2(\bz)] =&~ \E[ \| \bx \|_2^4 / d^3 ] + \sum_{l = 1}^{d - 1} \E[ \< \bx, L_l \bx\>^2 / d^3 ] \\
=&~ \E\Big[\Big(\sum_{i = 1}^d x_i^2 \Big)^2 \Big] / d^3 +  \sum_{l = 1}^{d-1} \E\Big[ \Big(\sum_{i = 1}^d x_i (L_l \bx)_i \Big)^2 \Big] / d^3\\
=&~\sum_{i, j = 1}^d \E[x_i^2 x_j^2] / d^3 + \sum_{l = 1}^{d-1} \sum_{i, j = 1}^d \E[x_i (L_l \bx)_i x_j (L_l \bx)_j] / d^3 \\
=&~ \Big(\frac{1}{d} + \frac{2}{d^2}\Big) +  \sum_{l = 1}^{d-1} \sum_{i, j = 1}^d \E[x_i (L_l \bx)_i x_j (L_l \bx)_j] / d^3. 
\end{aligned}
\]
Note that for either $\cG_d \in \{ \Cyc_d, \TwoCyc_{d_1, d_2}\}$, for any $i \in [d]$ and $1 \le l \le d-1$, the random variable $(L_l \bx)_i$ is independent from $x_i$. This gives 
\[
0 \le  \sum_{l = 1}^{d-1} \sum_{i, j = 1}^d \E[x_i (L_l \bx)_i x_j (L_l \bx)_j] / d^3 \le \frac{2(d-1)d}{d^3} = \Theta_d(d^{-1}). 
\]
As a consequence, we have 
\begin{equation}\label{eqn:cyclic_integral_F2}
\E_{\bz \sim \cN(\bzero, \id_d)}[F_2(\bz)] = \Theta_d(d^{-1}).
\end{equation}


\noindent
{\bf Step 3. The case $k \ge 3$. } By the moment formula of the $\chi^2$ distribution, we have
\[
\E_{\bz \sim \cN(\bzero, \id_d)}[(\| \bz \|_2^2/d)^k] = 1 + o_d(1). 
\]
Moreover, for either $\cG_d \in \{\Cyc_d, \TwoCyc_{d_1, d_2} \}$, for any $l \neq 0$, we have
\[
\E[\< \bz, L_l \bz\>]/d = 0.
\]
As a consequence, by the Hanson-Wright inequality as in Lemma \ref{lem:Hanson_Wright}, for any fixed $k \ge 3$ and $\eps > 0$, we have 
\[
\E_{\bz \sim \cN(\bzero, \id_d)}\Big[\sup_{1 \le l \le d-1} (\< \bz, L_l \bz \> / d)^k \Big] = O_d(d^{-k/2 + \eps}). 
\]
Therefore, for $k \ge 3$, we have 
\[
\Big\vert \E_{\bz \sim \cN(\bzero, \id_d)}[F_k(\bz)] - \frac{1}{d} \E_{\bz \sim \cN(\bzero, \id_d)}[(\| \bz \|_2^2/d)^k] \Big\vert \le \E_{\bz \sim \cN(\bzero, \id_d)}\Big[\sup_{1 \le l \le d-1} (\< \bz, L_l \bz \> / d)^k \Big] = o_d(1/d),
\]
so that 
\begin{equation}\label{eqn:cyclic_integral_F3}
 \E_{\bz \sim \cN(\bzero, \id_d)}[F_k(\bz)] = 1/d + o_d(1/d). 
\end{equation}
Combining Eq. (\ref{eqn:cyclic_integral_F1}), (\ref{eqn:cyclic_integral_F2}), and (\ref{eqn:cyclic_integral_F3}) proves Eq. (\ref{eqn:F_integral_Gaussian}). 

\noindent
{\bf Step 4. From Gaussian to spherical. } Note that when $\bz \sim \cN(\bzero, \id_d)$, we have $\| \bz \|_2^2 \sim \chi^2(d)$ which is independent of $\sqrt d \cdot \bz / \| \bz \|_2 \sim \Unif(\S^{d-1}(\sqrt d))$. Hence, we have
\[
\begin{aligned}
\E_{\bz \sim \cN(\bzero, \id_d)}[F_k(\bz)] =&~ \E_{\btheta \sim \S^{d-1}(\sqrt d), \bz \sim \cN(\bzero, \id_d)}[F_k(\btheta) (\| \bz \|_2^{2k} / d^k)] \\
=&~ \E_{\btheta \sim \S^{d-1}(\sqrt d)}[F_k(\btheta)] \cdot \E_{\bz \sim \cN(\bzero, \id_d)}[\| \bz \|_2^{2k} / d^k]. 
\end{aligned}
\]
Note that for fixed $k \ge 1$, the moment formula for $\chi^2$ distribution gives
\[
\E_{\bz \sim \cN(\bzero, \id_d)}[\| \bz \|_2^{2k} / d^k] = 1 + o_d(1).
\]
Combining with Eq. (\ref{eqn:F_integral_Gaussian}), we have
\[
\E_{\btheta \sim \S^{d-1}(\sqrt d)}[F_k(\btheta)]  = \E_{\bz \sim \cN(\bzero, \id_d)}[F_k(\bz)]/ \E_{\bz \sim \cN(\bzero, \id_d)}[\| \bz \|_2^{2k} / d^k] = \Theta_d(d^{-1}). 
\]
This proves Eq. (\ref{eqn:F_integral_Spherical}). 
\end{proof}

\begin{proof}[Proof of Proposition \ref{prop:degeneracy_cyclic}]
Denote 
\[
P_k(\btheta) \equiv \frac{1}{B(\cA_d; k)}\sum_{l = 1}^{D(\cA_d; k)} \overline Y_{k l}^{(d)}(\btheta)^2 = \int_{\cG_d} Q_k^{(d)}(\< \btheta, g \cdot \btheta\> ) \pi_d(\de g). 
\]
By Lemma \ref{lem:isometry_of_orthonomal_polynomials}, for any fixed $k \ge 1$, we have
\begin{equation}\label{eqn:PDB_degeneracy}
\E_{\btheta \sim \Unif(\cA_d)}[P_k(\btheta)] = \frac{D(\cA_d; k)}{B(\cA_d; k)}. 
\end{equation}
By Lemma \ref{lem:gegenbauer_coefficients}, we have 
\[
P_k(\btheta) = \sum_{m = 0}^k a_{d, k, m} F_m(\btheta), 
\]
where $\vert a_{d, k, m} \vert \le C_{k, m} / d^{(k - m)/2}$. As a result, we have 
\[
\E_{\btheta \sim \Unif(\cA_d)}[P_k(\btheta)] =  \sum_{m = 0}^k a_{d, k, m} \E_{\btheta \sim \Unif(\cA_d)}[F_k(\btheta)] = \Theta(d^{-1}). 
\]
Combining with Eq. (\ref{eqn:PDB_degeneracy}) shows that $D(\cA_d; k) = \Theta(d^{-1} B(\cA_d; k)) = \Theta(d^{k-1})$. This concludes the proof. 
\end{proof}

\subsubsection{Auxiliary lemmas}

\begin{lemma}[Hanson-Wright inequality]\label{lem:Hanson_Wright}
There exists a universal constant $c > 0$, such that for any $t > 0$ and $d \in \N$, and any permutation matrix $L \in \R^{d \times d}$ be any permutation matrix, when $\bx \sim \cN(\bzero, \id_d)$ or $\bx \sim \Unif(\Cube^d)$, we have
\[
\P\Big( \Big\vert \< \bx, L \cdot \bx\> - \E[\< \bx, L \cdot \bx\>] \Big\vert / d \ge t \Big) \le 2 \cdot \exp\{ - c d \cdot \min(t^2, t) \}. 
\]
\end{lemma}

\begin{proof}[Proof of Lemma \ref{lem:Hanson_Wright}]
Note that for any permutation matrix $L$, we have $\| L \|_F \le \sqrt d$, and $\| L \|_{\op} \le 1$. By the Hanson-Wright inequality of vectors with independent sub-Gaussian entries (for example, see Theorem 1.1 of \cite{rudelson2013hanson}), we have 
\[
\P\Big(  \Big\vert \< \bx, L \bx\> - \E[\< \bx, L \bx\>] \Big\vert / d > t \Big) \le 2 \exp\{ - c d \cdot \min(t^2, t) \}. 
\]
This concludes the proof
\end{proof}

\begin{lemma}\label{lem:gegenbauer_coefficients}
Let $Q_k^{(d)}$ be either the $k$'th Gegenbauer polynomial or the $k$'th hypercubic Gegenbauer polynomial (as defined in Section \ref{sec:technical_background}). Let coefficients of monomials in $Q_k^{(d)}(d \cdot x)$ to be $\{ a_{d, k, m} \}_{0 \le m \le k}$. That is, we have
\[
Q_k^{(d)}(x) = \sum_{m = 0}^k a_{d, k, m} (x/d)^m. 
\]
Then, for any fixed $k$, there exists constant $C(k)$, such that 
\[
\vert a_{d, k, m} \vert \le C(k) / d^{(k - m)/2}. 
\]
Moreover, we have
\[
\lim_{d \to \infty} a_{d, k, k} = 1.
\]
Finally, for $k$ and $m$ in different parity, we have 
\[
a_{d, k, m} = 0. 
\]
\end{lemma}

\begin{proof}[Proof of Lemma \ref{lem:gegenbauer_coefficients}]
The proof holds by the following equation
\[
\lim_{d \to \infty} \Coeff\Big\{ B(\cA_d; k)^{1/2} Q_k^{(d)}(\sqrt d \cdot x)  \Big\} = \Coeff\Big\{ \frac{1}{\sqrt{k!}} \He_k(x) \Big\}. 
\]
when $Q_k^{(d)}$ is either Gegenbauer polynomial or Hypercubic Gegenbauer polynomial (See Eq. (\ref{eq:Gegen-to-Hermite}) and Eq. (\ref{eq:Hyper-Gegen-to-Hermite})). 
\end{proof}

\subsection{Counting the degeneracy of band-limited function class (Example \ref{ex:band-limited})}

\begin{proposition}\label{prop:degeneracy_band_limited}
Follow the notations of Example \ref{ex:band-limited}. Then for any fixed $k \ge 1$, we have 
\[
D(\S^{(d-1)}; k) = \Theta_d(d^{k-1}). 
\]
\end{proposition}

Here we state Lemma \ref{lem:band_limited_integral} that is used to prove Proposition \ref{prop:degeneracy_band_limited}. Given Lemma \ref{lem:band_limited_integral}, the proof of Proposition \ref{prop:degeneracy_band_limited} is the same as the proof of Proposition \ref{prop:degeneracy_cyclic}. 

\begin{lemma}\label{lem:band_limited_integral}
Follow the notations of Example \ref{ex:band-limited}. Denote 
\[
F_k(\bz) = \int_{\Sft_d} (\< \bz, g \cdot \bz\>/d)^k \pi_d(\de g). 
\]
Then for any fixed $k \ge 1$, we have
\[
\E_{\bz \sim \cN(\bzero, \id_d)}[F_k(\bz)] = \Theta_d(d^{-1}). 
\]
\end{lemma}

\def\imag{{\mathsf i}}

\begin{proof}[Proof of Lemma \ref{lem:band_limited_integral}]

We prove the lemma for the case when $d$ is odd. We denote $u_1 = z_1^2$, and $u_i = z_{2i}^2 + z_{2i+1}^2$ for $i = 2, \ldots, (d-1)/2$. Then we have 
\begin{equation}\label{eqn:proof_band_limited_integral_1}
\begin{aligned}
&~\E_{\bz \sim \cN(\bzero, \id_d)}[F_k(\bz)] = d^{-k} \cdot \E_\bz \Big\{\int_{[0, 1]} \Big( \sum_{j = 0}^{(d-1)/2} u_j \cos(2 \pi j t) \Big)^k \de t\Big\}\\
=&~d^{- k} \cdot \E_\bz \Big\{ \int_{[0, 1]} \sum_{j_1, \ldots, j_k = 0}^{(d-1)/2} \Big(\prod_{s \in [k]} u_{j_s}  \cos(2 \pi j_{s} t) \Big) \de t \Big\}\\
=&~ d^{- k}\cdot \sum_{j_1, \ldots, j_k = 0}^{(d-1)/2} \E_\bz \Big\{ \prod_{s \in [k]} u_{j_s} \Big\} \Big(\int_{[0, 1]} \prod_{s \in [k]} \cos(2 \pi j_s t) \de t \Big). \\
\end{aligned} 
\end{equation}

\noindent
{\bf Step 1. Bound $Z$ function. } First, we denote 
\[
Z(j_1, \ldots, j_k) =  \E_\bz \Big\{ \prod_{s \in [k]} u_{j_s} \Big\}. 
\]
We have
\begin{equation}\label{eqn:proof_band_limited_Z_upper}
\begin{aligned}
&~\sup_{j_1, \ldots, j_k}Z(j_1, \ldots, j_k) \le \sup_{j_1, \ldots, j_k} \prod_{s \in [k]} \E[ u_{j_s}^{2k}]^{1/(2k)}\\
 \le&~ \sup_{j \in \{ 0, 1, \ldots, (d-1)/2\} } \E[u_j^{2k}]^{1/2} \le 2^{k} \cdot \E_{G \sim \cN(0, 1)}[G^{2k}]^{1/2} \equiv M_k. 
\end{aligned}
\end{equation}
Moreover, we have 
\begin{equation}\label{eqn:proof_band_limited_Z_lower}
\begin{aligned}
&~\inf_{j_1, \ldots, j_k}Z(j_1, \ldots, j_k) \ge \E_{G \sim \cN(0, 1)}[G^2] = 1. 
\end{aligned}
\end{equation}

\noindent
{\bf Step 2. Bound $\vert \cI \vert$. } Further, we denote 
\[
\cI = \Big\{ (j_1, \ldots, j_k) \in \{ 0, \ldots, (d-1)/2 \}^k : \exists (\eps_i)_{i \in [k]} \in \{ \pm 1\}^k, \sum_{i = 1}^k \eps_i j_i = 0 \Big\},
\]
Then it is easy to see that 
\begin{equation}\label{eqn:proof_band_limited_I}
[(d+1)/2]^{k-1} \le \vert \cI \vert \le 2 \cdot (d+1)^{k-1}. 
\end{equation}

\noindent
{\bf Step 3. Bound $E$ function. } Next, we denote 
\[
E(j_1, \ldots, j_k) = \int_{[0, 1]} \prod_{s \in [k]} \cos(2 \pi j_s t) \de t. 
\]
It is easy to see that 
\begin{equation}\label{eqn:proof_band_limited_E_upper}
\sup_{j_1, \ldots, j_k}  \vert E(j_1, \ldots, j_k)\vert\le 1. 
\end{equation}
Moreover, for any $(j_1, \ldots, j_k) \not \in \cI$, we have $E(j_1, \ldots, j_k) = 0$. For any $(j_1, \ldots, j_k) \in \cI$, we have 
\begin{equation}\label{eqn:proof_band_limited_E_lower}
E(j_1, \ldots, j_k) =  \frac{1}{2^k} \int_{[0, 1]} \prod_{s \in [k]} [\exp( \imag 2 \pi j_s t) + \exp( - \imag 2 \pi j_s t)] \de t \ge 1/2^k. 
\end{equation}
The last inequality used the fact that $(j_1, \ldots, j_k) \in \cI$. 

\noindent
{\bf Step 4. Concludes the proof. } Therefore, combining Eq. (\ref{eqn:proof_band_limited_integral_1}) (\ref{eqn:proof_band_limited_Z_upper}) (\ref{eqn:proof_band_limited_I}) (\ref{eqn:proof_band_limited_E_upper}), we have 
\[
\begin{aligned}
\E_{\bz \sim \cN(\bzero, \id_d)}[F_k(\bz)] \le&~ d^{- k} \cdot M_k\cdot \sum_{j_1, \ldots, j_k = 0}^{(d-1)/2} \vert E(j_1, \ldots, j_k) \vert \\
\le&~ d^{- k} \cdot M_k \cdot \vert \cI \vert = O_d(d^{-1}). 
\end{aligned}
\]
Combining Eq. (\ref{eqn:proof_band_limited_integral_1}) (\ref{eqn:proof_band_limited_Z_lower}) (\ref{eqn:proof_band_limited_I}) (\ref{eqn:proof_band_limited_E_lower}), we have 
\[
\begin{aligned}
\E_{\bz \sim \cN(\bzero, \id_d)}[F_k(\bz)] \ge&~ d^{- k} \cdot \sum_{j_1, \ldots, j_k = 0}^{(d-1)/2} \vert E(j_1, \ldots, j_k) \vert \\
\ge&~ d^{- k} \cdot \vert \cI \vert / 2^k = \Omega_d(d^{-1}). 
\end{aligned}
\]
This concludes the proof. 
\end{proof}

\section{Concentration for invariant groups with degeneracy $\alpha \leq 1$}\label{sec:concen_alpha}

Let $Q_k^{(d)}$ be the $k$'th Gegenbauer polynomial on $\cA_d \in \{ \S^{d-1} (\sqrt{d}) , \Cube^d \}$ (see Section \ref{sec:technical_background} for definitions). Let $\cG_d$ be an invariant group with degeneracy $\alpha$. That means, for any fixed $k \ge \alpha$, we have $B(\cA_d ; k) / [D(\cA_d ; k) d^\alpha] = \Theta_d(1)$. For $k \in \N_{\ge 0}$, we denote
\begin{equation}\label{eqn:def_Upsilon_k_general}
\Upsilon_k(\btheta) = \frac{1}{D(\cA_d ; k)} \sum_{l = 1}^{D(\cA_d ; k)} \overline Y_{k l}^{(d)}(\btheta) \overline Y_{kl}^{(d)} (\btheta) = \frac{B(\cA_d ; k)}{D(\cA_d ; k)} \int_{\cG_d} Q_k^{(d)}(\< \btheta, g \cdot \btheta\> ) \pi_d(\de g). 
\end{equation}
Then we have 
\[
\E[\Upsilon_k(\btheta)] = 1. 
\]
In this section, we show that $\Upsilon_k$ concentration around its mean, for any fixed $k \ge 2$ and $\alpha \leq 1$.

\subsection{Main proposition}

\begin{proposition}\label{prop:general_concentration_diagonal} Let $\cG_d$ be an invariant group with degeneracy $\alpha \leq 1$. Let $(\btheta_i)_{i \in [N]} \sim \Unif(\cA_d)$ where $N = O_d(d^p)$ for some fixed integer $p$. Let $\Upsilon_k$ be as defined in Eq. (\ref{eqn:def_Upsilon_k_general}). Then for any fixed $k \ge 2$, we have 
\[
\sup_{i \in [N]} \Big\vert \Upsilon_k(\btheta_i) - 1 \Big\vert = o_{d, \P}(1). 
\]
\end{proposition}

\begin{proof}[Proof of Proposition \ref{prop:general_concentration_diagonal}]
Let us first focus on the sphere case $\cA_d = \S^{d-1} (\sqrt{d})$. Let $(\bx_i )_{i \in [N]} \sim_{iid} \cN(\bzero, \id_d)$. Without loss of generality, we assume $\bx_i$ and $\btheta_i$ are coupled such that $\btheta_i = \sqrt d \cdot \bx_i / \| \bx_i \|_2$. Denote 
\[
F_k(\bz) = \int_{\cG_d} (\< \bz, g \cdot \bz\>/ d)^k \pi_d(\de g). 
\]

Let $\{ a_{d, k, m} \}_{0 \le m \le k}$ be the coefficients of monomials in $Q_k^{(d)}(d \cdot x)$. That is, we have
\[
Q_k^{(d)}(x) = \sum_{m = 0}^k a_{d, k, m} (x/d)^m. 
\]
Then 
\[
\int_{\cG_d} Q_k^{(d)}(\< \btheta, g \cdot \btheta\> ) \pi_d(\de g) = \sum_{m = 0}^k a_{d, k, m} F_m (\btheta). 
\]
Moreover, by Lemma \ref{lem:gegenbauer_coefficients}, we have $\vert a_{d, k, m} \vert \le C_{k, m} / d^{(k - m)/2}$, $\lim_{d \to \infty} a_{d, k, k} = 1$, and $a_{d, k, m} = 0$ for $k$ and $m$ of different parity. 

Then we have 
\[
\begin{aligned}
&~ \sup_{i \in [N]} \vert \Upsilon_k(\btheta_i) - \E[\Upsilon_k(\btheta_i)]\vert \\
=&~ \frac{B(\S^{d-1}; k)}{D(\S^{d-1}; k)} \sup_{i \in [N]} \Big\vert \int_{\cG_d} Q_{k}(\< \btheta_i, g \cdot \btheta_i\> ) \pi_d(\de g) - \E\Big[ \int_{\cG_d} Q_{k}(\< \btheta_i, g \cdot \btheta_i\>) \pi_d(\de g) \Big] \Big\vert \\
\le&~ C \times d^\alpha \times \sum_{m = 1}^k a_{d, k, m} \times \sup_{i \in [N]} \Big\vert  F_m(\btheta_i) - \E [F_m(\btheta_i)] \Big\vert \\
\le&~ C \times d^\alpha \times \sum_{m = 1}^k a_{d, k, m} \times \sup_{i \in [N]} \Big\vert  F_m(\bx_i) - \E [F_m(\bx_i)] \Big\vert \cdot [d^m / \| \bx_i \|_2^{2m}]. 
\end{aligned}
\]
By the concentration of $\chi^2$-distribution, for any $\eps > 0$, the following event happens with high probability
\[
\cE_1 \equiv \Big\{ \sup_{i \in [N]} \big\vert \| \bx_i \|_2^2/ d - 1 \big\vert \le 1/ d^{1/2 - \eps} \Big\}. 
\]
Moreover, combining Lemma \ref{lem:Fk_moments_order} with Lemma \ref{lem:general_variance}, for any fixed $m \ge 2$, we have
\[
\E[(F_m(\bx) - \E[F_m(\bx)])^2] \le C_m  d^{-1 - 3 \alpha/2}. 
\]
By the hypercontractivity property of Gaussian distribution as per Lemma \ref{lem:hypercontractivity_Gaussian}, for any $\eps > 0$, taking $q$ sufficiently large, we have 
\[
\begin{aligned}
&~\E\Big[ \sup_{i \in [N]} \Big\vert F_m(\bx_i) - \E[F_m(\bx_i)] \Big\vert \Big] \le \E\Big[ \sum_{i = 1}^N \Big(F_m(\bx_i) - \E[F_m(\bx_i)] \Big)^{2q} \Big]^{1/(2q)} \\
\le&~ C(q) \cdot d^{p/(2q)} \cdot \E[(F_m(\bx) - \E[F_m(\bx)])^2]^{1/2}
\le C d^{-1 - 3 \alpha/2 + \eps} 
\end{aligned}
\]
By Markov's inequality, we deduce that the following event happens with high probability
\[
\cE_2 \equiv \Big\{\forall 2 \le m \le k,   \sup_{i \in [N]} \Big\vert F_m(\bx_i) - \E[F_m(\bx_i)] \Big\vert \le C d^{-1 - 3 \alpha/2 + \eps}  \Big\}. 
\]
Finally, by Lemma \ref{lem:Fk_moments_order}, we have 
\[
\E[F_1(\bx)^2] \le C d^{-2\alpha}, 
\]
and by the hypercontractivity property of low degree polynomials with Gaussian measure (Lemma \ref{lem:hypercontractivity_Gaussian}), for any $\eps > 0$, taking $q$ sufficiently large, we have 
\[
\begin{aligned}
\E\Big[ \sup_{i \in [N]} \Big\vert F_1(\bx_i) - \E[F_1(\bx_i)] \Big\vert \Big] \le &~c \E\Big[ \sum_{i = 1}^N F_1(\bx_i)^{2q} \Big]^{1/(2q)} + \E[F_1(\bx)^2]^{1/2} \\
\le&~ C(q) \cdot d^{p/(2q)} \cdot \E[F_1(\bx)^2]^{1/2} + \E[F_1(\bx)^2]^{1/2}
\le C d^{-\alpha + \eps}. 
\end{aligned}
\]
As a result, the following event happens with high probability
\[
\cE_3 \equiv \Big\{  \sup_{i \in [N]} \Big\vert F_1(\bx_i) - \E[F_1(\bx_i)] \Big\vert \le C d^{-\alpha + \eps} \Big\}. 
\]
When all the events $\cE_1$, $\cE_2$, and $\cE_3$ happen, for any $k \ge 2$, we have 
\[
\begin{aligned}
&~ \sup_{i \in [N]} \vert \Upsilon_k(\btheta_i) - \E[\Upsilon_k(\btheta_i)]\vert \\
\le&~ C \times d^{\alpha} \times \sum_{m = 1}^k a_{d, k, m} \times \sup_{i \in [N]} \Big\vert  F_m(\bx_i) - \E [F_m(\bx_i)] \Big\vert \cdot [d^m / \| \bx_i \|_2^{2m}] \\
\le&~ C \times d^{\alpha} \times \Big[  d^{-(k-1)/2} d^{-\alpha + \eps} + \sum_{m = 2}^k d^{- (k - m)/2} \times d^{-1 - 3 \alpha /2 + \eps} \Big] = o_d(1). 
\end{aligned}
\]
The case of the hypercube $\cA_d \sim \Cube^d$ follows similarly without introducing the gaussian measure and using Lemma \ref{lem:general_variance_hypercube} instead of Lemma \ref{lem:general_variance}.
\end{proof}

\subsection{Auxiliary Lemmas}

%
%
%

\begin{lemma}\label{lem:Fk_moments_order}
Let $\cG_d$ be an invariant group with degeneracy $\alpha \leq 1$. Denote 
\[
F_k(\bz) = \int_{\cG_d} (\< \bz, g \cdot \bz\>/ d)^k \pi_d(\de g). 
\]
Then for any fixed $s \in [1, \infty)$ and integer $k \ge 1$, we have
\begin{align}
\E_{\bx \sim \cN(\bzero, \id_d)}[F_k(\bx)^s ]^{1/s} = &~ O_d(d^{-\alpha}) , \\
\E_{\btheta \sim \Unif ( \cA_d ) }[F_k(\btheta)^s ]^{1/s} = &~ O_d(d^{-\alpha} ).
\end{align}
\end{lemma}

\begin{proof}[Proof of Lemma \ref{lem:Fk_moments_order}] 

For $\btheta \in \cA_d$, denote
\[
P_k(\btheta) \equiv \frac{1}{B(\cA_d ; k)}\sum_{l = 1}^{D(\cA_d ; k)} \overline Y_{k l}^{(d)}(\btheta)^2 = \int_{\cG_d} Q_k^{(d)}(\< \btheta, g \cdot \btheta\> ) \pi_d(\de g). 
\]
By Lemma \ref{lem:isometry_of_orthonomal_polynomials} and by the assumption that $\cG_d$ is an invariant group with degeneracy $\alpha \leq 1$, i.e., $B(\cA_k ; k) / [D(\cA_k ; k) d^\alpha] = \Theta_d(1)$, for any fixed $k \ge 1$, we have
\[
\E[P_k(\btheta)] = \frac{D(\cA_d ; k)}{B(\cA_d ; k)} = O_d(d^{-\alpha}). 
\]
Throughout the proof, we will denote $L^s = L^s(\cA_d)$ to be the $L^s$ space with respect to distribution $\btheta \sim \Unif(\cA_d)$. 

By the hypercontractivity of low degree polynomials on the sphere and the hypercube, as per Lemmas \ref{lem:hypercontractivity_hypercube} and \ref{lem:hypercontractivity_sphere}, for any $s \ge 1$, we have
\begin{equation}\label{eqn:Fk_moments_order_1}
\begin{aligned}
\| P_k \|_{L^s} =&~ \frac{D(\cA_d ;k)}{B(\cA_d ; k)} \Big\| \frac{B(\cA_d ; k)}{D(\cA_d ; k)} P_k \Big\|_{L^s} \le \frac{C_{k, s}}{d^\alpha} \Big\| \frac{B(\cA_d ; k)}{D(\cA_d ; k)} P_k \Big\|_{L^2} \\
=&~ \frac{C_{k, s}}{d^\alpha} \Big[ \frac{1}{D(\cA_d ; k)^2} \sum_{l_1, l_2 = 1}^{D(\cA_d ; k)} \E \big[\overline Y_{k l_1}^{(d)}(\btheta)^2 \overline Y_{kl_2}^{(d)} (\btheta)^2\big] \Big]^{1/2} \\
\le&~ \frac{C_{k, s}}{d^\alpha} \Big[\frac{1}{D(\cA_d ; k)^2} \sum_{l_1, l_2 = 1}^{D(\cA_d ; k)} \E\big[\overline Y_{k l_1}^{(d)}(\btheta)^4\big]^{1/2} \E\big[\overline Y_{kl_2}^{(d)} (\btheta)^4\big]^{1/2} \Big]^{1/2}\\
\le&~ \frac{C_{k, s}}{d^\alpha} \Big[\frac{1}{D(\cA_d ; k)^2} \sum_{l_1, l_2 = 1}^{D(\cA_d ; k)} \E\big[\overline Y_{k l_1}^{(d)}(\btheta)^2\big] \E\big[\overline Y_{kl_2}^{(d)} (\btheta)^2\big] \Big]^{1/2} = \frac{C_{k, s}}{d^\alpha}. 
\end{aligned}
\end{equation}

Let $\{ a_{d, k, m} \}_{0 \le m \le k}$ be the coefficients of monomials in $Q_k^{(d)}(d \cdot x)$. That is, we have
\[
Q_k^{(d)}(x) = \sum_{m = 0}^k a_{d, k, m} (x/d)^m. 
\]
Then 
\begin{equation}\label{eqn:Fk_moments_order_2}
P_k = \sum_{m = 0}^k a_{d, k, m} F_m. 
\end{equation}
Moreover, by Lemma \ref{lem:gegenbauer_coefficients}, we have $\vert a_{d, k, m} \vert \le C_{k, m} / d^{(k - m)/2}$, $\lim_{d \to \infty} a_{d, k, k} = 1$, and $a_{d, k, m} = 0$ for $k$ and $m$ have different parity. 

We conclude the proof by induction over $k$. Note we have $F_0(\btheta) \equiv 1$. Moreover, for any $s \ge 1$, by Eq. (\ref{eqn:Fk_moments_order_1}) and (\ref{eqn:Fk_moments_order_2}) (and note that $a_{d, 1, 0} = 0$ and $\lim_{d \to \infty} a_{d, 1, 1} \to 1$), we have 
\[
\begin{aligned}
\| F_1\|_{L^s} =&~  \frac{1}{a_{d, 1, 1}} \| P_1 \|_{L^s} \le C_s / d^\alpha. \\
\end{aligned} 
\]
Fix a $k \ge 2$. Assume that, for any $1 \le u \le k-1$, we have $\| F_u \|_{L^s} \le C_u/ d^\alpha$ for $s \ge 1$, by Eq. (\ref{eqn:Fk_moments_order_2}) and (\ref{eqn:Fk_moments_order_1}), and the fact that $\vert a_{d, k, m} \vert \le C_{k, m} / d^{(k - m)/2}$ and $\lim_{d \to \infty} a_{d, k, k} = 1$, we have 
\[
\begin{aligned}
\| F_k \|_{L^s} =&~  \Big \| \frac{1}{a_{d, k, k}} P_k - \sum_{m = 1}^{k-1} a_{d, k, m} F_m - a_{d, k, 0} \Big\|_{L^s} \\
\le&~ \Big \| \frac{1}{a_{d, k, k}} P_k \Big\|_{L^s} + \sum_{m = 1}^{k-1} \vert a_{d, k, m} \vert \Big\| F_m \Big\|_{L^s} + \vert a_{d, k, 0}\vert \\
\le&~  C / d^\alpha + \Big[\sum_{m = 1}^{k-1} C / d^{(k - m)/2}\Big] \cdot C / d^\alpha + C /d^{k/2}  \le C_{k, s} / d^\alpha,
\end{aligned} 
\]
where we recall that we assume $\alpha \leq 1$.

Finally, for the case of $\cA_d = \S^{d-1} (\sqrt{d})$, recalling that we can write
\[
F_k ( \bx ) = (\| \bx \|_2^2 /d )^k F_k ( \btheta ),
\]
where $\btheta = \bx / \| \bx_i \|_2 \sim \Unif(\S^{d-1} (\sqrt{d})$ is independent of $\| \bx_i \|_2$ in the case of $\bx \sim \normal(\bzero, \id_d)$. Hence, we get by Cauchy-Schwarz inequality
\[
\begin{aligned}
\E_{\bx \sim \normal(\bzero, \id_d) }[F_k(\bx)^s ]^{1/s} =  
\| F_k \|_{L^{2s}} \cdot \E_{\bx \sim \normal(\bzero, \id_d)}\Big[ (\| \bx\|_2^2/d)^{2sk} \Big]^{1/(2s)} \le C_{k, s} / d^\alpha,  
\end{aligned}
\]
by hypercontractivity of low degree polynomials for Gaussian measure (Lemma \ref{lem:hypercontractivity_Gaussian}).
\end{proof}

\begin{lemma}\label{lem:general_variance}
Let $\bx \sim \normal(\bzero, \id_d)$.  Let $\cG_d$ be a general invariant group. Let $F_k (\bz)$ be defined as in Lemma \ref{lem:Fk_moments_order}. Then for any fixed $k \ge 1$, there exists a constant $C_k$, such that
\begin{itemize}
    \item If $k$ is odd, then
    \[
\Var_{\bx \sim \normal(\bzero, \id_d)}[F_k(\bx)] \le   \frac{C_k}{d} \E_\bx[F_{k-1}(\bx)^2].
\]
\item If $k$ is even, then
\[
\Var_{\bx \sim \normal(\bzero, \id_d)}[F_k(\bx)] \le 
\frac{C_k}{d} \Big( \E_\bx[F_{k-2}(\bx)^2] \wedge \E_\bx[F_k(\bx)^2]^{(2k-1)/(2k)} \Big).
\]
\end{itemize}

\end{lemma}

\begin{proof}[Proof of Lemma \ref{lem:general_variance}]

By the Gaussian Poincar\'e inequality, we have 
\[
\E_{\bx \sim \normal (\bzero, \id_d)}[(F_k(\bx) - \E[F_k(\bx)])^2] \le \E[\| \nabla F_k (\bx) \|_2^2]. 
\]
We have 
\[
\nabla F_k (\bx) =k \int_{\cG_d} (\< \bx, g \cdot \bx\>/ d)^{k-1} [(g \cdot \bx + g^{-1} \cdot \bx)/d] \pi_d(\de g), 
\]
which gives
\begin{equation}\label{eqn:gradient_square_calculation}
\begin{aligned}
&~ \E[\| \nabla F_k (\bx) \|_2^2]\\
= &~ \frac{4 k^2}{d} \int_{\cG_d \times \cG_d} \E\Big[(\< \bx, g_1 \cdot \bx\>/ d)^{k-1} (\< \bx, g_2 \cdot \bx\>/ d)^{k-1} \< \bx, g_1 g_2 \cdot \bx\>/d \Big] \pi_d(\de g_1) \pi_d(\de g_2). 
\end{aligned}
\end{equation}

\noindent
{\bf Case 1: Odd $k$. } When $k$ is odd, we have 
\[
\begin{aligned}
&~ \E[\| \nabla F_k (\bx) \|_2^2] \\
\le&~ \frac{4 k^2}{d} \int_{\cG_d \times \cG_d} \E\Big[(\< \bx, g_1 \cdot \bx\>/ d)^{k-1} (\< \bx, g_2 \cdot \bx\>/ d)^{k-1} \| \bx \|_2^2 /d \Big] \pi_d(\de g_1) \pi_d(\de g_2)\\
=&~ \frac{4 k^2}{d} \E \big[ F_{k-1}(\bx)^2 (\| \bx \|_2^2 / d) \big] 
\le  \frac{4 k^2}{d} \E \big[ F_{k-1}(\bx)^4 \big]^{1/2} \E \big[(\| \bx \|_2^2 / d)^2 \big]^{1/2}
\le \frac{C_k}{d} \E [ F_{k-1}(\bx)^2] \, ,
\end{aligned}
\]
where we used in the second line Cauchy-Schwarz inequality and that the matrix representations of $g$ are orthogonal matrices, and in the last inequality the hypercontractivity of low degree polynomials for Gaussian measures (Lemma \ref{lem:hypercontractivity_Gaussian}).

\noindent
{\bf Case 2: Even $k$. Bound 1. } When $k$ is even, we have the following first bound
\[
\begin{aligned}
&~\E[\| \nabla F_k (\bx) \|_2^2] \\
\le&~ \frac{4 k^2}{d} \int_{\cG_d \times \cG_d} \E\Big[(\< \bx, g_1 \cdot \bx\>/ d)^{k-2} (\< \bx, g_2 \cdot \bx\>/ d)^{k-2} \| \bx \|_2^6 /d^3 \Big] \pi_d(\de g_1) \pi_d(\de g_2)\\
=&~ \frac{4 k^2}{d} \E \big[ F_{k-2}(\bx)^2 (\| \bx \|_2^6 / d^3) \big]
\le \frac{4 k^2}{d} \E \big[ F_{k-2}(\bx)^4 \big]^{1/2} \E \big[(\| \bx \|_2^6 / d^3)^2 \big]^{1/2}
\le \frac{C_k}{d} \E \big[ F_{k-2}(\bx)^2 \big] \, . 
\end{aligned}
\]

\noindent
{\bf Case 3: Even $k$. Bound 2. } When $k$ is even, we have the following second bound, which follows by H\"older's inequality $\E[X Y] \le \E[\vert X \vert^{k/(k-1)}]^{(k-1)/k} \cdot \E[\vert Y \vert^k]^{1/k}$, 
\[
\begin{aligned}
&~ \E[\| \nabla F_k (\bx) \|_2^2]\\
= &~\frac{4 k^2}{d} \int_{\cG_d \times \cG_d} \E\Big[(\< \bx, g_1 \cdot \bx\>/ d)^{k-1} (\< \bx, g_2 \cdot \bx\>/ d)^{k-1} \< \bx, g_1 g_2 \cdot \bx\>/d \Big] \pi_d(\de g_1) \pi_d(\de g_2) \\
\le&~ \frac{4 k^2}{d} \E\Big[ \int_{\cG_d \times \cG_d}(\< \bx, g_1 \cdot \bx\>/ d)^{k} (\< \bx, g_2 \cdot \bx\>/ d)^{k} \pi_d(\de g_1) \pi_d(\de g_2) \Big]^{(k-1)/k} \\
&~ \times \E\Big[\int_{\cG_d} (\< \bx, g \cdot \bx\> / d )^k \pi_d(\de g) \Big]^{1/k} \\
=&~ \frac{4 k^2}{d} \E \big[F_k(\bx)^2 \big]^{(k-1)/k} \times \E \big[F_k(\bx)\big]^{1/k} \le \frac{4 k^2}{d} \E \big[F_k(\bx)^2\big]^{(2k-1)/(2k)}. 
\end{aligned}
\]
Combining these two bounds yields the result for $k$ even. 
\end{proof}

\begin{lemma}\label{lem:general_variance_hypercube}
Let $\btheta \sim \Unif( \Cube^d )$.  Let $\cG_d$ be a general invariant group that preserves $\Cube^d$. Let $F_k (\bz)$ be defined as in Lemma \ref{lem:Fk_moments_order}.
Then for any fixed $k \ge 1$, there exists a constant $C_k$, such that
\begin{itemize}
    \item If $k$ is odd, then
    \[
\Var_{\btheta \sim \Unif( \Cube^d )}[F_k(\btheta)] \le   \frac{C_k}{d} \E_\btheta [F_{k-1}(\btheta)^2] + \frac{C_k}{d^3}.
\]
\item If $k$ is even, then
\[
\Var_{\btheta \sim \Unif( \Cube^d )}[F_k(\btheta)] \le 
\frac{C_k}{d} \Big( \E_\btheta[F_{k-2}(\btheta)^2] \wedge \E_\btheta [F_k(\btheta)^2]^{(2k-1)/(2k)} \Big) + \frac{C_k}{d^3}.
\]
\end{itemize}

\end{lemma}

\begin{proof}[Proof of Lemma \ref{lem:general_variance_hypercube}]

The proof is similar to the proof of Lemma \ref{lem:general_variance}. By the discrete Poincar\'e inequality, we have 
\[
\E_{\btheta \sim \Unif(\Cube_d)} \big[ \big( F_k ( \btheta) - \E [ F_k ( \btheta) ] \big)^2 \big] \leq \E_\btheta \Big[ \sum_{i =1}^d D_i F_k ( \btheta )^2 \Big] \, ,
\]
where $D_i$ denote the discrete derivative defined as
\[
D_i f ( \btheta) = \frac{f(\btheta) - f( \btheta_{-i} ) }{2},
\]
with $\btheta_{-i} = ( \theta_1 , \ldots , \theta_{i-1}, -\theta_i , \theta_{i+1} , \ldots , \theta_d)$. Let $\varphi_g ( \btheta ) = (\< \btheta , g \cdot \btheta\>/d)^k$, then 
\[
\begin{aligned}
D_i \varphi_g ( \btheta ) =&~ \frac{(\< \btheta , g \cdot \btheta\>/d)^k - (\< \btheta_{-i} , g \cdot \btheta\>/d)^k }{2} + \frac{(\< \btheta_{-i} , g \cdot \btheta \>/d)^k - (\< \btheta_{-i} , g \cdot \btheta_{-i}\>/d)^k }{2} \\
= :&~ D_{i,1} \varphi_g ( \btheta ) + D_{i,2} \varphi_g ( \btheta ).
\end{aligned}
\]
We have $\< \btheta_{-i} , g \cdot \btheta\> = \< \btheta , g \cdot \btheta\> - 2 \theta_i (g \cdot \btheta)_i$. By Taylor expansion, the first term verifies (recall that $g \cdot \btheta \in \Cube^d$ and $\theta_i^2 (g \cdot \btheta)_i^2 = 1$)
\[
D_{i,1} \varphi_g ( \btheta ) = k (\< \btheta , g \cdot \btheta\>/d)^{k-1} (\theta_i (g \cdot \btheta)_i /d ) - k (k -1 ) X_{i,1}(\btheta,g) ^{k-2} /d^2,
\]
where $X_{i,1}(\btheta , g)$ is on the line segment between $\< \btheta , g \cdot \btheta\>/d$ and $\< \btheta_{-i} , g \cdot \btheta\>/d$. Similarly, Taylor expansion on the second term yields
\[
D_{i,2} \varphi_g ( \btheta ) = k (\< \btheta_{-i} , g \cdot \btheta_{-i}\>/d)^{k-1} (\theta_i (g^{-1} \cdot \btheta_{-i})_i /d ) + k (k -1 ) X_{i,2}(\btheta,g) ^{k-2} /d^2,
\]
where $X_{i,2} (\btheta , g)$ is on the line segment between $\< \btheta_{-i} , g \cdot \btheta_{-i}\>/d$ and $\< \btheta , g \cdot \btheta_{-i}\>/d$.

Using Jensen's inequality to separate each of the $4$ terms in $D_{i} \varphi_g ( \btheta )$, using that $\btheta_{-i}$ and $\btheta$ have the same distribution, we get 
\[
\begin{aligned}
&~ \E_\btheta \Big[ \sum_{i =1}^d D_i F_k ( \btheta )^2 \Big] \\
\leq &~  \frac{32 k^2}{d} \int_{\cG_d \times \cG_d} \E\Big[(\< \btheta, g_1 \cdot \btheta\>/ d)^{k-1} (\< \btheta, g_2 \cdot \btheta\>/ d)^{k-1} \< \btheta, g_1 g_2 \cdot \btheta\>/d \Big] \pi_d(\de g_1) \pi_d(\de g_2) \\
&~ + \frac{16 k^2 ( k-1)^2 }{d^4}   \sum_{s \in \{ 1,2\}} \sum_{i = 1}^d \int_{\cG_d \times \cG_d} \E\Big[ X_{i,s} (\btheta , g_1)^{k-2} X_{i,s} (\btheta , g_2)^{k-2} \Big] \pi_d(\de g_1) \pi_d(\de g_2) \, .
\end{aligned}
\]
Noticing that $\sup_{i, s, \btheta, g}|X_{i,s}(\btheta, g)| \leq 1$, the second term in the above equation can be bounded by $C_k/d^3$. The first term in the above equation can be bounded using the same way as bounding the right hand side of Eq. (\ref{eqn:gradient_square_calculation}) as in the proof of Lemma \ref{lem:general_variance}. This concludes the proof. 
\end{proof}

\section{Kernel concentration for the cyclic group and general $\sigma$}\label{sec:concentration_cyc}

Throughout this section, we will always take $\cG_d = \Cyc_d$ to be the cyclic group, and $\cA_d = \S^{d-1}(\sqrt{d})$ to be the sphere. We will write in short $B(d, k) = B(\S^{d-1}(\sqrt{d}); k)$ and $D(d, k) = D(\S^{d-1}(\sqrt{d}); \Cyc_d; k)$. We recall that the cyclic group has degeneracy $1$, i.e., for each integers $k \ge 1$, $B(d,k) / D(d,k) = \Theta_d (d)$.

\subsection{Main propositions}

Let the Gegenbauer decomposition of $\sigma$ be
\[
\sigma(x) = \sum_{k = 0}^\infty \xi_{d, k}(\sigma) B(d, k) Q_k^{(d)} (\sqrt d x). 
\]
For $S \subseteq \N$, we define 
\begin{equation}\label{eqn:def_sigma_S}
\sigma_{d, S}(x) = \sum_{k \in S} \xi_{d, k}(\sigma) B(d, k) Q_k^{(d)} (\sqrt d x). 
\end{equation}
For any $\| \btheta_1 \|_2 = \| \btheta_2 \|_2 = \sqrt d$ and any $S \subseteq \Z_{\ge 0}$, denote 
\[
\begin{aligned}
h_{d, S}(\< \btheta_1, \btheta_2\> / d) \equiv&~ \E_{\bx \sim \Unif(\S^{d-1}(\sqrt d))}[\sigma_{d, S}(\< \btheta_1, \bx\>/ \sqrt d) \sigma_{d, S}(\< \btheta_2, \bx\> / \sqrt d)]. 
\end{aligned}
\]

\begin{proposition}\label{prop:h_concentration_general_sigma}
Let $\ell \ge 2$ be a fixed integer. Assume that $\sigma \in C^{\ell \vee 3}(\R)$ be a $\ell \vee 3$'th continuously differentiable function with derivatives satisfy $\sup_{0\le k \le \ell \vee 3} \sigma^{(k)}(u) \le c_0\, \exp(c_1\, u^2/2)$ for some constants $c_0 > 0$ and $c_1 < 1$. 

Define $H_{d, S}: \S^{d-1}(\sqrt{d})\times \S^{d-1}(\sqrt{d})\to\reals$ via
\begin{align}
H_{d, S}(\btheta_1,\btheta_2) \equiv \int_{\Cyc_d} h_{d, S}(\< \btheta_1, g \cdot \btheta_2\>/d) \pi_d(\de g).
\end{align}
Then, for $N  = d^p$ for any fixed $p$, letting $(\btheta_i)_{i \in [N]}\sim \Unif(\S^{d-1}(\sqrt{d}))$, we have
\begin{align}
\sup_{i \in [N]} \Big\vert H_{d, \ge \ell}(\btheta_i,\btheta_i)- \E H_{d, \ge \ell}(\btheta, \btheta) \Big\vert = o_{d,\P}(1) \cdot \E H_{d, \ge \ell}(\btheta,\btheta).
\end{align}
Moreover, we have $\E H_{d, \ge \ell}(\btheta,\btheta) = O_d(d^{-1})$. 
\end{proposition}

\begin{proof}[Proof of Proposition \ref{prop:h_concentration_general_sigma}] 
We let $C$, $C_k$, $C_{k, \ell}$ be constants that depend on $\sigma$, $k$, and $\ell$ but independent of dimension $d$. The exact values of these constant can change from line to line. 


\noindent
{\bf Step 1. Finite subset $S \subseteq \{2 , 3 , \ldots \}$. } Note we have 
\[
H_{d, S}(\btheta_1, \btheta_2) \equiv \sum_{k \in S} \xi_{d, k}(\sigma)^2 B(d, k) \int_{\Cyc_d } Q_k^{(d)}(\< \btheta_1, g \cdot \btheta_2\>) \pi_d(\de g). 
\]
By Lemma \ref{lem:isometry_of_orthonomal_polynomials} and Proposition \ref{prop:degeneracy_cyclic}, for any $S \subseteq \N$ with finite cardinality $\vert S \vert < \infty$, we have 
\[
\E[H_{d, S}(\btheta, \btheta)] = \sum_{k \in S} \xi_{d, k}(\sigma)^2 D(d, k) = \Theta_d(d^{-1}). 
\]
Moreover, by Proposition \ref{prop:general_concentration_diagonal}, we have 
\[
\sup_{i \in [N]} \Big\vert H_{d, S}(\btheta_i, \btheta_i) - \E[H_{d, S}(\btheta, \btheta)] \Big\vert  \le \sum_{k \in S} \xi_{d, k}(\sigma)^2 D(d, k) \sup_{i \in [N]} \Big\vert \Upsilon_k(\btheta_i) - 1 \Big\vert = o_d(1)\cdot \E[H_{d, S}(\btheta, \btheta)]. 
\]

\noindent
{\bf Step 2. For general set $S = \{u : u \ge \ell \}$. } 

By Lemma \ref{lem:bound_h_d_ge_ell_derivative_uniform}, we have $\sup_{d \ge 1} \sup_{\gamma \in [-1, 1]} \vert h_{d, \ge \ell}^{(\ell)}(\gamma) \vert \le C_\ell$. Therefore, for any $\gamma \in [-1, 1]$, we have
\begin{align}\label{eqn:taylor_hdgeell}
\Big\vert h_{d, \ge \ell}( \gamma ) - \sum_{k=0}^{\ell - 1} \frac{1}{k!} h_{d, \ge \ell}^{(k)}(0) \gamma^k \Big\vert \le C_{\ell} \cdot \vert \gamma \vert^{\ell+1}. 
\end{align}
By Lemma \ref{lem:h_d_derivative_bound_at_0}, for any $k \le \ell - 1$, we have 
\begin{align}\label{eqn:h_concentration_general_sigma_01}
\Big\vert h_{d, \ge \ell}^{(k)}(0) \Big\vert \le C_{k, \ell} \cdot d^{-(\ell-k)/2}. 
\end{align}
Moreover, by the Hanson-Wright inequality as in Lemma \ref{lem:Hanson_Wright}, since $N$ is at most polynomial in $d$, then for any $\delta>0$, we have
\begin{equation}\label{eqn:h_concentration_general_sigma_02}
\begin{aligned}
\sup_{1 \le k \le \ell + 1} \sup_{g \in \Cyc_d \setminus \id } \sup_{i \in [N]} \Big\vert \< \btheta_i, g \cdot \btheta_i \>^k \Big\vert \cdot d^{ - k/2 - \delta} = o_{d, \P}(1),
\end{aligned}
\end{equation}
and 
\begin{equation}\label{eqn:h_concentration_general_sigma_03}
\begin{aligned}
\sup_{1 \le k \le \ell + 1} \sup_{g \in \Cyc_d \setminus \id } \E\Big[ \Big\vert \< \btheta_i, g \cdot \btheta_i \>^k \Big\vert \Big]  \cdot d^{ - k/2 - \delta} = o_{d}(1). 
\end{aligned}
\end{equation}
Therefore, by Eq. (\ref{eqn:taylor_hdgeell}), (\ref{eqn:h_concentration_general_sigma_01}), (\ref{eqn:h_concentration_general_sigma_02}) and (\ref{eqn:h_concentration_general_sigma_03}), we have 
\[
\sup_{g \in \Cyc_d \setminus \id } \sup_{i \in [N]} \Big\vert h_{d, \ge \ell}(\< \btheta_i, g \cdot \btheta_i\> / d) \Big\vert = O_{d, \P}(d^{-\ell/2 + \delta}), 
\]
and 
\[
\sup_{g \in \Cyc_d \setminus \id } \E\Big[\Big\vert h_{d, \ge \ell}(\< \btheta, g \cdot \btheta\> / d) \Big\vert \Big] = O_{d}(d^{-\ell/2 + \delta}). 
\]
As a result, for any $\ell \ge 3$, we have
\[
\begin{aligned}
&~\sup_{i \in [N]} \Big\vert H_{d, \ge \ell}(\btheta_i,\btheta_i)- \E H_{d, \ge \ell}(\btheta,\btheta)\Big\vert \\
=&~ \sup_{i \in [N]} \Big\vert \int_{\Cyc_d \setminus \id}h_{d, \ge \ell}(\<\btheta_i, g \cdot \btheta_i\>/d) \pi_d(\de g)- \E \int_{\Cyc_d\setminus \id}h_{d, \ge \ell}(\<\btheta, g \cdot \btheta\>/d)\pi_d(\de g)\Big\vert \\
\le&~ \sup_{i \in [N]} \Big\vert \int_{\Cyc_d \setminus \id}h_{d, \ge \ell}(\<\btheta_i, g \cdot \btheta_i\>/d) \pi_d(\de g) \Big\vert + \Big\vert \E \int_{\Cyc_d\setminus \id}h_{d, \ge \ell}(\<\btheta, g \cdot \btheta\>/d)\pi_d(\de g)\Big\vert \\
\le&~ O_{d, \P}(d^{- \ell/2 + \delta}) = o_{d, \P}(d^{-1}). 
\end{aligned}
\]
By the arguments in Step 1, for any $\ell \ge 2$, we have
\[
\sup_{i \in [N]} \Big\vert H_{d, \ge \ell}(\btheta_i,\btheta_i)- \E H_{d, \ge \ell}(\btheta,\btheta)\Big\vert  \le o_{d, \P}(d^{-1}). 
\] 
Finally,  for any $\ell \ge 2$, for any $\sigma$ such that $\sigma_{d, \ge \ell}$ that is non-trivial (if $\sigma_{d, \ge \ell} = 0$, this proposition holds trivially), we have 
\[
\E H_{d, \ge \ell}(\btheta,\btheta)  = \Theta_d( d^{-1}). 
\]
This proves the proposition. 
\end{proof}

\begin{proposition}\label{prop:hh_concentration_general_sigma}
Let $\ell \ge 2$ be a fixed integer. Assume that $\sigma \in C(\R)$ be a continuous function with $| \sigma(u) | \le c_0\, \exp(c_1\, u^2/2)$ for some constants $c_0 > 0$ and $c_1 < 1$. 

Define $H_{d, S}: \S^{d-1}(\sqrt{d})\times \S^{d-1}(\sqrt{d})\to\reals$ via
\begin{align}
H_{d, S}(\btheta_1,\btheta_2) \equiv \int_{\Cyc_d} h_{d, S}(\< \btheta_1, g \cdot \btheta_2\>/d) \pi_d(\de g).
\end{align}
Then, for $N  = O(d^p)$ for any fixed $p$, letting $(\btheta_i)_{i \in [N]}\sim \Unif(\S^{d-1}(\sqrt{d}))$, we have
\begin{align}
\sup_{i \in [N]} \Big\vert \E_\btheta[H_{d, \ge \ell}(\btheta_i,\btheta)^2] - \E_{\btheta, \btheta'} [H_{d, \ge \ell}(\btheta', \btheta)^2] \Big\vert = o_{d,\P}(1) \cdot \E_{\btheta, \btheta'} [H_{d, \ge \ell}(\btheta', \btheta)^2].
\end{align}
\end{proposition}

\begin{proof}[Proof of Proposition \ref{prop:hh_concentration_general_sigma}]

Denoting $\mu_k(\sigma) = \E_{G \sim \cN(0, 1)}[\sigma(G) \He_k(G)]$. Let $q = \min\{ k \ge \ell: \mu_k(\sigma) \neq 0\}$ and let $u = q + 2$. We consider the case when $q < \infty$, since for $q = \infty$, the claim holds trivially. We have the expression 
\[
\begin{aligned}
\E_\btheta[H_{d, \ge \ell}(\btheta_i,\btheta)^2] =&~ \sum_{k = \ell}^\infty \xi_{d, k}(\sigma)^4 B(d, k) \int_{\Cyc_d} Q_k^{(d)}(\< \btheta_i, g \cdot \btheta_i\> ) \pi_d(\de g) \\
=&~  \E_\btheta[H_{d, [\ell, u)}(\btheta_i,\btheta)^2] + \E_\btheta[H_{d, \ge u}(\btheta_i,\btheta)^2]. 
\end{aligned}
\]

\noindent
{\bf Step 1. Upper bounding $\E_{\btheta, \btheta'}[H_{d, \ge u}(\btheta',\btheta)^2]$ and $\E_{\btheta}[H_{d, \ge u}(\btheta_i, \btheta)^2]$. }We have 
\begin{equation}\label{eqn:hh_concentration_general_sigma_1}
\begin{aligned}
&~\sup_{\btheta_i}\E_{\btheta}[H_{d, \ge u}(\btheta_i,\btheta)^2]  \\
=&~ \sup_{\btheta_i}\sum_{k = u}^\infty \xi_{d, k}(\sigma)^4 \sum_{l \in [D(d, k)]} \overline Y_{kl}^{(d)}(\btheta_i)^2 \le \sup_{\btheta_i} \sum_{k = u}^\infty \xi_{d, k}(\sigma)^4 \sum_{l \in [B(d, k)]} Y_{kl}^{(d)}(\btheta_i)^2 \\
=&~  \sum_{k = u}^\infty \xi_{d, k}(\sigma)^4 B(d, k)  Q_k^{(d)}(d)= \sum_{k = u}^\infty \xi_{d, k}(\sigma)^4 B(d, k) \\
\le&~ \Big[ \sup_{k \ge u} B(d, k)^{-1} \Big] \cdot \Big[\sum_{k = u}^\infty \xi_{d, k}(\sigma)^2 B(d, k)\Big]^2 \\
=&~ B(d, u)^{-1} \E_{\bx, \btheta \sim \Unif(\S^{d-1}(\sqrt d))} [\sigma(\< \bx, \btheta\>/\sqrt d)^2] = \Theta_d(d^{-u}). 
\end{aligned}
\end{equation}
This also gives 
\begin{equation}\label{eqn:hh_concentration_general_sigma_2}
\E_{\btheta, \btheta'}[H_{d, \ge u}(\btheta', \btheta)^2] = \Theta_d(d^{-u}). 
\end{equation}

\noindent
{\bf Step 2. Upper bounding $\sup_{i \in [N]} \vert \E_\btheta[H_{d, [\ell, u)}(\btheta_i,\btheta)^2] - \E_{\btheta, \btheta'}[H_{d, [\ell, u)}(\btheta',\btheta)^2] \vert$. } By Proposition \ref{prop:general_concentration_diagonal}, we have 
\begin{equation}\label{eqn:hh_concentration_general_sigma_3}
\begin{aligned}
&~\sup_{i \in [N]} \Big\vert \E_\btheta[H_{d, [\ell, u)}(\btheta_i,\btheta)^2] - \E_{\btheta, \btheta'}[H_{d, [\ell, u)}(\btheta',\btheta)^2] \Big\vert \\
\le&~ \sum_{k = \ell}^{u - 1} \xi_{d, k}(\sigma)^4 D(d, k) \sup_{i \in [N]} \Big\vert \frac{B(d, k)}{D(d, k)} \int_{\Cyc_d} Q_k^{(d)}(\< \btheta_i, g \cdot \btheta_i\> )\pi_d(\de g) - 1 \Big\vert \\
=&~ o_{d, \P}(1) \cdot \Big[ \sum_{k = \ell}^{u - 1} \xi_{d, k}(\sigma)^4 D(d, k)\Big] = o_{d, \P}(1) \cdot \E_{\btheta, \btheta'}[H_{d, [\ell, u)}(\btheta',\btheta)^2]. 
\end{aligned}
\end{equation}

\noindent
{\bf Step 3. Lower bounding $\E_{\btheta, \btheta'}[H_{d, \ge \ell}(\btheta',\btheta)^2]$. } We have 
\begin{equation}\label{eqn:hh_concentration_general_sigma_4}
\begin{aligned}
\E_{\btheta, \btheta'}[H_{d, \ge \ell}(\btheta',\btheta)^2] =&~ \sum_{k = \ell }^\infty \xi_{d, k}(\sigma)^4 D(d, k)  \ge \xi_{d, q}(\sigma)^4 D(d, q)  = \Theta_d(d^{-q - 1}). 
\end{aligned}
\end{equation}
The last equality is by Proposition \ref{prop:degeneracy_cyclic}, and the fact that 
\[
\lim_{d \to \infty} \xi_{d, q}(\sigma)^2 B(d, q) = \mu_q(\sigma)^2 / q! > 0. 
\]

\noindent
{\bf Step 4. Complete the proof. } By Eq. (\ref{eqn:hh_concentration_general_sigma_1}), (\ref{eqn:hh_concentration_general_sigma_2}), (\ref{eqn:hh_concentration_general_sigma_3}) and (\ref{eqn:hh_concentration_general_sigma_4}), we have 
\[
\begin{aligned}
&~\sup_{i \in [N]} \Big\vert \E_\btheta[H_{d, \ge \ell}(\btheta_i,\btheta)^2] - \E_{\btheta, \btheta'} [H_{d, \ge \ell}(\btheta', \btheta)^2] \Big\vert \\
\le &~ \sup_{i \in [N]} \Big\vert \E_\btheta[H_{d, [\ell, u)}(\btheta_i,\btheta)^2] - \E_{\btheta, \btheta'} [H_{d, [\ell, u)}(\btheta', \btheta)^2] \Big\vert \\
&~ + \sup_{i \in [N]} \Big\vert \E_\btheta[H_{d, \ge \ell}(\btheta_i,\btheta)^2] \Big\vert + \Big\vert \E_{\btheta, \btheta'} [H_{d, \ge \ell}(\btheta', \btheta)^2] \Big\vert\\
\le &~ o_{d, \P}(1) \cdot \E_{\btheta, \btheta'}[H_{d, [\ell, u)}(\btheta',\btheta)^2] + \Theta_d(d^{-u}) = o_{d, \P}(1) \cdot  \E_{\btheta, \btheta'}[H_{d, \ge \ell}(\btheta',\btheta)^2]. 
\end{aligned}
\]
This completes the proof. 
\end{proof}

\subsection{Auxiliary lemmas}

The following lemma is a reformulation of \cite[Lemma 5]{ghorbani2019linearized}. 
\begin{lemma}\label{lemma:square_integrable}
Assume $\sigma \in C(\R)$ with $\sigma(u)^2\le c_0 \, \exp(c_1\, u^2/2)$ for some constants $c_0 > 0$ and $c_1<1$.
Then 
\begin{enumerate}
\item[$(a)$] $\E_{G \sim \normal(0, 1)}[\sigma(G)^2] < \infty$. 
\item[$(b)$] Let $\|\bx\|_2= \sqrt d$. Then there exists $d_0=d_0(c_1)$ such that, for $\bw \sim \Unif(\S^{d-1})$,
\begin{align}
\sup_{d \ge d_0} \E_{\bw}[\sigma(\<\bw,\bx\>)^2] < \infty\, .
\end{align}
\item[$(c)$] Let $\|\bx \|_2= \sqrt d$, $\bw \sim \Unif(\S^{d-1})$ and $\tau \sim \chi(d) / \sqrt{d}$. Then we have
\begin{align}
\lim_{d \to \infty} \E_{\bw, \tau} \Big[\Big(\sigma(\tau \<\bw,\bx\>) - \sigma(\< \bw, \bx\>) \Big)^2\Big] = 0.
\end{align}
\end{enumerate}
\end{lemma}

\begin{lemma}\label{lem:E_d_to_E}
Assume that $\psi, \phi \in C(\R)$ with $\psi(u)^2, \phi(u)^2 \le c_0\, \exp(c_1\, u^2/2)$ for some constants $c_0 > 0$ and $c_1 < 1$. Denote 
\[
\begin{aligned}
E_d[\psi, \phi](\gamma) \equiv&~ \E_{\bw \sim \Unif(\S^{d-1})}[\psi(\< \bx, \bw\>) \phi(\< \bx', \bw\>)],\\
E[\psi, \phi](\gamma) \equiv&~ \E_{\bg \sim \cN(\bzero, \id_d / d)}[\psi(\< \bx, \bg \>) \phi(\< \bx', \bg\>)], \\
\end{aligned}
\]
where $\| \bx \|_2 = \| \bx' \|_2 = \sqrt d$ such that $\< \bx, \bx' \> / d = \gamma$ (by an invariance argument, $E_d$ and $E$ do not depend on the choice of $\bx$ and $\bx'$). Then we have
\begin{equation}\label{eqn:E_d_to_E}
\lim_{d \to \infty} \sup_{\gamma \in [-1, 1]} \Big\vert E_d[\psi, \phi](\gamma) - E[\psi, \phi](\gamma) \Big\vert = 0, 
\end{equation}
and
\begin{equation}\label{eqn:bound_E_psi_phi}
\sup_{\gamma \in [-1, 1]} \Big\vert E[\psi, \phi](\gamma) \Big\vert < \infty. 
\end{equation}
\end{lemma}

\begin{proof}[Proof of Lemma \ref{lem:E_d_to_E}] Let $\bg \sim \cN(\bzero, \id_d / d)$, $\bw = \bg / \| \bg \|_2$ and $\tau = \| \bg \|_2$. Then we have $\bw \sim \Unif(\S^{d-1})$, $\tau \sim \chi(d) / \sqrt{d}$ independently. We further denote
\[
\overline E_d[\psi, \phi](\gamma) \equiv \E_{\bw, \tau}[\psi(\tau \< \bx, \bw\>) \phi(\< \bx', \bw\>)]
\]
where $\| \bx \|_2 = \| \bx' \|_2 = \sqrt d$ such that $\< \bx, \bx' \> / d = \gamma$. Note we have 
\[
\begin{aligned}
&~ \lim_{d \to \infty} \sup_{\gamma \in [-1, 1]} \Big\vert E_d[\psi, \phi](\gamma) - \overline E_d[\psi, \phi](\gamma) \Big\vert\\
\le&~ \lim_{d \to \infty} \sup_{\gamma \in [-1, 1]} \Big\vert \E_{\bw, \tau}\Big\{ \Big[\psi(\tau \< \bx, \bw\>) - \psi(\< \bx, \bw\>) \Big]  \phi(\< \bx', \bw\>) \Big\} \Big\vert\\
\le&~  \lim_{d \to \infty}  \E_{\bw, \tau}\Big\{ \Big[\psi(\tau \< \bx, \bw\>) - \psi(\< \bx, \bw\>) \Big]^2\Big\}^{1/2} \E_\bw[\phi(\< \bx', \bw\>)^2]^{1/2} = 0,\\
\end{aligned}
\]
where the last equality is by (b) and (c) in Lemma \ref{lemma:square_integrable}. Moreover, we have 
\[
\begin{aligned}
&~ \lim_{d \to \infty} \sup_{\gamma \in [-1, 1]} \Big\vert \overline E_d[\psi, \phi](\gamma) - E[\psi, \phi](\gamma) \Big\vert\\
\le&~ \lim_{d \to \infty} \sup_{\gamma \in [-1, 1]} \Big\vert \E_{\bw, \tau}\Big\{ \Big[\phi(\tau \< \bx, \bw\>) - \phi(\< \bx, \bw\>) \Big] \psi(\tau \< \bx', \bw\>) \Big\} \Big\vert\\
\le&~  \lim_{d \to \infty}  \E_{\bw, \tau}\Big\{ \Big[\phi(\tau \< \bx, \bw\>) - \phi(\< \bx, \bw\>) \Big]^2\Big\}^{1/2} \E_{G \sim \cN(0, 1)}[\psi(G)^2]^{1/2} = 0,\\
\end{aligned}
\]
where the last equality is by (a) and (c) in Lemma \ref{lemma:square_integrable}. Combining the two equations above proves Eq. (\ref{eqn:E_d_to_E}). 

Finally, note that we have 
\[
\begin{aligned}
\sup_{\gamma \in [-1, 1]} \Big\vert E[\psi, \phi](\gamma) \Big\vert \le&~ \E_{\bg}[\psi(\< \bx, \bg\>)^2]^{1/2} \E_\bg[\phi(\< \bx', \bg\>)^2]^{1/2} \\
=&~ \E_{G \sim \cN(0, 1)}[\psi(G)^2]^{1/2} \E_{G \sim \cN(0, 1)}[\phi(G)^2]^{1/2} < \infty. 
\end{aligned}
\] 
This proves Eq. (\ref{eqn:bound_E_psi_phi}). 
\end{proof}

\begin{lemma}\label{lem:h_d_derivative_to_h_derivative}
Assume that $\sigma \in C^\ell(\R)$ with derivatives satisfy $\sup_{0\le k \le \ell} |\sigma^{(k)}(u)|^2 \le c_0\, \exp(c_1\, u^2/2)$ for some constants $c_0 > 0$ and $c_1 < 1$. Denote 
\[
\begin{aligned}
h_d(\gamma) \equiv&~ \E_{\bw \sim \Unif(\S^{d-1})}[\sigma(\< \bx, \bw\>) \sigma(\< \bx', \bw\>)],\\
h(\gamma) \equiv&~ \E_{\bg \sim \cN(\bzero, \id_d / d)}[\sigma(\< \bx, \bg \>) \sigma(\< \bx', \bg\>)], \\
\end{aligned}
\]
where $\| \bx \|_2 = \| \bx' \|_2 = \sqrt d$ such that $\< \bx, \bx' \> / d = \gamma$ (by an invariance argument, $h_d$ and $h$ do not depend on the choice of $\bx$ and $\bx'$). Then we have
\[
\lim_{d \to \infty} \sup_{0 \le k \le \ell} \sup_{\gamma \in [-1, 1]} \Big\vert h_d^{(k)}(\gamma) - h^{(k)}(\gamma) \Big\vert = 0,
\]
and
\[
\sup_{0 \le k \le \ell} \sup_{\gamma \in [-1, 1]} \Big\vert h^{(k)}(\gamma) \Big\vert < \infty. 
\]
\end{lemma}

\begin{proof}[Proof of Lemma \ref{lem:h_d_derivative_to_h_derivative}]

For $k = 0$, the result is implied by Lemma \ref{lem:E_d_to_E} by observing that $h_d' = E_d[\sigma, \sigma]$ and $h' = E[\sigma, \sigma]$. 

For $k = 1$, the result is implied by Lemma \ref{lem:E_d_to_E} by the fact that $h_d' = E_d[u \sigma(u), \sigma'(u)]$ and $h' = E[u \sigma(u), \sigma'(u)]$, and there exist constants $c_0 > 0$  and $c_1< 1$ such that $\sigma'(u), u \sigma(u) \le c_0 e^{c_1 u^2/2}$. Indeed, for $\| \bx \|_2 = \| \bx' \|_2 = \sqrt d$ such that $\< \bx, \bx' \> / d = \gamma$, we have (similarly for $h'$)
\[
\begin{aligned}
h_d'(\gamma) =&~ \lim_{\delta \to 0} \delta^{-1} \Big\{ \E_\bw\Big[\sigma(\< \bx, \bw\>) \sigma(\< (1 - \delta^2)^{1/2}\bx' + \delta \bx, \bw\>) \Big] - \E_\bw\Big[\sigma(\< \bx, \bw\>) \sigma(\< \bx', \bw\>) \Big] \Big\} \\
=&~ \E_\bw\Big[\sigma(\< \bx, \bw\>) \sigma'(\< \bx', \bw\>) \< \bx, \bw\> \Big] = E_d[u \sigma(u), \sigma'(u)](\gamma). 
\end{aligned}
\]

By an induction argument, for any fixed $k$, $h_d^{(k)}$ can be identified by a fixed number of combinations of $E_d[\psi, \phi]$ with $\psi, \phi \in \Lambda_k \equiv \{ u^s \sigma^{(t)}(u) \}_{0 \le s, t \le k}$. Further, for any fixed $k$, there exists $c_{0, k} > 0$ and $c_{1, k} < 1$ such that, for any $\psi \in \Lambda_k$, we have $\psi(u) \le c_{0, k} e^{c_{1, k} u^2/2}$. Applying Lemma \ref{lem:E_d_to_E} proves the lemma. 
\end{proof}

\begin{lemma}\label{lem:bound_h_d_ge_ell_derivative_uniform}
Assume that $\sigma \in C^k(\R)$ with derivatives satisfy $\sup_{0\le s \le k} |\sigma^{(s)}(u) |^2 \le c_0\, \exp(c_1\, u^2/2)$ for some constants $c_0 > 0$ and $c_1 < 1$. For any $\| \btheta_1 \|_2 = \| \btheta_2 \|_2 = \sqrt d$ and any $\ell \ge 1$, denote 
\[
\begin{aligned}
h_{d, S}(\< \btheta_1, \btheta_2\> / d) \equiv&~ \E_{\bx \sim \Unif(\S^{d-1}(\sqrt d))}[\sigma_{d, S}(\< \btheta_1, \bx\>/ \sqrt d) \sigma_{d, S}(\< \btheta_2, \bx\> / \sqrt d)]. 
\end{aligned}
\]
where $\sigma_{d, S}$ is given in Eq. (\ref{eqn:def_sigma_S}). Then we have
\[
\sup_{d \ge 1}\sup_{\gamma \in [-1, 1]} \Big\vert h_{d, \ge \ell}^{(k)}(\gamma) \Big\vert \le C_{k, \ell}. 
\]
\end{lemma}

\begin{proof}[Proof of Lemma \ref{lem:bound_h_d_ge_ell_derivative_uniform}]
Note we have 
\[
h_{d, \ge \ell}(\gamma) = h_{d}(\gamma) - h_{d, < \ell}(\gamma). 
\]
By Lemma \ref{lem:h_d_derivative_to_h_derivative}, we have 
\[
\sup_{d \ge 1} \sup_{\gamma \in [-1, 1]} \Big\vert h_{d}^{(k)}(\gamma)\Big\vert \le  C_k. 
\]
Moreover, since $h_{d, < \ell}(\gamma)$ is a degree $\ell - 1$ polynomial on $[-1, 1]$ and its coefficients converge to the coefficients of $h_{< \ell}$ with $h_{< \ell}(\< \btheta_1, \btheta_2\>/d) = \E_{\bx \sim \cN(\bzero, \id_d)}[\sigma_{d, < \ell}(\< \btheta_1, \bx\>/ \sqrt d) \sigma_{d, < \ell}(\< \btheta_2, \bx\> / \sqrt d)]$. Then, it is easy to see that 
\[
\sup_{d \ge 1} \sup_{\gamma \in [-1, 1]} \Big\vert h_{d, < \ell}^{(k)}(\gamma)\Big\vert \le C_{k, \ell}. 
\]
This proves the lemma. 
\end{proof}

\begin{lemma}\label{lem:h_d_derivative_bound_at_0}
Assume that $\sigma \in C^\ell(\R)$ with derivatives satisfy $\sup_{0\le s \le \ell} |\sigma^{(s)}(u)|^2 \le c_0\, \exp(c_1\, u^2/2)$ for some constants $c_0 > 0$ and $c_1 < 1$. Then there exists constant $C_{k, \ell}$, such that 
\[
\Big\vert h^{(k)}_{d, \ge \ell}(0) \Big\vert \le C_{k, \ell} \cdot d^{-(\ell-k)/2}. 
\]
\end{lemma}

\begin{proof}[Proof of Lemma \ref{lem:h_d_derivative_bound_at_0}]

We let $C$, $C_k$, $C_{k, \ell}$ be constants that depend on $\sigma$, $k$, and $\ell$ but independent of dimension $d$. The exact values of these constant can change from line to line. 

We let $\ttau_d$ be the measure of $\<\be_1,\bx\>$ when $\bx\sim\Unif(\S^{d-1}(\sqrt{d}))$ (hence converging weakly to a standard Gaussian), and $\tQ^{(d)}_k(x) = \sqrt{B(d,k)}Q^{(d)}_k(x / \sqrt d)$ to be the rescaled Gegenbauer polynomials, forming an orthonormal system with respect to $\ttau_d$. In particular $\tQ^{(d)}_k$ converges to the $k$-th Hermite polynomial. We let $\<\,\cdot\,,\,\cdot\,\>$ denote the scalar product with respect to $\ttau_d$. 

By the definition of $h_{d, \ge \ell}$, we have 
\begin{align}\label{eq:Ortho_eq1}
\<h_{d, \ge \ell}( \, \cdot\, / \sqrt d),\tQ^{(d)}_k \> = 0, ~~~~ \forall k\le \ell-1 .
\end{align}
Let $\tilde{h}_{d, \ge \ell}(x)$ be obtained from $h_{d, \ge \ell}(x)$ by removing its Taylor expansion up to term $x^{\ell-1}$, i.e., we have
\[
\tilde{h}_{d, \ge \ell}(x) = h_{d, \ge \ell}(x) - \sum_{k = 0}^{\ell - 1} \frac{h_{d, \ge \ell}^{(k)}(0)}{k!} x^k. 
\]
Then Eq. (\ref{eq:Ortho_eq1}) gives 
\begin{align}
&~\sum_{j = 0}^{\ell-1}\<\tQ^{(d)}_k, x\>^j \Big(\frac{h_{d, \ge \ell}^{(k)}(0)}{j!d^{j/2}} \Big) = - \frac{\Delta_k(d)}{d^{\ell/2}}, ~~~~ \forall k\le \ell-1, \label{eq:Ortho}\\
&~\Delta_k(d) \equiv d^{\ell/2} \<\tilde{h}_{d, \ge \ell}( \,\cdot\, / \sqrt d),\tQ_k\>. \nonumber
\end{align}
We claim that $\sup_{d \ge 1} \vert \Delta_k(d) \vert \le C_{k, \ell}$. Indeed, by Rodrigues formula, there exist non-negative constants $A_{d,k}$, $\tilde A_{d,k}$ with $ \sup_{d \ge 1} A_{d, k} \vee \tilde A_{d, k} \le C_k$, such that
\begin{equation}\label{eqn:Rodrigues_implications}
\begin{aligned}
\Delta_k(d) =&~ (-1)^kd^{\ell/2}A_{d,k} \int_{-1}^1 \tilde{h}_{d, \ge \ell}\big(x/\sqrt{d}\big)\, \frac{\de^k\phantom{x}}{\de x^k}\Big(1-\frac{x^2}{d}\Big)^{\frac{d-3}{2}+k}\de x\\
=&~ A_{d,k}\, d^{(\ell-k)/2}\int_{-1}^1 \tilde{h}^{(k)}_{d, \ge \ell}\big(x/\sqrt{d}\big) \Big(1-\frac{x^2}{d}\Big)^{\frac{d-3}{2}+k}\de x \\
=&~ \tilde{A}_{d, k} \, d^{(\ell-k)/2} \cdot \E_{X_d \sim \ttau_d}\Big\{\tilde{h}^{(k)}_{d, \ge \ell}\big(X_d/\sqrt{d}\big) \Big(1-\frac{X_d^2}{d}\Big)^{k}\Big\}.
\end{aligned}
\end{equation}
By the definition of $\tilde{h}_{d, \ge \ell}$, using the Taylor expansion in the integral form, we have 
\[
\tilde{h}_{d, \ge \ell}(\gamma) = \int_0^\gamma h_{d, \ge \ell}^{(\ell)}(u) \frac{(\gamma - u)^{\ell-1}}{(\ell-1) !} \de u, 
\]
and hence for any $k \le \ell - 1$, we have
\[
\tilde{h}_{d, \ge \ell}^{(k)}(\gamma) = \int_0^\gamma h_{d, \ge \ell}^{(\ell)}(u) \frac{(\gamma - u)^{\ell-1 - k}}{(\ell-1 - k) !} \de u, 
\]
so that for any $\gamma \in [-1, 1]$, we have
\[
\sup_{d \ge 1}\vert \tilde{h}_{d, \ge \ell}^{(k)}(\gamma)\vert \le C_{k, \ell} \cdot \sup_{d \ge 1} \sup_{u \in [-1, 1]} \vert h_{d, \ge \ell}^{(\ell)}(u) \vert \cdot \vert \gamma\vert ^{\ell - k} \le C_{k, \ell} \cdot \vert \gamma \vert^{\ell - k}.  
\]
The last inequality is by Lemma \ref{lem:bound_h_d_ge_ell_derivative_uniform} (here we used the assumption that $\sigma \in C^\ell(\R)$). Therefore, by Eq. (\ref{eqn:Rodrigues_implications}), we have (note $X_d$ converges in distribution to a standard Gaussian random variable)
\begin{align} \label{eqn:Delta_bound_in_kernel_concentration_diagonal}
\vert \Delta_k(d) \vert \le C_{k, \ell} \cdot \E_{X_d \sim \ttau_d}\{\vert X_d\vert^{\ell-k}\} \le C_{k, \ell}.
\end{align}

To conclude, we reconsider Eq.~\eqref{eq:Ortho}. Let $\bM(d) = (M_{k, q}(d))_{0 \le k, q \le \ell - 1} \in \R^{\ell\times \ell}$ be the matrix with entries $M_{k, q}(d) \equiv \<\tQ^{(d)}_{k},x^{q}\>$, $\bxi(d) = (\xi_q(d))_{0 \le q \le \ell - 1} \in \R^{\ell}$ the vector with entries $\xi_{q}(d) \equiv h_{d, \ge \ell}^{(q)}(0)/(q!d^{q/2})$, and $\bDelta(d) = (\Delta_0(d), \ldots, \Delta_{\ell - 1}(d))^\sT \in \R^\ell$. We can therefore rewrite this equation as
\begin{align}
\bM(d)\bxi(d) = \bDelta(d)/ d^{\ell/2}.
\end{align}
As $d\to\infty$, $\bM(d)$ converges entrywise to $\bM(\infty) = (M_{k, q}(\infty))_{0 \le k, q \le \ell - 1}$, whereby 
\[
M_{k,q}(\infty) \equiv \E_{G \sim \cN(0, 1)}[\He_{k}(G) G^{q}]/\sqrt{k!}.
\]
Since $\bM(\infty)$ is non-singular (because the Hermite polynomials are a basis), it follows that $\sigma_{\min}(\bM(d))$ is bounded away from zero for $d$ large enough, and therefore $\sup_{d \ge 1} \sigma_{\max}(\bM(d)^{-1})<\infty$. Therefore combining with Eq. (\ref{eqn:Delta_bound_in_kernel_concentration_diagonal}), we get
\begin{align}
\| \bxi(d)\|_2 \le C_\ell \cdot \|\bDelta(d)\|_2 \cdot d^{-\ell/2} \le C_\ell \cdot d^{-\ell/2}.
\end{align}
Therefore, for any $0 \le k \le \ell - 1$, we have 
\begin{align}
\Big\vert h_{d, \ge \ell}^{(k)}(0)\Big\vert \le k! d^{k/2} \vert \xi_k(d) \vert \le C_{k, \ell} \cdot d^{-(\ell-k)/2}.
\end{align}
This proves the lemma. 
\end{proof}

\section{Hypercontractivity of general activation $\sigma$ for $ (\S^{d-1}(\sqrt d) , \Cyc_d)$}\label{sec:hyper_high}

Let us consider an activation function $\sigma : \R \to \R$ and denote for $\bx\in \S^{d-1}(\sqrt d)$ and $ \bw \in \S^{d-1}(1)$,
\[
\osigma ( \bx ; \bw ) =  \int_{\Cyc_d} \sigma (\< \bx, g \cdot \bw\>) \pi_d(\de g) = \frac{1}{d} \sum_{i=0}^{d-1} \sigma ( \< \bx , \bL^i \bw \> ),
\]
where $\bL \in \R^{d \times d}$ is the cyclic permutation matrix that shifts the coordinates by one (hence $\bL^i$ shifts the coordinates by $i$).

Denote $\osigma_{>\ell} = \oproj_{>\ell} \osigma$ the projection of $\osigma$ orthogonal to cyclic polynomials of degree less or equal to $\ell$. From the discussion in Section \ref{sec:invariant_polynomials}, we have 
\[
\oproj_{>\ell} \osigma (  \, \cdot \, ; \bw ) = \oproj_{>\ell} \cS [ \sigma ( \< \, \cdot \, , \bw\> / \sqrt{d} )] = \cS \tproj_{>\ell} [\sigma ( \<  \, \cdot \, , \bw \> / \sqrt{d} ) ] \, ,
\]
where $\cS : L^2 ( \S^{d-1}(\sqrt d) ) \to L^2 ( \S^{d-1}(\sqrt d) , \Cyc_d ) $ is the symmetrization operator defined in Section \ref{sec:invariant_class} and $\tproj_{>\ell} : L^2 (\S^{d-1}(\sqrt d)) \to L^2 ( \S^{d-1}(\sqrt d) )$ is the projection orthogonal to (general) polynomials of degree less or equal to $\ell$ in $L^2 (\S^{d-1}(\sqrt d))$ (see Section \ref{sec:technical_background}). Hence, denoting $\sigma_{>\ell} = \tproj_{>\ell} \sigma$, we have
\begin{equation}\label{eq:formula_osigma_sigma_cyclic}
\osigma_{>\ell} ( \bx ; \bw ) =  \frac{1}{d} \sum_{i=0}^{d-1} \sigma_{>\ell} ( \< \bx , \bL^i \bw \> ) \, .
\end{equation}

\begin{proposition}\label{prop:assumption2a_cyclic}
Consider fixed integers $m \ge 1$ and $\ell \geq 4m$. Let $\sigma : \R \to \R$ be a differentiable activation function such that $| \sigma(x)| , | \sigma ' (x) |  \leq c_0 \exp (c_1 x^2 / (8m) )$ for some constants $c_0 > 0$ and $c_1 <1 $.  Let $\bx \sim \Unif(\S^{d-1} (\sqrt{d} ))$ and $\bw \sim \Unif ( \S^{d-1} (1))$, then for any $\eps > 0$,
\begin{equation}\label{eq:bound_hyper_lem}
\E_{\bx,\bw} \Big[ \osigma_{>\ell} ( \bx ; \bw ) ^{2m} \Big]^{1/(2m)} = d^{\eps -  1/2} \cdot O_{d} (1).
\end{equation}
\end{proposition}

\subsection{Proof of Proposition \ref{prop:assumption2a_cyclic}}
\label{sec:proof_cyclic_hyp}

The goal of this proof is to replace $\bx \sim \Unif ( \S^{d-1} (\sqrt{d} ))$ by $\bg \sim \normal ( 0 , \id_d)$ and using Proposition \ref{prop:gaussian_cyclic} (stated in Section \ref{sec:gaussian_cyc}), which is the Gaussian equivalent of Proposition \ref{prop:assumption2a_cyclic}.

Recall that $\sigma_{>\ell}$ is defined as the projection of $\sigma$ orthogonal to degree $\ell$ polynomials with respect to the distribution $\< \bx , \be \>$ with $\bx \sim \Unif ( \S^{d-1} (\sqrt{d} ))$ and $ \| \be \|_2 = 1$ arbitrary. We can write it explicitly in terms of Gegenbauer polynomials: 
\[
\sigma_{>\ell} (x) = \sigma(x) - \sum_{k = 0}^\ell \xi_{d,k} B( \S^{d-1} ; k ) Q_k ( \sqrt{d} x).
\]
Let us introduce $\varphi_{>\ell}$ defined as the projection of $\sigma$ orthogonal to degree $\ell$ polynomials with respect to the Gaussian measure. It is given explicitely by
\[
\varphi_{>\ell} (x) = \sigma(x) - \sum_{k = 0}^\ell \frac{\mu_k (\sigma)}{k!} \He_k (x),
\]
where $\He_k$ denote the $k$-th Hermite polynomial (see Section \ref{sec:Hermite} for definitions).

Consider the symmetrized activation functions
\[
\begin{aligned}
\osigma_{>\ell} ( \bx ; \bw ) = &~ \osigma ( \bx ; \bw ) - \sum_{k = 0}^\ell \xi_{d,k} B( \S^{d-1} ; k ) \barQ_k ( \bx ; \bw) \, ,\\
\barphi_{>\ell} ( \bg ; \bw ) = &~ \osigma ( \bg ; \bw ) - \sum_{k = 0}^\ell  \frac{\mu_k (\sigma)}{k!}  \barHe_k ( \bg ; \bw) \, ,
\end{aligned}
\]
where we denoted the symmetrized polynomials
\[
\begin{aligned}
\barQ_k ( \bx ; \bw ) =&~ \frac{1}{d}  \sum_{i = 0 }^{d-1} Q_k (\sqrt{d} \< \bx , \bL^i \bw \>) \, , \\
\barHe_k ( \bg ; \bw ) =&~ \frac{1}{d}  \sum_{i = 0 }^{d-1} \He_k ( \< \bg , \bL^i \bw \>) \, .
\end{aligned}
\]

Consider $\bx \sim \Unif(\S^{d-1} (\sqrt{d} ))$, $\bg \sim \normal( 0 , \id_d)$ and $\bw \sim \Unif ( \S^{d-1} (1) )$. Because $\< \bx , \be \>$ converges in distribution to a normal distribution, we expect the moments of $\osigma_{>\ell} ( \bx ; \bw )$ to converge to the moments of $\barphi_{>\ell} ( \bg ; \bw )$. Let us show that this convergence occurs with rate $O_d (d^{\eps - 1/2})$. By triangle inequality, we have
\[
\E_{\bg , \bw} \Big[ \big(\osigma_{>\ell} ( \sqrt{d} \bg / \| \bg \|_2 ; \bw ) - \barphi_{>\ell} ( \bg ; \bw )  \big)^{2m} \Big]^{1/(2m)} \leq R_1 + R_2 + R_3 + R_4 \, ,
\]
with
\[
\begin{aligned}
R_1 = &~ \E_{\bg , \bw} \Big[ \big(\osigma ( \sqrt{d} \bg / \| \bg \|_2 ; \bw ) - \osigma ( \bg ; \bw )  \big)^{2m} \Big]^{1/(2m)} \, , \\
R_2 = &~ \E_{\bg , \bw} \Big[ \big(A_{\leq 2} ( \sqrt{d} \bg / \| \bg \|_2 ; \bw ) - B_{\leq 2} ( \bg ; \bw )  \big)^{2m} \Big]^{1/(2m)} \, , \\
R_3 = &~ \E_{\bx , \bw} \Big[ A_{[3:\ell]} ( \bx ; \bw )^{2m} \Big]^{1/(2m)} \, , \\
R_4= &~ \E_{\bg , \bw} \Big[ B_{[3:\ell]} ( \bg ; \bw )^{2m} \Big]^{1/(2m)} \, , \\
\end{aligned}
\]
where we denoted $[3:\ell] = \{ 3 , \ldots, \ell \}$ and for any subset $S \subset \{ 0 , \ldots , \ell \}$,
\[
\begin{aligned}
A_S ( \bx ; \bw ) =&~ \sum_{k \in S} \xi_{d,k} B( \S^{d-1} ; k ) \barQ_k ( \bx ; \bw) \, ,\\
B_S ( \bg ; \bw ) = &~  \sum_{k \in S} \frac{\mu_k (\sigma)}{k!}  \barHe_k ( \bg ; \bw) \, .
\end{aligned}
\]

\noindent
{\bf Step 1. Bound on $R_1$. }  

Denote $\tau = \| \bg \|_2 / \sqrt{d}$ and $\bx = \sqrt{d} \bg / \| \bg \|_2$, such that $\tau$ and $\bx$ are independent, and $\bx \sim \Unif ( \S^{d-1} (\sqrt{d} ))$. By the mean value theorem, there exists $\ttau$ on the line segment between $1$ and $\tau$ such that
\[
\osigma ( \tau \cdot \bx ; \bw ) - \osigma ( \bx ; \bw ) = (\tau - 1) \< \nabla_{\bx} \osigma ( \ttau \cdot \bx ; \bw ) , \bx \>.
\]
By Cauchy-Schwarz inequality, we get
\[
\begin{aligned}
R_1 = &~ \E_{\tau , \bx , \bw} \Big[ \big( \osigma ( \tau \cdot \bx ; \bw ) - \osigma ( \bx ; \bw )  \big)^{2m} \Big]^{1/(2m)} \\
=&~ \E_{\tau , \bx , \bw} \Big[  (\tau - 1)^{2m} \< \nabla_{\bx} \osigma ( \ttau \cdot \bx ; \bw ) , \bx \>^{2m} \Big]^{1/(2m)}  \\
\leq &~  \E_{\tau} \big[  (\tau - 1)^{4m} \big]^{1/(4m)} \cdot \E_{\tau , \bx , \bw} \Big[  \< \nabla_{\bx} \osigma ( \ttau \cdot \bx ; \bw ) , \bx \>^{4m} \Big]^{1/(4m)} 
\end{aligned}
\]
Let us bound the first term:
\begin{equation}\label{eq:R_1_sub1}
\begin{aligned}
 \E_{\tau} \big[  (\tau - 1)^{4m} \big]^{1/(4m)} \leq &~ \E_{\tau} \big[  (\tau^2 - 1)^{4m} \big]^{1/(4m)} \leq C_{4m} \E_{\tau} \big[ (\tau^2 - 1 )^2 \big]^{1/2} = C_{4m}  \sqrt{\frac{2}{d}} \, ,
\end{aligned}
\end{equation}
where we used in the first inequality that $| \tau - 1 | \leq | \tau^2 - 1|$ for $\tau \geq 0$; in the second inequality that $\tau^2 - 1$ is a degree 2 polynomial in $\bg \sim \normal ( 0 , \id_d )$ and verifies the hypercontractivity property of Lemma \ref{lem:hyper_multi_gaussian}; last equality, that $d \cdot \tau^2 = \|\bg \|_2^2$ follows a chisquared distribution of degree $d$.

For the second term, we have
\[
\< \nabla_{\bx} \osigma ( \ttau \cdot \bx ; \bw ) , \bx \> = \frac{1}{d} \sum_{i = 0 }^{d-1} \< \bx , \bL^i \bw \> \sigma^{(1)} ( \< \ttau \cdot \bx , \bL^i \bw \>) \, .
\]
Recall that $\ttau \cdot \bx$ is between $\tau \cdot \bx$ and $\bx$ which have marginal distributions $\bg \sim \normal(0, \id_d)$ and $\bx \sim \Unif ( \S^{d-1} (\sqrt{d}))$ respectively. Denote $x_1$ the first coordinate of $\bx$ (therefore $\tau \cdot x_1 \sim \normal (0,1)$). By Jensen's inequality and using that by rotation $\< \bx , \bL^i \bw \>$ has the same distribution as $x_1$, we get
\begin{equation}\label{eq:R_1_sub2}
\begin{aligned}
 \E_{\tau , \bx , \bw} \Big[  \< \nabla_{\bx} \osigma ( \ttau \cdot \bx ; \bw ) , \bx \>^{4m} \Big] \leq  &~  \E_{\tau , x_1 } \Big[  x_1^{4m} \sigma ' ( \ttau \cdot x_1)^{4m} \Big] \\
 \leq &~ C \cdot \E_{G \sim \normal(0,1)} \Big[ \max(G^{4m} , 1) \exp \Big\{ c_1 \max(G^2 , 1) / 2 \Big\} \Big] \\
 =&~ O_d (1) \, ,
\end{aligned}
\end{equation}
where we used that $c_1 < 1$.

Combining Eqs.~\eqref{eq:R_1_sub1} and \eqref{eq:R_1_sub2} yields 
\begin{equation}\label{eq:bound_R1}
R_1 = d^{-1/2} \cdot O_d(1).
\end{equation}

\noindent
{\bf Step 2. Bound on $R_3$. }

We have
\[
\begin{aligned}
R_3  =  \E_{\bx , \bw } \Big[ A_{[3:\ell]} (\bx ; \bw )^{2m} \Big]^{1/(2m)} \leq &~ C_{2m} \E_{\bx} \Big[  \E_{\bw} \Big[ A_{[3:\ell]} (\bx ; \bw )^2 \Big]^{m} \Big]^{1/(2m)}\\
\leq &~ C_m C_{2m} \E_{\bx, \bw} \Big[ A_{[3:\ell]} (\bx ; \bw )^2 \Big]^{1/2} \, ,
\end{aligned}
\]
where in the first inequality we used hypercontractivity of low-degree polynomials on the sphere with respect to $\bw$ (Lemma \ref{lem:hypercontractivity_sphere}), and in the second we used hypercontractivity of low-degree symmetric functions with respect to $\bx$ (Lemma 6 in \cite{mei2021generalization}). By Lemma \ref{lem:isometry_of_orthonomal_polynomials}, we have
\[
\begin{aligned}
 \E_{\bx,\bw} \Big[ A_{[3:\ell]} (\bx ; \bw )^2 \Big] = \sum_{k = 3}^\ell \xi_{d,k}^2  B( \S^{d-1} ; k ) \cdot \frac{D( \S^{d-1} ; k )}{B( \S^{d-1} ; k )} = O_d (d^{-1}).
\end{aligned}
\]
We deduce that 
\begin{equation}\label{eq:bound_R3}
R_3 = d^{-1/2} \cdot O_d(1).
\end{equation}

\noindent
{\bf Step 3. Bound on $R_4$. }

Similarly to $R_3$, we have
\[
\begin{aligned}
R_4  =  \E_{\bg , \bw } \Big[ B_{[3:\ell]} (\bg ; \bw )^{2m} \Big]^{1/(2m)} \leq &~ C_{2m} \E_{\bw} \Big[  \E_{\bg} \Big[ B_{[3:\ell]} (\bg ; \bw )^2 \Big]^{m} \Big]^{1/(2m)}\\
\leq &~ C_m C_{2m} \E_{\bg, \bw} \Big[ B_{[3:\ell]} (\bg ; \bw )^2 \Big]^{1/2} \, ,
\end{aligned}
\]
where in the first inequality we used hypercontractivity of low-degree polynomials with respect to $\bg$ (Lemma \ref{lem:hyper_multi_gaussian}), and in the second we used hypercontractivity of low-degree symmetric functions with respect to $\bw$.

Following the proof of Proposition \ref{prop:gaussian_cyclic}, by setting $m = 1$ and $\barphi_{>2} ( \bg ; \bw ) = B_{[3:\ell]} (\bg ; \bw )$, we have for any $\eps >0$,
\begin{equation}\label{eq:bound_R4}
R_4 = d^{\eps -1/2} \cdot O_d(1).
\end{equation}

\noindent
{\bf Step 4. Conclude. }  

The bound on $R_2$ is more technical and we defer it to Section \ref{sec:tech_cyc}. By Lemma \ref{lem:hyper_degree_2}, we have
\begin{equation}\label{eq:bound_R2}
R_2 = d^{-1/2} \cdot O_d(1).
\end{equation}
Hence combining the bounds \eqref{eq:bound_R1}, \eqref{eq:bound_R3}, \eqref{eq:bound_R4} and \eqref{eq:bound_R2}, we obtain for any $\eps >0$,
\[
\E_{\bx , \bw} \Big[ \osigma_{>\ell} ( \bx ; \bw )^{2m} \Big]^{1/(2m)} \leq \E_{\bg , \bw} \Big[ \barphi_{>\ell} ( \bg ; \bw )^{2m} \Big]^{1/(2m)}  + O_d(d^{\eps -1/2}) \, .
\]
Using Proposition \ref{prop:gaussian_cyclic} concludes the proof.

\subsection{Proof in the Gaussian case}\label{sec:gaussian_cyc}

Recall that we defined
\begin{equation}\label{eq:expansion_barphi}
\barphi_{>\ell} ( \bg ; \bw ) = \frac{1}{d} \sum_{i = 0}^{d-1} \varphi_{>\ell} ( \< \bg , \bL^i \bw \> ),
\end{equation}
where 
\begin{equation}\label{eq:varphi_def}
\varphi_{>\ell} (x ) = \sigma(x) - \sum_{k = 0}^{\ell} \frac{\mu_k ( \sigma) }{k!} \He_k (x).
\end{equation}
Let us now state and prove the Gaussian version of Proposition \ref{prop:assumption2a_cyclic}.

\begin{proposition}\label{prop:gaussian_cyclic}
Consider fixed integers $m \ge 1$ and $\ell \geq 4m$. Let $\sigma : \R \to \R$ be an activation function such that $| \sigma(x)| \leq c_0 \exp (c_1 x^2 / (8m) )$ for some constants $c_0 > 0$ and $c_1 <1 $.  Let $\bg \sim \normal ( 0 , \id_d )$ and $\bw \sim \Unif ( \S^{d-1} (1))$, then
\begin{equation}\label{eq:gaussian_cyclic}
\E_{\bg,\bw} \Big[ \barphi_{>\ell} ( \bg ; \bw ) ^{2m} \Big]^{1/(2m)} = d^{ -  1/2} \cdot O_{d} (1).
\end{equation}
\end{proposition}

\begin{proof}[Proof of Proposition \ref{prop:gaussian_cyclic}]
 Let us expand $\barphi_{>\ell}$ as in Eq.~\eqref{eq:expansion_barphi}
 \[
\begin{aligned}
 \E_{\bg,\bw} \Big[ d^{2m} \cdot \barphi_{>\ell} ( \bg ; \bw )^{2m} \Big] = &~ \sum_{0 \leq i_1 , \ldots , i_{2m} \leq d-1} \E_{\bg,\bw} \Big[ \prod_{k \in [2m]} \varphi_{>\ell} ( \< \bg , \bL^{i_k} \bw \> ) \Big].
 \end{aligned}
 \]
Let us consider the event
\[
\cA_\eps : = \Big\{ \bw \in \S^{d-1} ( 1 ) ;  \sup_{k \in [d-1]} | \< \bw , \bL^k \bw \> | \leq C d^{\eps - 1/2} \Big\} \, ,
\]
and for each set of indices $\cI = \{ i_1 , \ldots , i_{2m} \} $, consider separately the expectation over $\cA_\eps$ and $\cA_\eps^c$:
\[
 \E_{\bg,\bw} \Big[ \prod_{i \in \cI} \varphi_{>\ell} ( \< \bg , \bL^{i} \bw \> ) \Big] = A + B,
\]
where
 \[
 \begin{aligned}
A := &~ \E_{\bw} \Big[ \ones_{\cA_\eps} \E_{\bg} \Big[ \prod_{i \in \cI} \varphi_{>\ell} ( \< \bg , \bL^{i} \bw \> ) \Big] \Big], \\
B := &~ \E_{\bw} \Big[ \ones_{\cA_\eps^c}  \E_\bg \Big[ \prod_{i \in \cI} \varphi_{>\ell} ( \< \bg , \bL^{i} \bw \> ) \Big] \Big] .
 \end{aligned}
 \]
 By Cauchy-Schwarz and Jensen's inequality, we have
 \[
 B  \leq \P (\cA_\eps^c)^{1/2} \cdot \E_{\bg,\bw} \Big[ \prod_{i \in \cI} \varphi_{>\ell} ( \< \bg , \bL^{i} \bw \> )^2 \Big]^{1/2} \, ,
 \]
 with
 \[
 \begin{aligned}
\E_{\bg,\bw} \Big[ \prod_{i \in \cI} \varphi_{>\ell} ( \< \bg , \bL^{i} \bw \> )^2 \Big]^{1/2} 
\leq &~    \prod_{i \in \cI} \E_{\bg, \bw} \Big[ \varphi_{>\ell} ( \< \bg , \bL^{i} \bw \> )^{4m} \Big]^{1/(4m)} \\
  =&~  \E_{G} \Big[ \varphi_{>\ell} (G )^{4 m} \Big]^{1/2}   =  O_d (1) \, ,
  \end{aligned}
 \]
 where we used H\"older's inequality and that $\varphi_{>\ell}$ is the sum of a degree $\ell $ polynomial and $\sigma$ with $| \sigma (x) | \leq c_0 \exp ( c_1 x^2 / (8m))$, with constants $c_0>0$ and $c_1 <1$. Combining these bounds and Lemma \ref{lem:proba_Aeps}, we deduce there exists a constant $C$ independent of $d$ and $\cI$ such that
 \begin{equation}\label{eq:bound_B}
 B \leq C  \exp ( - c d^{2\eps} ).
 \end{equation}
 Similarly, by H\"older's inequality, we have the following first bound on $A$:
 \begin{equation}\label{eq:bound_A_1}
 A \leq \prod_{i \in \cI} \E_{\bg, \bw} \Big[ \varphi_{>\ell} ( \< \bg , \bL^{i} \bw \> )^{2m} \Big]^{1/(2m)} =   \E_{G} \Big[ \varphi_{>\ell} (G )^{2 m} \Big]  \leq C \, .
 \end{equation}
 
 Fix $\bw \in \cA_\eps$. Denote $\cI_0$ the set of distinct indices in $\cI$ and $p = | \cI_0 | \leq 2m$. Denote for each $i \in \cI_0$, $r_i$ the multiciplity of $i$ in $\cI$, and $g_i = \< \bg , \bL^i \bw \>$. We have $\sup_{i \neq j} | \E [ g_i g_j ] | \leq \sup_{k \in [d-1]} | \< \bw , \bL^k \bw \> | \leq C d^{\eps - 1/2}$ and $\E [ g_i^2 ] = 1$. Hence, if there exists $i \in \cI$ that appears only once, we have by taking $\psi (x) = \varphi_{>\ell} (\bx)$ and $q = 2m \leq \ell/2$ in Lemma \ref{lem:bound_isolated_index} stated below
 \[
 \Big\vert  \E_{\bg} \Big[ \prod_{i \in \cI} \varphi_{>\ell} ( \< \bg , \bL^{i} \bw \> ) \Big] \Big\vert \leq C' d^{(2m+1)(\eps - 1/2)} \, ,
 \]
 where $C'$ is independent of $\bw$. We deduce that
  \begin{equation}\label{eq:bound_A_2}
 A \leq C' d^{(2m+1)(\eps - 1/2)} \, .
 \end{equation}
 There are at most $m^{2m} d^{m}$ sets of indices $\cI$ with no isolated index. Hence, combining the bounds \eqref{eq:bound_B}, \eqref{eq:bound_A_1} and \eqref{eq:bound_A_2}, we get
  \[
\begin{aligned}
 \E_{\bg,\bw} [ d^{2m} \cdot \barphi_{>\ell} ( \bg ; \bw )^{2m} ] \leq & C d^{2m} \exp ( - c d^{2\eps} ) + Cm^{2m} d^{m} + C' d^{2m}  \cdot d^{(2m+1)(\eps - 1/2)}\, .
 \end{aligned}
 \]
 Taking $\eps \leq 1/(4m+2)$, we get
 \[
 \E_{\bg,\bw} [  \barphi_{>\ell} ( \bg ; \bw )^{2m} ]^{1/(2m)} = d^{-1/2} \cdot O_d (1) \, ,
 \]
 which concludes the proof.
\end{proof}

The proof of Proposition \ref{prop:gaussian_cyclic} relies on the following key lemma:

\begin{lemma}\label{lem:bound_isolated_index}
Let $q, p, m \ge 1$ be three integers such that $p \leq 2m$. Let $\psi : \R \to \R$ be a function such that $|\psi(x)| \leq c_0 \exp ( c_1 x^2 / (4m) )$ for some constants $c_0 > 0$ and $c_1 <1$. Furthermore, for all $k = 0 , \ldots , 2q$,
\[
\mu_k ( \psi ) = \E_G [ \psi (G) \He_k ( G) ] = 0,
\]
where $G \sim \normal (0,1)$, i.e., $\psi$ is orthogonal to all polynomials of degree less or equal to $q$ with respect to the standard normal distribution. Let $\bg = (g_1 , \ldots , g_p ) \sim \normal ( 0 , \bSigma)$ with $\Sigma_{11} = \ldots = \Sigma_{pp} = 1$ and $\sup_{i \neq j} | \Sigma_{ij} | \leq C d^{\eps - 1/2} $. Let $(r_1 , \ldots , r_p)$ be $p$ integers such that $r_1 + \ldots + r_p = 2m$ and there exists $k$ such that $r_k = 1$. Then there exists $C' >0$ depending only on $c_0, c_1 , C, q,m$ such that
\begin{equation}
\Big\vert \E_{\bg} \Big[ \prod_{k \in [ p]} \psi ( g_k )^{r_k} \Big] \Big\vert \leq C'  d^{(q+1 )(\eps - 1/2)}.
\end{equation}
\end{lemma}

\begin{proof}[Proof of Lemma \ref{lem:bound_isolated_index}]
Without loss of generality, let us assume that $r_1  = 1$. Let us rewrite the expectation with respect to $\tbg \sim \normal ( 0 , \id_p)$:
\begin{equation}
\E_{\bg} \Big[ \prod_{k \in [ p]} \psi ( g_k )^{r_k} \Big] =\frac{1}{\sqrt{\det(\bSigma)} } \E_{\tbg} \Big[ \prod_{k \in [ p]} \psi ( \tg_k )^{r_k} \cdot \exp \Big\{  \tbg^\sT \bM \tbg /2 \Big\} \Big]\, ,
\end{equation}
where we denoted $\bM = \id_p - \bSigma^{-1}$.

By Taylor expansion around $0$ at order $q +1$, there exists $\zeta ( \tbg)$ between $0$ and $ \tbg^\sT \bM \tbg /2$ such that
\[
\exp \Big\{  \tbg^\sT \bM \tbg /2 \Big\}  = \sum_{s = 0}^{q}  \frac{1}{2^s s!}  (\tbg^\sT \bM \tbg)^s  + \frac{1}{2^{q+1} (q +1 )!} \exp \{ \zeta ( \tbg) \} \cdot (\tbg^\sT \bM \tbg)^{q+1}.
\]
Notice that the terms $s = 0, \ldots ,q$ are polynomials of degree smaller or equal to $2q$ in $\tbg$. By the assumption of orthonormality of $\psi$ to polynomials of degree less or equal to $2q$, we deduce
\[
\begin{aligned}
&~ \Big\vert \E_{\bg} \Big[ \prod_{k \in [ p]} \psi ( g_k )^{r_k} \Big] \Big\vert \\
=&~ \frac{1}{2^{q+1} (q +1 )! \sqrt{\det(\bSigma)}} \Big\vert \E_{\tbg} \Big[ \prod_{k \in [ p]} \psi ( \tg_k )^{r_k} \cdot \exp \{ \zeta ( \tbg) \} \cdot (\tbg^\sT \bM \tbg)^{q+1} \Big] \Big\vert  \\
\leq &~ \frac{\| \bM \|_{\op}^{q+1} }{2^{q+1} (q +1 )! \sqrt{\det(\bSigma)}} \E_{\tbg} \Big[ \prod_{k \in [ p]} |\psi ( \tg_k )|^{r_k} \cdot \exp \{\| \bM \|_{\op} \| \tbg \|_2^2 \} \cdot \| \tbg \|_2^{2(q+1)} \Big]\, .
\end{aligned} 
\]
Furthermore, from the bound $| \psi (x)| \leq c_0 \exp (c_1 x^2 / (4m) )$ and that $r_k \leq 2m$, we have
\begin{equation}\label{eq:MopInteg}
\begin{aligned}
\Big\vert \E_{\bg} \Big[ \prod_{k \in [ p]} \psi ( g_k )^{r_k} \Big] \Big\vert \leq &~\frac{c_0^{2m} p^{2q} \| \bM \|_{\op}^{q+1} }{2^{q+1} (q +1 )! \sqrt{\det(\bSigma)}} \E_{G} \Big[ G^{2q +2 } \exp \{ c_1 G^2 /2 + \| \bM \|_{\op} G^2  \}  \Big]^p \, .
\end{aligned} 
\end{equation}
 From the assumptions on $\bSigma$, we have $\| \bSigma - \id_p \|_{\op} \leq \| \bSigma - \id_p \|_F \leq p \sup_{i \neq j} | \Sigma_{ij} | = O_d ( d^{\eps-1/2} )$, and therefore $\| \bM \|_{\op} = O_d ( d^{\eps - 1/2} )$ and $\det(\bSigma)^{-1/2} = O_d (1)$. 
 
 From the assumption that $c_1 < 1$ and taking $d$ sufficiently large such that $\| \bM \|_\op < (1 - c_1)/4$, the expectation on the right hand side of Eq.~\eqref{eq:MopInteg} is bounded by a constant. We deduce that
 \[
 \E_{\bg} \Big[ \prod_{k \in [ p]} \psi ( g_k )^{r_k} \Big] = \| \bM \|_{\op}^{q+1} \cdot O_d(1) =  d^{(q+1 )(\eps - 1/2)} \cdot O_d (1 ) \, ,
 \]
 which concludes the proof.
\end{proof}

\subsection{Technical lemmas}\label{sec:tech_cyc}

The first lemma is a straightforward consequence of the proof of Lemma \ref{lem:hypercontractivity_Gaussian} (we include a proof for completeness).

\begin{lemma}\label{lem:hyper_multi_gaussian}
For any $\ell \in \N$ and $f \in L^2(\R^d, \gamma_d)$ to be a degree $\ell$ polynomial on $\R^d$, where $\gamma_d =\normal ( 0 , \id_d)$ is the isotropic Gaussian distribution. Then for any $q \ge 2$, we have 
\[
\| f \|_{L^q(\R^d, \gamma_d)}^2 \le (q - 1)^{\ell} \cdot \| f \|_{L^2(\R^d, \gamma_d)}^2. 
\]
\end{lemma}

\begin{proof}[Proof of Lemma \ref{lem:hyper_multi_gaussian}]
Let $\beps  = ( \eps_{i,j})_{i \in [d], j \in [D]} \sim \Unif ( \Cube^{dD})$ and define for $i = 1, \ldots , d$,
\[
G_i = \frac{\eps_{i,1} + \ldots + \eps_{i,D} }{\sqrt{D}}.
\]
Consider $f$ a degree $\ell$ polynomial on $\R^d$ and define 
\[
\Tilde f ( \beps ) = f(G_1 , \ldots , G_d ).
\]
From hypercontractivity of low degree polynomials on the hypercube (Lemma \ref{lem:hypercontractivity_hypercube}), we have
\begin{equation}\label{eq:gauss_to_hypercube}
\| \Tilde f \|_{L^q (\Cube^{d^2})}^2 \le (q - 1)^{\ell} \cdot \| \Tilde f \|_{L^2(\Cube^{d^2})}^2. 
\end{equation}
Furthermore, by the multivariate central limit theorem, as $D \to \infty$ for $d$ fixed, $(G_1 , \ldots , G_d )$ converges in distribution to $\bg \sim \normal(0 , \id_d)$. By dominated convergence theorem, we have $\| \Tilde f \|_{L^q (\Cube^{dD})}^2 \to \| f \|_{L^q(\R^d, \gamma_d)}^2$ and $\| \Tilde f \|_{L^2(\Cube^{dD})}^2 \to \| f \|_{L^2(\R^d, \gamma_d)}^2$, and taking the limit in inequality \eqref{eq:gauss_to_hypercube} yields the result.
\end{proof}

\begin{lemma}\label{lem:hyper_degree_2}
Follow the notations in Section \ref{sec:proof_cyclic_hyp}. We have
\[
 \E_{\bg , \bw} \Big[ \big(A_{\leq 2} ( \sqrt{d} \bg / \| \bg \|_2 ; \bw ) - B_{\leq 2} ( \bg ; \bw )  \big)^{2m} \Big]^{1/(2m)} = O_d ( d^{-1/2} ).
\]
\end{lemma}

\begin{proof}[Proof of Lemma \ref{lem:hyper_degree_2}]
Denote $\tau = \| \bg \|_2 / \sqrt{d}$ and $\bx = \sqrt{d} \bg / \| \bg \|_2$. Recall that we defined
\[
\begin{aligned}
A_{\leq 2} ( \bx , \bw ) = &~  \xi_{d,0} + \xi_{d,1} B (\S^{d-1} ; 1) \barQ_1 ( \bx ; \bw ) +  \xi_{d,2} B (\S^{d-1} ; 2) \barQ_2 ( \bx ; \bw ) \, , \\
B_{\leq 2} (\tau \cdot \bx , \bw ) = &~ \mu_0 (\sigma) + \mu_1 (\sigma)  \barHe_1 ( \tau \cdot \bx ; \bw ) +  \frac{ \mu_2 (\sigma) }{2} \barHe_2 ( \tau \cdot \bx ; \bw ) \, , \\
\end{aligned}
\]
Let us bound the difference of each term separately.

\noindent
{\bf Step 1. Bound $0$th order term. }

Following the same argument as in the bound of $R_1$ in Section \ref{sec:proof_cyclic_hyp}, we have
\begin{equation}\label{eq:0th_term}
\begin{aligned}
c_0 := | \mu_0 (\sigma) -  \xi_{d,0}  | =&~ \Big\vert \E_{\tau , x_1 } \big[ \sigma ( \tau \cdot x_1 ) - \sigma ( x_1 ) \big] \Big\vert \\
 \leq&~ \E_{\tau} \big[ (\tau - 1)^2 \big]^{1/2} \E_{\tau , x_1} \Big[ x_1^2 \sigma ' (\ttau  \cdot x_1)^2 \Big]^{1/2} = O_d (d^{-1/2}).
\end{aligned}
\end{equation}

\noindent
{\bf Step 2. Bound $1$st order term. }

We have $\He_1 ( x ) = x$ and $B (\S^{d-1} ; 1)^{1/2} \cdot Q_1 (\sqrt{d} x ) = x$. Hence,
\[
\begin{aligned}
c_1: = &~\E_{\bg , \bw} \Big[ \Big(\mu_1 (\sigma)  \barHe_1 ( \tau \cdot \bx ; \bw ) -   \xi_{d,1} B (\S^{d-1} ; 1) \barQ_1 ( \bx ; \bw ) \Big)^{2m} \Big] \\
 =&~ \E_{\tau} \Big[ \Big( \tau \cdot \mu_1 (\sigma)  -  \xi_{d,1} B (\S^{d-1} ; 1)^{1/2}  \Big)^{2m} \Big] \cdot \E_{\bx , \bw} \Big[ B (\S^{d-1} ; 1)^m \barQ_1 ( \bx ; \bw )^{2m} \Big] \, .
\end{aligned}
\]
Using the convergence of Gegenbauer coefficients to Hermite coefficients (see Eq.~\eqref{eq:Gegen-to-Hermite} in Section \ref{sec:Hermite}), there exists a constant $C>0$ such that
\begin{equation}\label{eq:c1_1}
\E_{\tau} \Big[ \Big( \tau \cdot \mu_1 (\sigma) -  \xi_{d,1} B (\S^{d-1} ; 1)^{1/2}  \Big)^{2m} \Big]^{1/(2m)} \leq C \Big[ \E_{\tau} \big[ \tau^{2m} \big]^{1/(2m)} + 1 \Big] = O_d (1) \, ,
\end{equation}
where we used for example that low-degree polynomials of $\tau^2$ are hypercontractive (see the bound on $R_1$ in Section \ref{sec:proof_cyclic_hyp}). From the same argument as in the bound of $R_3$ in Section \ref{sec:proof_cyclic_hyp}, we have
\begin{equation}\label{eq:c1_2}
\E_{\bx , \bw} \Big[ B (\S^{d-1} ; 1)^m \barQ_1 ( \bx ; \bw )^{2m} \Big]^{1/(2m)} \leq C_{2m} \frac{D (\S^{d-1} ; 1)^{1/2}}{B (\S^{d-1} ; 1)^{1/2}} = O_d (d^{-1/2} ).
\end{equation}
Combining Eqs.~\eqref{eq:c1_1} and \eqref{eq:c1_2} yields
\begin{equation}\label{eq:1th_term}
c_1 = O_d (d^{-1/2}).
\end{equation}

\noindent
{\bf Step 3. Bound $2$nd order term. }

We have $\He_2 ( x ) = x^2 - 1$ and $B (\S^{d-1} ; 2)^{1/2} \cdot Q_2 (\sqrt{d} x ) = a_{2,d} \cdot (x^2 - 1)$ with $a_{2,d} = \Theta_d(1)$. We can rewrite
\[
\He_2 ( \tau \cdot x_1 ) = \tau^2 B (\S^{d-1} ; 2)^{1/2} \cdot Q_2 (\sqrt{d} x ) / a_{2,d} + \tau^2 - 1.
\]
Hence, by triangle inequality,
\[
\begin{aligned}
c_1: = &~\E_{\bg , \bw} \Big[ \Big(\mu_2 (\sigma)  \barHe_2 ( \tau \cdot \bx ; \bw )/2 -   \xi_{d,2} B (\S^{d-1} ; 2) \barQ_2 ( \bx ; \bw ) \Big)^{2m} \Big]^{1/(2m)} \\
 \leq&~ \E_{\tau} \Big[ \Big( \tau \cdot \mu_2 (\sigma) / (2a_{2,d})  -  \xi_{d,2} B (\S^{d-1} ; 2)^{1/2}  \Big)^{2m} \Big]^{1/(2m)} \cdot \E_{\bx , \bw} \Big[ B (\S^{d-1} ; 2)^m \barQ_2 ( \bx ; \bw )^{2m} \Big]^{1/(2m)} \, \\
 &~+ \frac{|\mu_2(\sigma)|}{2}\E_{\tau} \big[ (\tau^2 - 1 )^{2m} \big]^{1/(2m)} \, .
\end{aligned}
\]
The first term is bounded as the $1$-st order term while the second term is bounded as the $0$-th order term. Combining the two yields
\begin{equation}\label{eq:2th_term}
c_2 = O_d (d^{-1/2}).
\end{equation}

\noindent
{\bf Step 4. Conclude. }

Combining the bounds \eqref{eq:0th_term}, \eqref{eq:1th_term} and \eqref{eq:2th_term}, we get by triangle inequality  
\[
\E_{\tau, \bx , \bw} \Big[ \big(A_{\leq 2} ( \bx ; \bw ) - B_{\leq 2} ( \tau \cdot \bx ; \bw )  \big)^{2m} \Big]^{1/(2m)}  \leq c_0 + c_1 + c_2 = O_d (d^{-1/2} ) \, ,
\]
which concludes the proof.
\end{proof}

 \begin{lemma}\label{lem:proba_Aeps}
 Let $\eps > 0 $ and $\bw \sim \Unif ( \S^{d-1} ( 1 ))$. Then there exists $C,c >0$ such that
 \[
 \P ( \cA_\eps^{c} ) \leq C \exp (- c d^{2\eps} ).
 \]
 \end{lemma}
 
 \begin{proof}[Proof of Lemma \ref{lem:proba_Aeps}]
 Let us use the correspondence between uniform distribution and Gaussian distribution: $\bw \sim \bz / \| \bz \|_2$, where $\bz \sim \normal ( 0 ,\id_d)$. We have for $k = 1, \ldots, d-1$,
 \[
 \P ( | \< \bw , \bL^k \bw \> | \geq t ) = \P ( | \< \bz , \bL^k \bz \> / \| \bz \|_2^2 | \geq t ) \leq  \P ( | \< \bz , \bL^k \bz \>/d | \geq t/2 ) + \P ( \| \bz \|_2^2 \leq d/2).
 \]
 Note that for any $k \in [d-1]$, we have $\| \bL^k \|_F \leq \sqrt{d}$ and $\| \bL^k \|_\op \leq 1$. By the Hanson-Wright inequality, for any $k \neq 0$, we have
 \[
 \P \Big( \big\vert \< \bz , \bL^k \bz \> /d \big\vert > t \Big) \leq 2 \exp \{ -c d \cdot \min (t^2 , t ) \}.
 \]
 Furthermore, by standard concentration of the norm of Gaussian vectors, we have
 \[
 \P ( \| \bz \|_2^2 \leq d/2) \leq C \exp ( - c d).
 \]
 Taking $t = C d^{\eps - 1/2}$ and combining the above two bounds, we get
 \[
 \P ( | \< \bw , \bL^k \bw \> | \geq t )  \leq 2 \exp ( -c d^{2\eps} ) + C \exp ( - cd ).
 \]
 Taking the union bounds over $k \in [d-1]$ concludes the proof.
 \end{proof}

\section{Technical background of function spaces}\label{sec:technical_background}

\subsection{Functions on the sphere}\label{sec:functions_sphere}

\subsubsection{Functional spaces over the sphere}

For $d \ge 3$, we let $\S^{d-1}(r) = \{\bx \in \R^{d}: \| \bx \|_2 = r\}$ denote the sphere with radius $r$ in $\reals^d$.
We will mostly work with the sphere of radius $\sqrt d$, $\S^{d-1}(\sqrt{d})$ and will denote by $\tau_{d}$  the uniform probability measure on $\S^{d-1}(\sqrt d)$. 
All functions in this section are assumed to be elements of $ L^2(\S^{d-1}(\sqrt d) ,\tau_d)$, with scalar product and norm denoted as $\<\,\cdot\,,\,\cdot\,\>_{L^2}$
and $\|\,\cdot\,\|_{L^2}$:
\begin{align}
\<f,g\>_{L^2} \equiv \int_{\S^{d-1}(\sqrt d)} f(\bx) \, g(\bx)\, \tau_d(\de \bx)\,.
\end{align}

For $\ell\in\integers_{\ge 0}$, let $\tilde{V}_{d,\ell}$ be the space of homogeneous harmonic polynomials of degree $\ell$ on $\reals^d$ (i.e. homogeneous
polynomials $q(\bx)$ satisfying $\Delta q(\bx) = 0$), and denote by $V_{d,\ell}$ the linear space of functions obtained by restricting the polynomials in $\tilde{V}_{d,\ell}$
to $\S^{d-1}(\sqrt d)$. With these definitions, we have the following orthogonal decomposition
\begin{align}
L^2(\S^{d-1}(\sqrt d) ,\tau_d) = \bigoplus_{\ell=0}^{\infty} V_{d,\ell}\, . \label{eq:SpinDecomposition}
\end{align}
The dimension of each subspace is given by
\begin{align}
\dim(V_{d,\ell}) = B(\S^{d-1}; \ell) = \frac{2 \ell + d - 2}{d - 2} { \ell + d - 3 \choose \ell} \, .
\end{align}
For each $\ell\in \integers_{\ge 0}$, the spherical harmonics $\{ Y_{\ell, j}^{(d)}\}_{1\le j \le B(\S^{d-1}; \ell)}$ form an orthonormal basis of $V_{d,\ell}$:
\[
\<Y^{(d)}_{ki}, Y^{(d)}_{sj}\>_{L^2} = \delta_{ij} \delta_{ks}.
\]
Note that our convention is different from the more standard one, that defines the spherical harmonics as functions on $\S^{d-1}(1)$.
It is immediate to pass from one convention to the other by a simple scaling. We will drop the superscript $d$ and write $Y_{\ell, j} = Y_{\ell, j}^{(d)}$ whenever clear from the context.

We denote by $\oproj_k$  the orthogonal projections to $V_{d,k}$ in $L^2(\S^{d-1}(\sqrt d),\tau_d)$. This can be written in terms of spherical harmonics as
\begin{align}
\oproj_k f(\bx) \equiv& \sum_{l=1}^{B(\S^{d-1}; k)} \< f, Y_{kl}\>_{L^2} Y_{kl}(\bx). 
\end{align}
We also define
$\oproj_{\le \ell}\equiv \sum_{k =0}^\ell \oproj_k$, $\oproj_{>\ell} \equiv \id -\oproj_{\le \ell} = \sum_{k =\ell+1}^\infty \oproj_k$,
and $\oproj_{<\ell}\equiv \oproj_{\le \ell-1}$, $\oproj_{\ge \ell}\equiv \oproj_{>\ell-1}$.

\subsubsection{Gegenbauer polynomials}
\label{sec:Gegenbauer}

The $\ell$-th Gegenbauer polynomial $Q_\ell^{(d)}$ is a polynomial of degree $\ell$. Consistently
with our convention for spherical harmonics, we view $Q_\ell^{(d)}$ as a function $Q_{\ell}^{(d)}: [-d,d]\to \reals$. The set $\{ Q_\ell^{(d)}\}_{\ell\ge 0}$
forms an orthogonal basis on $L^2([-d,d],\tilde\tau^1_{d})$, where $\tilde\tau^1_{d}$ is the distribution of $\sqrt{d}\<\bx,\be_1\>$ when $\bx\sim \tau_d$,
satisfying the normalization condition:
\begin{align}
\< Q^{(d)}_k(\sqrt{d}\< \be_1, \cdot\>), Q^{(d)}_j(\sqrt{d}\< \be_1, \cdot\>) \>_{L^2(\S^{d-1}(\sqrt d))} = \frac{1}{B(\S^{d-1};k)}\, \delta_{jk} \, .  \label{eq:GegenbauerNormalization}
\end{align}
In particular, these polynomials are normalized so that  $Q_\ell^{(d)}(d) = 1$. 
As above, we will omit the superscript $(d)$ in $Q_\ell^{(d)}$ when clear from the context.

Gegenbauer polynomials are directly related to spherical harmonics as follows. Fix $\bv\in\S^{d-1}(\sqrt{d})$ and 
consider the subspace of  $V_{\ell}$ formed by all functions that are invariant under rotations in $\reals^d$ that keep $\bv$ unchanged.
It is not hard to see that this subspace has dimension one, and coincides with the span of the function $Q_{\ell}^{(d)}(\<\bv,\,\cdot\,\>)$.

We will use the following properties of Gegenbauer polynomials
\begin{enumerate}
\item For $\bx, \by \in \S^{d-1}(\sqrt d)$
\begin{align}
\< Q_j^{(d)}(\< \bx, \cdot\>), Q_k^{(d)}(\< \by, \cdot\>) \>_{L^2} = \frac{1}{B(\S^{d-1}; k)}\delta_{jk}  Q_k^{(d)}(\< \bx, \by\>).  \label{eq:ProductGegenbauer}
\end{align}
\item For $\bx, \by \in \S^{d-1}(\sqrt d)$
\begin{align}
Q_k^{(d)}(\< \bx, \by\> ) = \frac{1}{B(\S^{d-1}; k)} \sum_{i =1}^{ B(\S^{d-1}; k)} Y_{ki}^{(d)}(\bx) Y_{ki}^{(d)}(\by). \label{eq:GegenbauerHarmonics}
\end{align}
\end{enumerate}
These properties imply that ---up to a constant--- $Q_k^{(d)}(\< \bx, \by\> )$ is a representation of the projector onto 
the subspace of degree -$k$ spherical harmonics
\begin{align}
(\oproj_k f)(\bx) = B(\S^{d-1}; k) \int_{\S^{d-1}(\sqrt{d})} \, Q_k^{(d)}(\< \bx, \by\> )\,  f(\by)\, \tau_d(\de\by)\, .\label{eq:ProjectorGegenbauer}
\end{align}
For a function $\sigma \in L^2([-\sqrt d, \sqrt d], \tau^1_{d})$ (where $\tau^1_{d}$ is the distribution of $\< \be_1, \bx \> $ when $\bx \sim \Unif(\S^{d-1}(\sqrt d))$), denoting its spherical harmonics coefficients $\xi_{d, k}(\sigma)$ to be 
\begin{align}\label{eqn:technical_lambda_sigma}
\xi_{d, k}(\sigma) = \int_{[-\sqrt d , \sqrt d]} \sigma(x) Q_k^{(d)}(\sqrt d x) \tau^1_{d}(\de x),
\end{align}
then we have the following equation holds in $L^2([-\sqrt d, \sqrt d],\tau^1_{d})$ sense
\[
\sigma(x) = \sum_{k = 0}^\infty \xi_{d, k}(\sigma) B(\S^{d-1}; k) Q_k^{(d)}(\sqrt d x). 
\]

For any rotationally invariant kernel $H_d(\bx_1, \bx_2) = h_d(\< \bx_1, \bx_2\> / d)$,
with $h_d(\sqrt{d}\, \cdot \, ) \in L^2([-\sqrt{d},\sqrt{d}],\tau^1_{d})$,
we can associate a self adjoint operator $\cuH_d:L^2(\S^{d-1}(\sqrt{d}))\to L^2(\S^{d-1}(\sqrt{d}))$
via
\begin{align}
\cuH_df(\bx) \equiv \int_{\S^{d-1}(\sqrt{d})} h_d(\<\bx,\bx_1\>/d)\, f(\bx_1) \, \tau_d(\de \bx_1)\, .
\end{align}
By rotational invariance,   the space $V_{k}$ of homogeneous polynomials of degree $k$ is an eigenspace of
$\cuH_d$, and we will denote the corresponding eigenvalue by $\xi_{d,k}(h_d)$. In other words
$\cuH_df(\bx) \equiv \sum_{k=0}^{\infty} \xi_{d,k}(h_d) \oproj_{k}f$.   The eigenvalues can be computed via
\begin{align}
  \xi_{d, k}(h_d) = \int_{[-\sqrt d , \sqrt d]} h_d\big(x/\sqrt{d}\big) Q_k^{(d)}(\sqrt d x) \tau^1_d (\de x)\, .
\end{align}

\subsubsection{Hermite polynomials}
\label{sec:Hermite}

The Hermite polynomials $\{\bbHe_k\}_{k\ge 0}$ form an orthogonal basis of $L^2(\reals,\gamma)$, where $\gamma(\de x) = e^{-x^2/2}\de x/\sqrt{2\pi}$ 
is the standard Gaussian measure, and $\bbHe_k$ has degree $k$. We will follow the classical normalization (here and below, expectation is with respect to
$G\sim\normal(0,1)$):
\begin{align}
\E\big\{\bbHe_j(G) \,\bbHe_k(G)\big\} = k!\, \delta_{jk}\, .
\end{align}
As a consequence, for any function $g\in L^2(\reals,\gamma)$, we have the decomposition
\begin{align}\label{eqn:sigma_He_decomposition}
g(x) = \sum_{k=0}^{\infty}\frac{\mu_k(g)}{k!}\, \bbHe_k(x)\, ,\;\;\;\;\;\; \mu_k(g) \equiv \E\big\{g(G)\, \bbHe_k(G)\}\, .
\end{align}

The Hermite polynomials can be obtained as high-dimensional limits of the Gegenbauer polynomials introduced in the previous section. Indeed, the Gegenbauer polynomials (up to a $\sqrt d$ scaling in domain) are constructed by Gram-Schmidt orthogonalization of the monomials $\{x^k\}_{k\ge 0}$ with respect to the measure 
$\tilde\tau^1_{d}$, while Hermite polynomial are obtained by Gram-Schmidt orthogonalization with respect to $\gamma$. Since $\tilde\tau^1_{d}\Rightarrow \gamma$
(here $\Rightarrow$ denotes weak convergence),
it is immediate to show that, for any fixed integer $k$, 
\begin{align}
\lim_{d \to \infty} \Coeff\{ Q_k^{(d)}( \sqrt d x) \, B(\S^{d-1}; k)^{1/2} \} = \Coeff\left\{ \frac{1}{(k!)^{1/2}}\,\bbHe_k(x) \right\}\, .\label{eq:Gegen-to-Hermite}
\end{align}
Here and below, for $P$ a polynomial, $\Coeff\{ P(x) \}$ is  the vector of the coefficients of $P$. As a consequence,
for any fixed integer $k$, we have
\begin{align}\label{eqn:mu_lambda_relationship}
\mu_k(\sigma) = \lim_{d \to \infty} \xi_{d,k}(\sigma) (B(\S^{d-1}; k)k!)^{1/2}, 
\end{align}
where $\mu_k(\sigma)$ and $\xi_{d,k}(\sigma)$ are given in Eq. (\ref{eqn:sigma_He_decomposition}) and (\ref{eqn:technical_lambda_sigma}). 

\subsection{Functions on the hypercube}\label{sec:functions_hypercube}

Fourier analysis on the hypercube is a well studied subject \cite{o2014analysis}. The purpose of this section is to introduce
some notations that  make the correspondence with proofs on the sphere straightforward.
For convenience, we will adopt the same notations as for their spherical case. 

\subsubsection{Fourier basis}

Denote $\Cube^d = \{ -1 , +1 \}^d $ the hypercube in $d$ dimension. Let us denote $\tau_{d}$ to be the uniform probability measure on $\Cube^d$. All the functions will be assumed to be elements of $L^2 (\Cube^d, \tau_d)$ (which contains all the bounded functions $f : \Cube^d \to \R$), with scalar product and norm denoted as $\< \cdot , \cdot \>_{L^2}$ and $\| \cdot \|_{L^2}$:
\[
\< f , g \>_{L^2} \equiv \int_{\Cube^d} f(\bx) g(\bx) \tau_d ( \de \bx) = \frac{1}{2^n} \sum_{\bx \in \Cube^d} f(\bx) g(\bx).
\]
Notice that $L^2 ( \Cube^d, \tau_d)$ is a $2^n$ dimensional linear space. By analogy with the spherical case we decompose $L^2 ( \Cube^d, \tau_d)$ as a direct sum of $d+1$ linear spaces obtained from polynomials of degree $\ell = 0, \ldots , d$
\[
L^2 ( \Cube^d, \tau_d) = \bigoplus_{\ell = 0}^d V_{d,\ell}.
\]

For each $\ell \in \{ 0 , \ldots , d \}$, consider the Fourier basis $\{ Y_{\ell, S}^{(d)} \}_{S \subseteq [d], |S | =\ell}$ of degree $\ell$, where for a set $S \subseteq [d]$, the basis is given by
\[
Y_{\ell, S}^{(d)} (\bx) \equiv x^S \equiv \prod_{i \in S} x_i.
\]
It is easy to verify that (notice that $x_i^k = x_i$ if $k$ is odd and  $x_i^k = 1$ if $k$ is even)
\[
\< Y_{\ell, S}^{(d)} , Y_{k, S'}^{(d)} \>_{L^2} = \E[ x^{S} \times x^{S'} ] = \delta_{\ell,k} \delta_{S,S'}. 
\]
Hence $\{ Y_{\ell, S}^{(d)} \}_{S \subseteq [d], |S | =\ell}$ form an orthonormal  basis of $V_{d,\ell}$ and 
\[
\dim (V_{d,\ell} ) = B(\Cube^d;\ell) = {{d}\choose{\ell}}.
\]
As above, we will omit the superscript $(d)$ in $Y_{\ell, S}^{(d)}$ when clear from the context.

\subsubsection{Hypercubic Gegenbauer}

We consider the following family of polynomials $\{ Q^{(d)}_\ell \}_{\ell = 0 , \ldots, d}$ that we will call hypercubic Gegenbauer, defined as
\[
Q^{(d)}_\ell (  \< \bx , \by \> ) = \frac{1}{B(\Cube^d; \ell)} \sum_{S \subseteq [d], |S| = \ell} Y_{\ell,S}^{(d)} ( \bx ) Y_{\ell,S}^{(d)} ( \by ). 
\]
Notice that the right hand side only depends on $ \< \bx , \by \>$ and therefore these polynomials are uniquely defined. In particular,
\[
\< Q_\ell^{(d)} ( \< \ones , \cdot \> ) , Q_k^{(d)} ( \< \ones , \cdot \> ) \>_{L^2} = \frac{1}{B(\Cube^d;k)} \delta_{\ell k}.
\]
Hence $\{ Q^{(d)}_\ell \}_{\ell = 0 , \ldots, d}$ form an orthogonal basis of $L^2 ( \{ -d , -d+2 , \ldots , d-2 ,d\}, \Tilde \tau_d^1 )$ where $\Tilde \tau_d^1$ is the distribution of $\< \ones , \bx \>$ when $\bx \sim \tau_d$, i.e., $\Tilde \tau_d^1 \sim 2 \text{Bin}(d, 1/2) - d/2$.

We have
\[
\< Q_\ell^{(d)} ( \< \bx , \cdot \> ) , Q_k^{(d)} ( \< \by , \cdot \> ) \>_{L^2} =   \frac{1}{B(\Cube^d;k)} Q_{k} ( \< \bx , \by \> )\delta_{\ell k} .
\]
For a function $\sigma ( \cdot / \sqrt{d} ) \in L^2 ( \{ -d , -d+2 , \ldots , d-2 ,d\}, \Tilde \tau_d^1 )$, denote its hypercubic Gegenbauer coefficients $\xi_{d,k} ( \sigma)$ to be
\[
\xi_{d,k} (\sigma ) = \int_{\{-d, -d+2 , \ldots , d-2 , d\} } \sigma(x / \sqrt{d} ) Q_k^{(d)} ( x) \Tilde \tau_d^1 (\de x).
\]

Notice that by weak convergence of $\< \ones, \bx \> / \sqrt{d}$ to the normal distribution, we have also convergence of the (rescaled) hypercubic Gegenbauer polynomials to the Hermite polynomials, i.e., for any fixed $k$, we have
\begin{align}
\lim_{d \to \infty} \Coeff\{ Q_k^{(d)}( \sqrt d x) \, B(\Cube^d; k)^{1/2} \} = \Coeff\left\{ \frac{1}{(k!)^{1/2}}\,\bbHe_k(x) \right\}\, .\label{eq:Hyper-Gegen-to-Hermite}
\end{align}

\subsection{Hypercontractivity of Gaussian measure and uniform distributions on the sphere and the hypercube}
\label{app:hypercontractivity}

By Holder's inequality, we have $\| f \|_{L^p} \le \| f \|_{L^q}$ for any $f$ and any $p \le q$. The reverse inequality does not hold in general, even up to a constant. However, for some measures, the reverse inequality will hold for some sufficiently nice functions. These measures satisfy the celebrated hypercontractivity properties \cite{gross1975logarithmic, bonami1970etude, beckner1975inequalities, beckner1992sobolev}. 

\begin{lemma}[Hypercube hypercontractivity \cite{beckner1975inequalities}]\label{lem:hypercontractivity_hypercube} For any $\ell = \{ 0 , \ldots , d \}$ and $f_d \in L^2 ( \Cube^d)$ to be a degree $\ell$ polynomial, then for any integer $q\ge 2$, we have
\[
\| f_d \|_{L^q ( \Cube^d)}^2 \leq (q-1)^\ell \cdot \| f_d \|^2_{L^2 (\Cube^d)}.
\]
\end{lemma}

\begin{lemma}[Spherical hypercontractivity \cite{beckner1992sobolev}]\label{lem:hypercontractivity_sphere}
For any $\ell \in \N$ and $f_d \in L^2(\S^{d-1})$ to be a degree $\ell$ polynomial, for any $q \ge 2$, we have 
\[
\| f_d \|_{L^q(\S^{d-1})}^2 \le (q - 1)^\ell \cdot \| f_d \|_{L^2(\S^{d-1})}^2. 
\]
\end{lemma}

\begin{lemma}[Gaussian hypercontractivity]\label{lem:hypercontractivity_Gaussian}
For any $\ell \in \N$ and $f \in L^2(\R, \gamma)$ to be a degree $\ell$ polynomial on $\R$, where $\gamma$ is the standard Gaussian distribution. Then for any $q \ge 2$, we have 
\[
\| f \|_{L^q(\R, \gamma)}^2 \le (q - 1)^{\ell} \cdot \| f \|_{L^2(\R, \gamma)}^2. 
\]
\end{lemma}

The Gaussian hypercontractivity is a direct consequence of hypercube hypercontractivity.

\end{document}